\documentclass[sigconf]{acmart}

\usepackage{ulem}
\usepackage{color}
\usepackage{dsfont}
\usepackage{amsmath}
\usepackage{rotating}
\usepackage{makecell}
\usepackage{multirow}
\usepackage{subcaption}
\usepackage{float}
\usepackage{amsthm}
\usepackage{bbm}
\usepackage[ruled,linesnumbered]{algorithm2e}

\newtheorem{theorem}{Theorem}[section]

\DeclareMathOperator*{\argmin}{arg\,min}
\RestyleAlgo{ruled}

\begin{document}

\title[COMBINEX]{COMBINEX: A Unified Counterfactual Explainer for Graph Neural Networks via Node Feature and Structural Perturbations}

\author{Flavio Giorgi}
\email{giorgi@di.uniroma1.it}
\affiliation{%
  \institution{Sapienza University of Rome\\Department of Computer Science}
  \city{Rome}
  \country{Italy}
}
\author{Fabrizio Silvestri}
\email{fsilvestri@diag.uniroma1.it}
\affiliation{%
  \institution{Sapienza University of Rome\\Department of Computer, Control and Management Engineerin}
  \city{Rome}
  \country{Italy}
}

\author{Gabriele Tolomei}
\email{tolomei@di.uniroma1.it}
\affiliation{%
  \institution{Sapienza University of Rome\\Department of Computer Science}
  \city{Rome}
  \country{Italy}
}

\begin{abstract}
Counterfactual explanations have emerged as a powerful tool to unveil the opaque decision-making processes of graph neural networks (GNNs). However, existing techniques primarily focus on edge modifications, often overlooking the crucial role of node feature perturbations in shaping model predictions. To address this limitation, we propose COMBINEX, a novel GNN explainer that generates counterfactual explanations for both node and graph classification tasks. Unlike prior methods, which treat structural and feature-based changes independently, COMBINEX optimally balances modifications to edges and node features by jointly optimizing these perturbations. This unified approach ensures minimal yet effective changes required to flip a model's prediction, resulting in realistic and interpretable counterfactuals. Additionally, COMBINEX seamlessly handles both continuous and discrete node features, enhancing its versatility across diverse datasets and GNN architectures. Extensive experiments on real-world datasets and various GNN architectures demonstrate the effectiveness and robustness of our approach over existing baselines.
\end{abstract}




\maketitle

\section{Introduction}
Recent breakthroughs in deep learning have propelled significant advancements in artificial intelligence (AI) systems across a wide array of scientific and non-scientific fields. From engineering to social sciences, many areas of human knowledge have greatly benefited from these innovations. However, despite the excitement surrounding the potential of deep learning, growing concerns about the explainability and interpretability of these complex models persist among both the public and researchers.
Moreover, explainability is not only crucial for end-users but also for regulators and policymakers. For instance, the European Union has been actively working on regulations such as the Artificial Intelligence Act (AI Act)~\cite{european2023artificial}, which includes provisions for the ``\textit{right to explanation}''. This regulation requires that individuals have the right to obtain an explanation of decisions made by automated AI systems. 

In response to these needs, considerable efforts have been made to establish the foundations of Explainable Artificial Intelligence \cite{dovsilovic2018explainable,angelov2021explainable} (XAI). 
Among the various XAI techniques proposed, \textit{counterfactual explanations} (CFEs) have emerged as one of the most promising methods for explaining model predictions~\cite{stepin2021survey}. The primary goal of a CFE is to elucidate a model's prediction for a given instance by identifying minimal changes to the input features that would change the model's output. 
Therefore, CFEs are designed to answer ``what if'' questions, helping users comprehend the inner logic of a complex model in the form: ``\textit{If A had been different, B would \textbf{not} have occurred}''. This capability is particularly valuable in sensitive domains such as finance, healthcare, and justice, where understanding the reasoning behind a model's prediction is essential for building trust and ensuring accountability.
Moreover, this approach not only aids in understanding the model's decision-making process but also provides \textit{actionable} insights for users.


Existing CFE methods have played a crucial role in interpreting and validating predictions from various machine learning models~\cite{ribeiro2016should,lundberg2017unified,tolomei2017interpretable,chen2022relax}. Recently, however, there has been a growing need to extend these techniques to accommodate diverse data types and model architectures. Among these, Graph Neural Networks (GNNs) have emerged as particularly effective for tasks involving graph-structured data, such as node classification and link prediction. GNNs have delivered significant benefits across multiple sectors; for instance, in the financial industry, they underpin (semi-)automated fraud detection systems~\cite{10.1145/3589334.3645673}. Furthermore, advancements in GNNs have broadened their applicability to fields like chemistry and biology. In Wong et al.~\cite{Wong2023DiscoveryOA}, for example, GNNs were employed to predict the chemical properties of molecules, thereby bypassing the need for costly and time-intensive experimental screenings of large chemical libraries. Due to the unique characteristics of GNNs, developing methods to explain their predictions is, therefore, critical. 

Inspired by CF-GNNExplainer~\cite{lucic2022cf}, we propose a unified framework, COMBINEX, to find the optimal counterfactual explanations for GNN models.
Unlike prior methods, which treat structural and feature-based changes independently, COMBINEX balances modifications to edges and node features by jointly optimizing these perturbations (see Figure~\ref{fig:sparse}).
\begin{figure}
    \centering
    \includegraphics[width=1\linewidth]{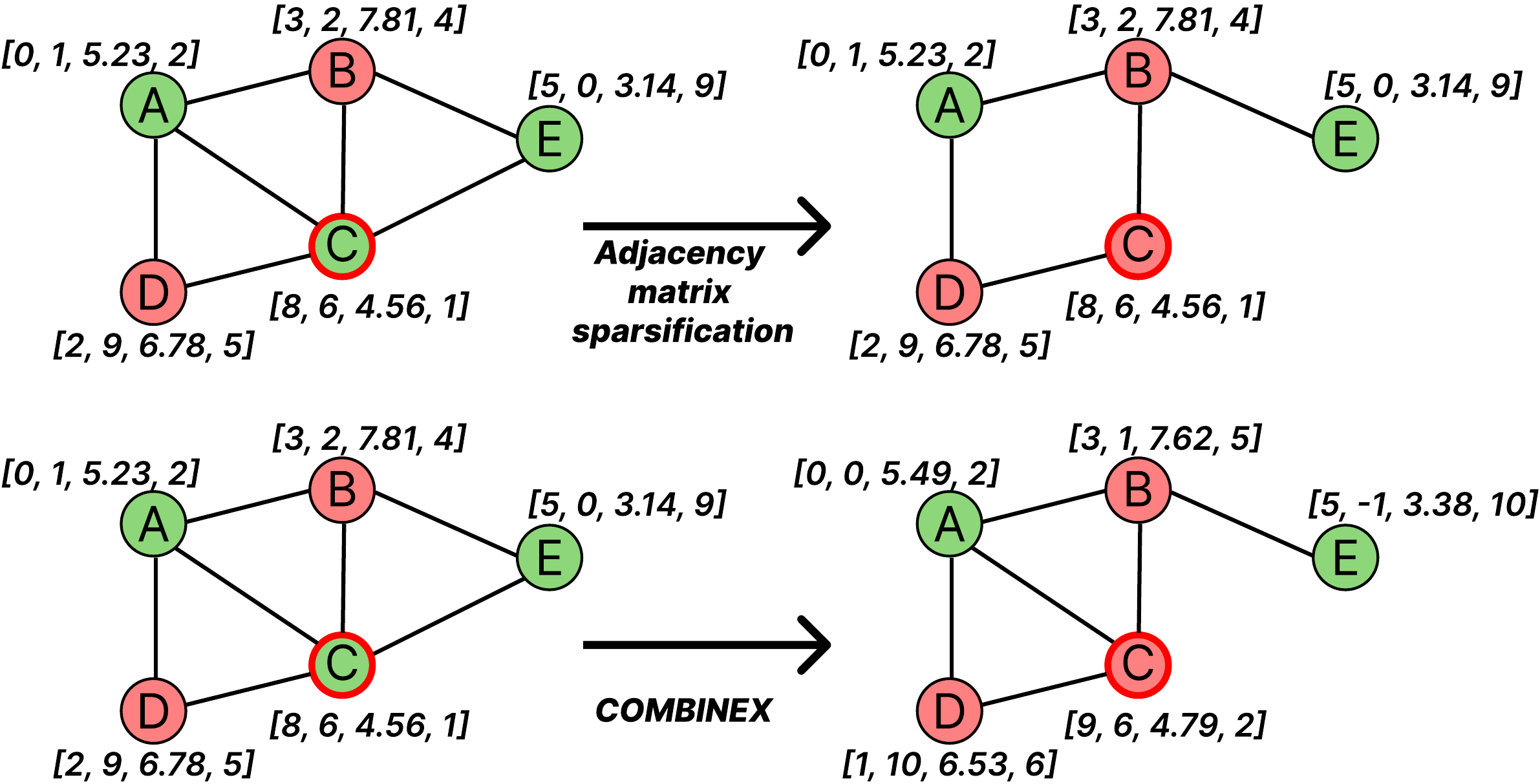}
    \caption{Comparison of two approaches for generating counterfactual explanations in Graph Neural Networks (GNNs): CF-GNNExplainer vs. COMBINEX. CF-GNNExplainer (\textit{top}) modifies the graph structure by perturbing the adjacency matrix through edge removal. COMBINEX (\textit{bottom}) takes a unified approach, balancing both node feature and structural perturbations to find optimal counterfactual explanations.}
    \label{fig:sparse}
\end{figure}
Through extensive experiments, we show that modifying node features improves standard quality metrics for explanations
across many different datasets. Additionally, we find that the model training degree has a direct impact on the performance of some of the explainers we evaluated. Our contribution can be summarized as follows:
\begin{itemize}
  \item We propose \textbf{COMBINEX}, a novel counterfactual explainer for Graph Neural Networks (GNNs) that generates explanations by perturbing both node features and graph structure, introducing a \textbf{stochastic optimization framework} that efficiently finds minimal perturbations required to alter model predictions while preserving interpretability.
    \item We conduct extensive experiments on multiple real-world datasets, demonstrating that COMBINEX outperforms existing counterfactual explainers in terms of \textit{validity}, \textit{fidelity}, and \textit{sparsity}.
    \item We provide the source code of our method at the following GitHub repository: \url{https://github.com/flaat/COMBINEX}.
\end{itemize}
The remainder of this paper is structured as follows. In Section\ref{sec:related}, we discuss related work. Section~\ref{sec:formulation} covers background concepts. We introduce our proposed method (COMBINEX) in Section~\ref{sec:method}, which we extensively validate in Section~\ref{sec:experiments}. Section~\ref{sec:feasibility} discusses the feasibility of our method. Finally, we draw the conclusions in Section~\ref{sec:conclusion}.

\section{Related Work}
\label{sec:related}
Recently, counterfactual explanations have become popular for explaining black-box models like GNNs. In general, many different approaches address the problem of generating a counterfactual example given a factual instance and a predictive model.

A number of techniques have been developed to explain prediction for node classification tasks. CF-GNNExplainer \cite{lucic2022cf} is a perturbation-based method that uses stochastic optimization to find a counterfactual graph that changes the model's classification of a given node. The optimization process trains a perturbation matrix that modifies the graph structure, eliminating edges until the prediction changes. 
In RCE-GNN \cite{bajaj2021robust}, the authors model the decision logic of a GNN using multiple decision regions. A set of linear decision boundaries of the GNN induces each region. The linear decision boundaries of the decision region capture the common decision logic on all the graph instances inside the decision region. Exploring the common decision logic encoded in the linear boundaries, it is possible to produce counterfactual samples. Using the linear boundaries of the decision region, they build a loss function that is used to train a neural network that generates a counterfactual explanation for an oracle, ensuring that the counterfactual sample lies in the decision region.
GNN-MOeXP \cite{liu2021multi} is a multi-objective factual-based explanation method for GNN node classification tasks. GNN-MOExp imposes counterfactual relevance to its factual explanation subgraphs. It looks for a subgraph in the original instance that optimizes factual and counterfactual features. GNN-MOExp comes with several limitations that limit the expressiveness of the produced counterfactual: (1) the factual subgraphs are required to be acyclic, and (2) the explanation size is specified a priori.
The authors in CFF \cite{tan2022learning} build an optimization framework to get GNN explanations. The framework integrates counterfactual and factual reasoning objectives: the counterfactual objective maintains edges relevant to the explanation, while the factual objective ensures that the extracted explanation contains sufficient information. 
The UNR-Explainer \cite{kangunr}  instead generates counterfactual (CF) explanations for unsupervised node representation learning by identifying subgraphs whose perturbation significantly alters a node’s \textit{top-k} nearest neighbors in the embedding space.
Generative AI models like autoencoders and diffusion models have been widely used to generate counterfactual samples to explain oracle. To this day, there are several generative-based counterfactual explainers. CLEAR \cite{ma2022clear} is a generative VAE-based counterfactual explainer that uses variational autoencoders to generate counterfactual explanations on graphs for graph-level prediction models. Another approach, like D4Explainer \cite{chen2023d4explainer}, instead uses discrete diffusion models to generate counterfactual graphs by means of a discrete diffusion process on the adjacency matrix.

The method presented in this work differs from existing approaches as it is the first unified framework that 
balance edge and node feature perturbations.
\vspace{-1mm}
\section{Background}\label{sec:formulation}

\noindent \textbf{\textit{Graph Neural Networks.}}
Graph Neural Networks (GNNs) extend deep learning to graph-structured data. 
Formally, let $G=(V,E)$ be a graph with $n$ nodes $V$, and $m$ edges $E$.
The structure of $G$ is encoded by its adjacency matrix $\mathbf{A} \in \lbrace 0,1 \rbrace^{n \times n}$, where $\mathbf{A}_{i,j}$ = 1 iff $(i, j) \in E$. 
Moreover, we assume there exists a feature matrix $\mathbf{X} \in \mathbb{R}^{n \times k}$, which associates features to nodes of $G$.
Generally speaking, a GNN $g$ learns a hidden representation of nodes in the graph (i.e., a node embedding). Such representation is, in turn, used for downstream tasks of interest like node classification, link prediction, and graph classification \cite{wu2022graph}.
GNNs generate node embeddings through a process known as message passing, where a node's features are iteratively updated based on the features of its neighbors.
Formally, let $\mathbf{h}^l_u$ denote the embedding of node $u\in V$ at the $l$-th layer of $g$ and $\mathcal{N}_u$ the set of $u$'s neighbors. The updated node embedding for $u$ at layer $l+1$ is then computed as follows.
\[
\mathbf{h}_u^{(l+1)} = g(\mathbf{A}^{(l+1)}_u,\mathbf{X}^{(l+1)}_u) =\sigma \left( \text{AGG} \left( \left\{ \mathbf{h}_v^{(l)} : v \in \mathcal{N}(u) \right\} \right), \mathbf{h}_u^{(l)} \right),
\]
where: $\mathbf{A}^{l+1}_v$ and $\mathbf{X}^l_v$ are the adjacency matrix and the feature matrix of the subgraph $G^l_u$ of $G$, induced by $u$ and its $l$-hop neighbors; \(\text{AGG}(\cdot)\) is a function combining neighbor embeddings (e.g., sum, mean, or attention); \(\sigma(\cdot)\) is an activation function like ReLU. 



\noindent \textbf{\textit{Counterfactual Explanations.}}
The general formulation of the counterfactual explanation problem in a classification task can be expressed as the following optimization problem. Given an input sample $\boldsymbol{x}$ and a classifier $f$ -- hereinafter referred to as \textit{oracle} -- the objective is to find a \textit{counterfactual sample} $\boldsymbol{x'}$ such that $f(\boldsymbol{x}) \neq f(\boldsymbol{x'})$ while minimizing the distance  $d(x, x')$. 
Formally:
\begin{equation}
\label{eq:cf-train}
\begin{aligned}
\boldsymbol{x'} = \argmin_{\boldsymbol{\widetilde{x}}}  &~d(\boldsymbol{x}, \boldsymbol{\widetilde{x}})\\
\text{ s.t.: } &f(\boldsymbol{x}) \neq f(\boldsymbol{\widetilde{x}}).
\end{aligned}
\end{equation}
The function $d$ enforces similarity between the counterfactual sample $\boldsymbol{x'}$ and the original factual sample $\boldsymbol{x}$.
Note that the formulation above is general enough and also encompasses a \textit{targeted} version of the counterfactual explanation problem, where the objective is not only to ensure $f(\boldsymbol{x}) \neq f(\boldsymbol{\widetilde{x}})$ but also to enforce a specific target prediction, i.e., $f(\boldsymbol{\widetilde{x}}) = y_t$.

\section{Proposed Method}
\label{sec:method}
\subsection{Problem Formulation}
Tackling the counterfactual explanation problem outlined in (\ref{eq:cf-train}) within the context of GNNs -- where the input instance is a graph -- introduces unique challenges.
To formalize this, we follow the notation used by Romero et al.~\cite{Prado_Romero_2023}, replacing the original input instance $\boldsymbol{x}$ with the graph $G(V, E)$ and the counterfactual sample $\boldsymbol{x'}$ as $G'(V', E')$.
Since $G$ and $G'$ can be represented by their adjacency and node feature matrices, respectively, the (targeted) counterfactual explanation problem for GNNs can be rewritten as follows:
\begin{equation}
\label{eq:cf-gnn-train}
\begin{aligned}
\mathbf{A'},\mathbf{X'} = \argmin_{\mathbf{\widetilde{A}},\mathbf{\widetilde{X}}}  &~\Big\{d_{\text{topology}}(\mathbf{A}, \mathbf{\widetilde{A}}) +  d_{\text{features}}(\mathbf{X}, \mathbf{\widetilde{X}})\Big\}\\
\text{ s.t.: } &f(g(\mathbf{A}, \mathbf{X})) \neq f(g(\mathbf{\widetilde{A}}, \mathbf{\widetilde{X}})) = y_t,
\end{aligned}
\end{equation}
where $d_{\text{topology}}$ captures the structural distance between the original graph $G$ and its counterfactual $G'$, $d_{\text{features}}$ measures the magnitude of node feature perturbation between $G$ and $G'$, $f$ is a classifier representing a downstream graph-related task, and $y_t$ denotes the desired target prediction for the counterfactual example.
Note that the input to the classifier $f$ can be any representation produced by the GNN $g$. For instance, if $f$ operates on a single node embedding, the task corresponds to node classification. If $f$ processes the entire graph representation, it performs a graph classification task. Finally, if $f$ takes two node embeddings as input, it may determing whether the nodes are connected by a link or not, i.e., edge classification.

The counterfactual optimization problem in (\ref{eq:cf-gnn-train}) can be directly translated into a loss function ($\mathcal{L}_{total}$) as follows:
\begin{equation}
\label{eq:loss}
\mathcal{L}_{total} = \mathcal{L}_{CF} + (1 - \alpha) \mathcal{L}_E + \alpha \mathcal{L}_X, 
\end{equation}
where: $\mathcal{L}_{CF}$ penalizes when the counterfactual goal is \textit{not} met, i.e., when the modified input instance does \textit{not} lead to an actual classification change; $\mathcal{L}_{E}$ that enforces minimal structural perturbations of the input graph; $\mathcal{L}_{X}$ that tries to control the magnitude of node feature perturbations. 

Since we focus on classification tasks only, the first term ($\mathcal{L}_{CF}$) can be expressed as $\eta\mathcal{L}_{CE}$, where $\mathcal{L}_{CE}$ is the cross-entropy loss between the original and the counterfactual predictions and  $\eta$ is an indicator variable ensuring this term is considered only when the classification change has not yet occurred.  The graph structural loss ($\mathcal{L}_E$) is defined as the sum of absolute differences between the original adjacency matrix and the perturbed one. The node feature loss ($\mathcal{L}_X$) is composed of two terms: L1 loss for discrete features, ensuring minimal modifications while preserving categorical interpretability and Mean Squared Error (MSE) for continuous features, penalizing large deviations while allowing smooth gradient-based optimization.
Finally, $\alpha$ is an adjustable trade-off hyperparameter that regulates the balance between feature and structural modifications. 
This \textit{joint} formulation ensures that both types of perturbations remain minimal while still enforcing the desired classification change.
The impact of $\alpha$ will be further analyzed in Section~\ref{subsec:alpha}.

To optimize the loss function defined in (\ref{eq:loss}) using gradient-based methods, we introduce two differentiable perturbation matrices. The first, the node feature perturbation matrix $\mathbf{P}$, is responsible for modifying node features, while the second, the edge perturbation matrix 
$\mathbf{EP}$, governs changes in the graph topology by modifying edge values.

Given that both structural and node feature data in graphs often consist of discrete values, we adopt the approach proposed by Lucic et al.~\cite{lucic2022cf} to preserve differentiability. Specifically, we apply a 
$\tanh$-based transformation to discrete node features and a sigmoid activation to edge perturbations. This ensures that modifications remain within valid bounds while allowing gradients to propagate effectively during optimization. The differentiable versions of these matrices are used to update model parameters via backpropagation, whereas their corresponding thresholded (non-differentiable) versions are ultimately employed to generate the final counterfactual examples.

\subsection{The COMBINEX Algorithm}
In this section, we detail our proposed counterfactual explanation algorithm (Algorithm \ref{alg:one}), which solves the objective defined in (\ref{eq:loss}) via gradient-based optimization.

The algorithm takes as input a graph $G(V, E)$, a pre-trained GNN model $g$ with fixed parameters, an integer $k$ representing the maximum number of optimization epochs, a target class $y_t$, a vector $\mathbf{M}_d$ serving as a mask for discrete features, and a learning rate $\gamma$. 

The procedure begins by initializing the perturbation matrices for node features and edges, denoted as $\mathbf{P}^0$ and $\mathbf{EP}^0$, respectively. It then computes the initial prediction $y$ for the input graph $G(V, E)$ and sets the epoch counter to 1 (lines 1--2). Subsequently, the algorithm extracts the edge matrix and node feature matrix from the original graph (lines 3--5) and initializes the variable that will store the counterfactual sample. Additionally, the continuous feature mask is derived from the discrete feature mask $\mathbf{M}_d$, which is used to selectively enable or disable specific features during computation.

The optimization process begins in line 6. In line 7, we get both the edge and node perturbations using the function \FuncSty{get\_pert} (see Algorithm \ref{alg:perturbation}). 
At line 10, the algorithm computes the loss function using the \FuncSty{get\_loss}
function (Algorithm \ref{alg:loss}).
Finally, from lines 12--20, the total loss is computed, and the perturbation matrices are updated to minimize it. If the model's prediction changes with the newly perturbed node features and edges while achieving the lowest loss observed thus far, the algorithm identifies a valid counterfactual sample.
Below, we describe the algorithms used to compute feature and edge perturbations as well as the loss function that drives the optimization process at the core of COMBINEX. 

\SetKwComment{Comment}{/* }{ */}

\begin{algorithm}
\SetAlgoLined
\caption{\textbf{COMBINEX}}\label{alg:one}
\KwIn{$G(V, E)$: Graph to explain, $g$: Oracle, $k$: maximum number of epochs, $y_t$:  target class, $\mathbf{M_d}$ discrete features mask vector, $\gamma$: learning rate, $\mathbf{R}_{min}$ and $\mathbf{R}_{max}$: vectors containing lower and upper bound for each feature}
\KwOut{The counterfactual sample $G'$}

$\mathbf{P}^0 \gets \mathbf{0}^{n\times f}$, $\mathbf{EP}^0 \gets \mathbf{1}^{n \times 1}$\;
$y \gets g(G(V, E))$, $epoch \gets 1$\;
$\mathbf{E} \gets G(E), \mathbf{X} \gets G(V)$\;
$G' \gets \emptyset, \mathcal{L}_{best} \gets +\infty$\;
$ \mathbf{M_c} \gets {\textbf{1}} - \mathbf{M_d}$\;

\While{$epoch \leq k $}{

    $\mathbf{E}^{nd}_p, \mathbf{X}^{nd}_p, \mathbf{E}_p, \mathbf{X}_p, \mathbf{X}_d, \mathbf{X}_c  \gets \FuncSty{get\_pert}(\mathbf{R}_{min}, \mathbf{R}_{max},\mathbf{X}, \mathbf{E}, \mathbf{M}_d, \mathbf{M}_c, \mathbf{P}^t, \mathbf{EP}^t)$\;
    $y_{new} \gets g(G(\mathbf{X}_p, \mathbf{E}_p))$\; 
    $y^{nd}_{new} \gets g(G(\mathbf{X}^{nd}_p, \mathbf{E}^{nd}_p))$\;
    $\mathcal{L}_{total} \gets \FuncSty{get\_loss}(y_{new}, y_t, \mathbf{EP}^t, \mathbf{X}, \mathbf{M}_d, \mathbf{X}_d, \mathbf{M}_c, \mathbf{X}_c, epoch)$\;
    $\alpha \gets \FuncSty{get\_alpha}(epoch)$\; 
    $\mathbf{P}^{(t+1)}_x \gets \mathbf{P}^{(t)}_x - \gamma \nabla_{P, EP}(\mathcal{L}_{total})$\;
    $\mathbf{EP}^{(t+1)}_x \gets \mathbf{EP}^{(t)}_x - \gamma \nabla_{P, EP} (\mathcal{L}_{total})$\;    

    \If{$y_{new} = y_t \wedge \mathcal{L} < \mathcal{L}_{best}$}{
        $G' \gets G(\mathbf{X}_p, \mathbf{E}_p)$\;
        $\mathcal{L}_{best} \gets \mathcal{L}_{total}$
        
    }
    $epoch \gets epoch + 1$
  }
  \textbf{return} $G'$
\end{algorithm}

\begin{algorithm}
\SetAlgoLined
\caption{Loss Computation (\FuncSty{get\_loss})}\label{alg:loss}
\KwIn{$y^{nd}_{new}$: the new prediction, $y_t$: counterfactual target class, $\mathbf{EP}$: The edge perturbation vector, $\mathbf{X}$: The original nodes features, $\mathbf{M}_d$: discrete features mask,  $\mathbf{X}_d $: The discrete perturbed nodes features, $\mathbf{M}_c$: continuous features mask, $\mathbf{X}_c$: perturbed continuous features, $epoch$: the current epoch}
\KwOut{The loss value $\mathcal{L}_{total}$}

    $\eta \gets \mathbbm{1}[argmax(y_{new}) = y_t]$\;
    $\mathcal{L}_{CE} \gets \text{CE}(y_{new}, y_t)$\;
    $\mathcal{L}_{E} \gets \sum |\sigma(\mathbf{EP}_x) - \mathbf{1}|$\; 
    $\mathcal{L}_{d} \gets \text{L1}(\mathbf{X} \odot \mathbf{M}_d, \mathbf{X}_d)$\;
    $\mathcal{L}_{c} \gets \text{MSE}(  \mathbf{X} \odot \mathbf{M}_c, \mathbf{X}_c)$\;
    $\mathcal{L}_{X} \gets \mathcal{L}_{d} + \mathcal{L}_{c}$\;
    $\alpha \gets \FuncSty{get\_alpha}(epoch)$\; 
    $\mathcal{L}_{total} = \eta \mathcal{L}_{CE} + (1 - \alpha)  \mathcal{L}_E + \alpha \mathcal{L}_X $\; 

  \textbf{return} $\mathcal{L}_{total}$
\end{algorithm}

\smallskip
\noindent \textbf{\textit{Features and Edge Perturbation.}} 
Algorithm~\ref{alg:perturbation} outlines the perturbation mechanism applied to both node features and edges. The process begins by computing a scaled perturbation vector $\mathbf{P}^{(t)}_s$, to perturb the discrete node features, using a $\tanh$-based transformation, ensuring that the values remain within the predefined feature bounds (line 1). Next, discrete features are updated by applying the Hadamard product between the scaled perturbation and the discrete feature mask $\mathbf{M}_d$, followed by summation with the original node features $\mathbf{X}$, and subsequently clamped within the feature range (line 2). Similarly, continuous features are updated without any transformation on the perturbation and are clamped accordingly (line 3). The final perturbed node feature matrix is obtained by summing the contributions of discrete and continuous feature updates, yielding $\mathbf{X}_p = \mathbf{X}_d + \mathbf{X}_c$ (line 4). 
\begin{algorithm}[]
\SetAlgoLined
\caption{Feature and Edge Perturbation (\FuncSty{get\_pert})}\label{alg:perturbation}
\KwIn{$\mathbf{R}_{min}$ and $\mathbf{R}_{max}$: Vectors containing lower and upper bound for each feature, $\mathbf{X}$: Original nodes features matrix, $\mathbf{E}$: Original edges matrix, $\mathbf{M}_d$: discrete features mask, $\mathbf{M}_c$: continuous features mask, $\mathbf{P}^{(t)}$: Node feature perturbation vector at the $t$-th iteration, $\mathbf{EP}^{(t)}$: Edge  perturbation vector at the $t$-th iteration}
\KwOut{$\mathbf{E}^{nd}_p$: the non-differentiable edge matrix, $\mathbf{X}^{nd}_p$: the non-differentiable nodes features matrix, $\mathbf{E}_p$: the differentiable edge matrix, $\mathbf{X}_p$:  the differentiable nodes features matrix, $\mathbf{X}_d$: perturbed discrete features, $\mathbf{X}_c$: perturbed continuous features}

    $\mathbf{P}^{(t)}_s = (\mathbf{R}_{min} + (\mathbf{R}_{max} - \mathbf{R}_{min})) *\tanh(\mathbf{P}^{(t)})$\;
    $\mathbf{X}_{d} \gets  \FuncSty{clamp}(\mathbf{M}_d \odot \tanh(\mathbf{P}_s^{(t)}) + \mathbf{X}, \mathbf{R}_{min}, \mathbf{R}_{max})$\;
    $\mathbf{X}_{c} \gets  \FuncSty{clamp}(\mathbf{M}_c \odot \mathbf{P}^{(t)} + \mathbf{X}, \mathbf{R}_{min}, \mathbf{R}_{max})$\;
    $\mathbf{X}_p \gets \mathbf{X}_d + \mathbf{X}_c$\;
    $\mathbf{E}_p \gets  \mathbf{E} * \sigma(\mathbf{EP}^{(t)})$\;
    $\mathbf{E}^{nd}_p \gets \mathbbm{1}[\sigma(\mathbf{EP}^{(t)}) > 0.5]$\;
    $\mathbf{X}^{nd}_p \gets \FuncSty{clamp}(\mathbf{M}_d \odot \FuncSty{to\_int}(\tanh(\mathbf{P}_s^{(t)}) + \mathbf{X}_c), \mathbf{R}_{min}, \mathbf{R}_{max})$\;

  \textbf{return} $\mathbf{E}^{nd}_p, \mathbf{X}^{nd}_p, \mathbf{E}_p, \mathbf{X}_p, \mathbf{X}_d, \mathbf{X}_c $
\end{algorithm}

For edge perturbations, the differentiable edge matrix $\mathbf{E}_p$ is computed by applying a sigmoid activation $\sigma$ to the edge perturbation vector $\mathbf{EP}^{(t)}$, ensuring that the resulting values are in the range $(0,1)$ (line 5). The non-differentiable edge matrix $\mathbf{E}^{nd}_p$ is obtained by thresholding the sigmoid output at $0.5$, enforcing a hard binary decision on edge presence (line 6). Similarly, the non-differentiable node feature matrix $\mathbf{X}^{nd}_p$ is computed by converting the discrete feature updates into integer values and applying clamping to maintain the predefined feature bounds (line 7).
We employ both differentiable and non-differentiable versions of the matrices for edges and features in order to maintain gradient flow during optimization while ultimately obtaining the final discrete prediction. Specifically, the differentiable matrices $\mathbf{X}_p$ and $\mathbf{E}_p$ are used to produce a continuous output from the model, which is essential for computing the loss via backpropagation. In contrast, the non-differentiable matrices $\mathbf{X}^{nd}_p$ and $\mathbf{E}^{nd}_p$ are derived by thresholding $\mathbf{EP}^{(t)}$ and $\mathbf{P}^{(t)}_s$ (see Algorithm \ref{alg:perturbation}, lines 6-7), thereby yielding the definitive prediction. It is important to note that the final counterfactual sample is constructed from these non-differentiable components, ensuring that the discrete nature of the graph is preserved in the final output.
\noindent \textbf{\textit{Loss Function.}} The loss function is computed in Algorithm  \ref{alg:loss}: line 1 computes the indicator variable $\eta$, which determines whether the cross-entropy loss $\mathcal{L}_{CE}$ should be applied (line 2). This ensures that the classification loss is only considered when the counterfactual prediction does not yet match the counterfactual target class $y_t$. The \textit{edge distance loss} $\mathcal{L}_E$ is then computed in line 3 by summing the absolute differences between the sigmoid-transformed edge perturbation values and $1$, encouraging minimal modifications to the graph structure. 

From lines 4--6, the \textit{node feature loss} $\mathcal{L}_X$ is computed separately for discrete and continuous features. The discrete feature loss $\mathcal{L}_d$ is defined using the L1 loss between the masked original feature values and their perturbed versions, while the continuous feature loss $\mathcal{L}_c$ is computed using the Mean Squared Error (MSE). These two losses are then combined to form the total feature loss $\mathcal{L}_X = \mathcal{L}_d + \mathcal{L}_c$.

Next, in line 7, the function $\FuncSty{get\_alpha}$ (see Appendix \ref{alg:get_alpha}) determines the value of the weighting parameter $\alpha$, which regulates the trade-off between edge and feature modifications. To enhance the flexibility of the framework, we introduce the ability to select from various \textit{scheduling policies}, allowing $\alpha$ to dynamically evolve over training epochs. This adaptive approach ensures that the model progressively refines its focus on structural and feature perturbations, leading to a more effective optimization process and improving the quality of the generated counterfactual explanations. Finally, in line 8, the total loss $\mathcal{L}_{total}$ is computed as specified in (\ref{eq:loss}). This formulation ensures that perturbations remain minimal while enforcing the desired classification change. 

\vspace{-2.5mm}
\subsection{Computational Complexity Analysis}
The time complexity of the COMBINEX algorithm is primarily determined by its iterative optimization process over $k$ training epochs. 
Let $n$ and $m$ denote the number of nodes and edges in the graph, respectively. Moreover, let $f$ the number of node features. 
Therefore, each iteration consists of:\\ 
\noindent (1) \textbf{Perturbation Step:} $O(nf + m)$ for updating node features and edge perturbations;\\ 
\noindent (2) \textbf{GNN Forward Pass:} $O(L(n + m) f)$, where $L$ is the number of GNN layers;\\
\noindent (3) \textbf{Loss Computation:} $O(nf + m)$ for feature and edge losses;\\
\noindent (4) \textbf{Gradient Update:} $O(nf + m)$ for updating perturbation vectors.

Overall, the worst-case complexity is, therefore, $\mathcal{O}(k n f + k m f)$. This implies that the algorithm scales \textit{linearly} with the number of nodes, edges, features, and training epochs.

\section{Experiments}\label{sec:experiments}
Below, we outline our experimental setup and evaluation methodology to assess the effectiveness of our proposed COMBINEX approach for generating counterfactual explanations for GNNs. 

\noindent \textbf{\textit{Tasks, Datasets, and Models.}} Our method is tested on two tasks -- \textit{node classification} and \textit{graph classification} -- using a diverse collection of real-world datasets to ensure robustness and generalizability across domains such as citation networks, web page classification, and social network analysis (see Tables~\ref{tab:datasets} and \ref{tab:datasets_graph}). Specifically, the dataset collection includes the Planetoid datasets (CiteSeer, Cora, and PubMed)
, the WebKB datasets, 
the Attributed datasets (Wiki and Facebook), 
the Biological datasets (AIDS, Enzymes, and Proteins), originally designed for graph classification and here adapted also for node classification; the COIL-DEL dataset,
and the Miscellaneous category, which includes the Karate and Actor datasets.

We experiment with three GNN models: Graph Convolutional Network (GCN) \cite{kipf2016semi}, Chebyshev Network (ChebNet) \cite{he2022convolutional}, and GraphConv Networks \cite{morris2019weisfeiler}.

\noindent \textbf{\textit{Baselines.}} We evaluate our approach against a set of baselines that span naive strategies (random-edges, random-features, and ego-graph) and state-of-the-art techniques (CF-GNNExplainer \cite{lucic2022cf}, CounterFactual and Factual Explainer (CFF) \cite{tan2022learning}, and UNR \cite{kangunr}). 
Baselines such as random-edges and random-features perturb the graph's adjacency and feature matrices by applying random modifications, whereas the ego-graph method extracts a subgraph centered on a given node. 

\noindent \textbf{\textit{Evaluation Measures.}} To assess the quality of generated counterfactuals, we consider several key metrics. \textit{Validity} is an indicator function that returns $1$ if the counterfactual successfully alters the model's prediction compared to the original instance, as defined in~\cite{guidotti2019factual, guidotti2022counterfactual}. \textit{Fidelity} measures how well the counterfactual explanation aligns with the oracle’s decisions, following~\cite{Prado_Romero_2023}. \textit{Node/Edge sparsities} are computed as the ratio of modified features/edges between the factual and counterfactual graphs, respectively, ensuring minimal perturbations. Finally, \textit{Distribution Distance} is quantified using the $L_2$ distance between a graph's embedded representation and the dataset's mean embedding, capturing how much the counterfactual deviates from the original average data distribution.

\noindent \textbf{\textit{Settings.}}
The experiments were conducted on two machines, each equipped with an Nvidia GTX 4090 GPU, 64 GB of RAM, and an AMD Ryzen 9 7900 processor. To ensure robustness, we employed 4-fold cross-validation. Additionally, we performed an extensive hyperparameter tuning process to identify the optimal parameter configurations for our method, which were set as follows: learning rate of $0.1$, and $500$ training epochs. For the baseline models, we used the same hyperparameters specified in the original papers.  

\noindent \textbf{\textit{Scheduling Policy $\pmb{\alpha}$.}} We explored several $\alpha$ scheduling policies to control the trade-off between edge and feature modifications in our loss function. These different strategies result in multiple variants of our COMBINEX method. Specifically, we consider the following distinct approaches. The \textit{linear} policy (COMBINEX$_{\textit{lin}}$) lets $\alpha$ decrease linearly from $1.0$ to $0.0$ over the course of training. The \textit{exponential} policy (COMBINEX$_{\textit{exp}}$) sets $\alpha = e^{-\frac{epoch}{\delta}}$, leading to an exponential decay. The \textit{sinusoidal} policy (COMBINEX$_{\textit{cos}}$) modulates $\alpha$ according to $\alpha = 0.5 \times \left(1 + \cos\left(\pi \times \frac{epoch}{epochs_{max}}\right)\right)$, resulting in a cosine-shaped decay. The \textit{dynamic} policy (COMBINEX$_{\textit{dyn}}$) adjusts $\alpha$ based on the relative magnitudes of the edge loss and node feature loss, setting $\alpha=0$ when the edge loss dominates and $\alpha=1$ otherwise. Finally, if no policy is specified (COMBINEX$_{\textit{def}}$), $\alpha$ is fixed at a constant, default value $\alpha_{default}$. We also evaluated a variant where $\alpha$ is fixed at $1$, meaning only node features are perturbed. We refer to this approach as COMBINEX$_{\textit{feat}}$.

For a comprehensive list of parameters and settings, we refer the reader to Appendix \ref{subsec:params} and the GitHub repository.


\begin{table}
\centering
\scalebox{0.9}{
\begin{tabular}{|l|c|c|c|c|}\hline
 \textbf{Dataset Name}& \textbf{\#Nodes}& \textbf{\#Edges}& \textbf{\#Features}&\textbf{\#Classes}\\ \hline 
 Cora \cite{yang2016revisiting}& 2,708& 5,429& 1,433&7\\ \hline 
 CiteSeer \cite{yang2016revisiting} & 3,312& 4,732& 3,703&6\\ \hline 
 PubMed \cite{yang2016revisiting}& 19,717& 44,338& 500&3\\ \hline 
 Karate \cite{zachary1977information}& 34 & 156 & 34 & 4 \\ \hline 
 Actor  \cite{pei2020geom} & 7,600& 30,019& 932&5\\\hline
 WebKB Cornell \cite{pei2020geom}& 183& 298& 1,073&5\\ \hline 
 WebKB Texas  \cite{pei2020geom} & 183& 325& 1,073&5\\ \hline 
 WebKB Wisconsin  \cite{pei2020geom} & 251& 515& 1,073&5\\ \hline
Wiki \cite{yang2023pane}& 2,405 & 17,981& 4,973&17\\ \hline
Facebook \cite{leskovec2016snap} & 4,039 & 88,234& 1,283 &193\\ \hline
AIDS \cite{AIDS2004, Riesen2008} & 31,385 & 64,780& 4 & 37 \\ \hline
Proteins \cite{dobson2003distinguishing, borgwardt2005protein} & 43,471 & 162,088 & 29 & 2\\ \hline
Enzymes \cite{borgwardt2005protein, schomburg2004brenda} & 19,580 & 74,564& 18  & 2\\ \hline
\end{tabular}
}
\caption{Datasets used for \textit{node classification}.
}
\label{tab:datasets}
\end{table}

\begin{table}
\centering
\scalebox{0.85}{
\begin{tabular}{|l|c|c|c|c|c|}\hline
 \textbf{Dataset Name}& \textbf{\#Graphs} & \textbf{\#Nodes}& \textbf{\#Edges}& \textbf{\#Features}&\textbf{\#Classes}\\ \hline 
AIDS \cite{AIDS2004, Riesen2008}& 2000 & 15.69 & 16.20 & 4 & 2 \\ \hline
Proteins \cite{dobson2003distinguishing, borgwardt2005protein} & 1113 & 39.06 & 72.82 & 29 & 2\\ \hline 	
Enzymes \cite{borgwardt2005protein, schomburg2004brenda}& 600 & 32.63 &	62.14 & 18  & 6\\ \hline
COIL-DEL \cite{Riesen2008} & 3900 & 21.54 &	54.24 & 2 & 100\\ \hline
\end{tabular}
}
\caption{Datasets used for \textit{graph classification}.}
\label{tab:datasets_graph}
\end{table}

\newcommand{\btable}[3]{
    \begin{table*}[h!]
        \begin{center}
            \caption{#2\label{#3}}
            \begin{tabular}{#1}
    }

\newcommand{\etable}{
    \end{tabular}
    \end{center}
    \end{table*}
}

\newcommand{\mc}[2]{\multicolumn{#1}{c}{#2}}

\btable{l|ccccc}{Results for the CiteSeer dataset (node classification task) and the AIDS dataset (graph classification task).}{tab:results_general} \hline\hline
\mc{1}{\textbf{Explainer}} 
& \mc{1}{\textbf{Validity} $\uparrow$} 
& \mc{1}{\textbf{Fidelity} $\uparrow$} 
& \mc{1}{\textbf{Distribution Distance} $\downarrow$} 
& \mc{1}{\textbf{Node Sparsity} $\downarrow$} 
& \mc{1}{\textbf{Edge Sparsity} $\downarrow$} 
\\ \hline 


 & \mc{5}{\textbf{Dataset: CiteSeer -  Model: GCNConv -  Task: Node Classification}} \\ \hline

COMBINEX$_{\textit{def}}$ &$\mathbf{1.000 (\pm 0.000)}$ & $\mathbf{0.755 (\pm 0.007)}$ & $5.620 (\pm 0.212)$ & $\uline{0.031 (\pm 0.002)}$ & $\mathbf{0.000 (\pm 0.000)}$ \\
COMBINEX$_{\textit{feat}}$ &$\mathbf{1.000 (\pm 0.000)}$ & $\mathbf{0.755 (\pm 0.007)}$ & $2.464 (\pm 0.002)$ & $\mathbf{0.001 (\pm 0.000)}$ & $n.d.(\pm n.d.)$ \\
COMBINEX$_{\textit{dyn}}$ &$\mathbf{1.000 (\pm 0.000)}$ & $0.750 (\pm 0.000)$ & $10.997 (\pm 0.201)$ & $0.095 (\pm 0.002)$ & $\mathbf{0.000 (\pm 0.000)}$ \\
COMBINEX$_{\textit{exp}}$ &$\mathbf{1.000 (\pm 0.000)}$ & $\mathbf{0.755 (\pm 0.007)}$ & $27.121 (\pm 0.111)$ & $0.313 (\pm 0.002)$ & $\mathbf{0.000 (\pm 0.000)}$ \\
COMBINEX$_{\textit{lin}}$ &$\mathbf{1.000 (\pm 0.000)}$ & $\mathbf{0.755 (\pm 0.007)}$ & $11.231 (\pm 0.300)$ & $0.090 (\pm 0.004)$ & $\mathbf{0.000 (\pm 0.000)}$ \\
COMBINEX$_{\textit{sin}}$ &$\mathbf{1.000 (\pm 0.000)}$ & $\uline{0.752 (\pm 0.003)}$ & $11.557 (\pm 0.304)$ & $0.094 (\pm 0.004)$ & $\mathbf{0.000 (\pm 0.000)}$ \\
EGO &$0.012 (\pm 0.003)$ & $0.250 (\pm 0.500)$ & $1.995 (\pm 0.206)$ & $n.d.(\pm n.d.)$ & $0.937 (\pm 0.003)$ \\
Random Edges &$0.123 (\pm 0.009)$ & $0.476 (\pm 0.099)$ & $\mathbf{1.716 (\pm 0.089)}$ & $n.d.(\pm n.d.)$ & $0.393 (\pm 0.014)$ \\
Random Features &$\uline{0.514 (\pm 0.065)}$ & ${0.721 (\pm 0.024)}$ & $32.926 (\pm 0.291)$ & $\uline{0.489 (\pm 0.000)}$ & $n.d.(\pm n.d.)$ \\
CFF &$0.010 (\pm 0.004)$ & $0.125 (\pm 0.629)$ & $2.325 (\pm 1.323)$ & $n.d.(\pm n.d.)$ & $0.594 (\pm 0.194)$ \\
CF-GNNExplainer &$0.108 (\pm 0.014)$ & $0.534 (\pm 0.103)$ & $\uline{1.783 (\pm 0.134)}$ & $n.d.(\pm n.d.)$ & $\uline{0.070 (\pm 0.022)}$ \\
UNR &$0.047 (\pm 0.012)$ & $0.202 (\pm 0.162)$ & $2.389 (\pm 0.336)$ & $n.d.(\pm n.d.)$ & $0.186 (\pm 0.068)$ \\ \hline

 & \mc{5}{\textbf{Dataset: AIDS -  Model: GCNConv -  Task: Graph Classification}} \\ \hline

COMBINEX$_{\textit{def}}$ &$\mathbf{1.000 (\pm 0.000)}$ & $\uline{0.517 (\pm 0.007)}$ & $3.665 (\pm 0.493)$ & $\mathbf{0.087 (\pm 0.009)}$ & $\mathbf{0.004 (\pm 0.003)}$ \\
EGO &$0.015 (\pm 0.008)$ & $\mathbf{0.562 (\pm 0.315)}$ & $\mathbf{2.301 (\pm 0.101)}$ & $n.d.(\pm n.d.)$ & $0.892 (\pm 0.025)$ \\
Random Edges &$\uline{0.458 (\pm 0.003)}$ & $-0.033 (\pm 0.014)$ & $\uline{2.566 (\pm 0.012)}$ & $n.d.(\pm n.d.)$ & $0.292 (\pm 0.004)$ \\
Random Features &$\mathbf{1.000 (\pm 0.000)}$ & $0.513 (\pm 0.008)$ & $24.123 (\pm 1.771)$ & $\uline{0.521 (\pm 0.022)}$ & $n.d.(\pm n.d.)$ \\
CFF &$n.d.(\pm n.d.)$ & $n.d.(\pm n.d.)$ & $n.d.(\pm n.d.)$ & $n.d.(\pm n.d.)$ & $n.d.(\pm n.d.)$ \\
CF-GNNExplainer &$\uline{0.458 (\pm 0.019)}$ & $-0.034 (\pm 0.034)$ & $2.618 (\pm 0.022)$ & $n.d.(\pm n.d.)$ & $\uline{0.076 (\pm 0.005)}$ \\

\hline

 & \mc{5}{\textbf{Dataset: CiteSeer - Model: ChebConv - Task: Node Classification} } \\ \hline

COMBINEX$_{\textit{def}}$ &$\mathbf{1.000 (\pm 0.000)}$ & $\uline{0.753 (\pm 0.022)}$ & $\uline{4.165 (\pm 0.215)}$ & $\mathbf{0.021 (\pm 0.004)}$ & $\mathbf{0.000 (\pm 0.000)}$ \\

EGO &$0.000 (\pm 0.000)$ & $n.d.(\pm n.d.)$ & $n.d.(\pm n.d.)$ & $n.d.(\pm n.d.)$ & $n.d.(\pm n.d.)$ \\
Random Edges &$0.000 (\pm 0.000)$ & $n.d.(\pm n.d.)$ & $n.d.(\pm n.d.)$ & $n.d.(\pm n.d.)$ & $n.d.(\pm n.d.)$ \\
Random Features &$\uline{0.675 (\pm 0.114)}$ & $0.742 (\pm 0.030)$ & $32.451 (\pm 0.242)$ & $\uline{0.491 (\pm 0.001)}$ & $n.d.(\pm n.d.)$ \\
CFF &$0.210 (\pm 0.019)$ & $\mathbf{0.805 (\pm 0.126)}$ & $\mathbf{2.126 (\pm 0.176)}$ & $n.d.(\pm n.d.)$ & $\uline{0.615 (\pm 0.047)}$ \\
CF-GNNExplainer &$0.000 (\pm 0.000)$ & $n.d.(\pm n.d.)$ & $n.d.(\pm n.d.)$ & $n.d.(\pm n.d.)$ & $n.d.(\pm n.d.)$ \\
UNR &$0.000 (\pm 0.000)$ & $n.d.(\pm n.d.)$ & $n.d.(\pm n.d.)$ & $n.d.(\pm n.d.)$ & $n.d.(\pm n.d.)$ \\
\hline
 & \mc{5}{\textbf{Dataset: AIDS - Model: ChebConv - Task: Graph Classification}} \\ \hline
COMBINEX$_{\textit{def}}$ &$\mathbf{1.000 (\pm 0.000)}$ & $\mathbf{0.513 (\pm 0.007)}$ & $4.772 (\pm 0.475)$ & $\mathbf{0.079 (\pm 0.002)}$ & $\mathbf{0.000 (\pm 0.000)}$ \\
EGO &$0.000 (\pm 0.000)$ & $n.d.(\pm n.d.)$ & $n.d.(\pm n.d.)$ & $n.d.(\pm n.d.)$ & $n.d.(\pm n.d.)$ \\
Random Edges &$0.240 (\pm 0.000)$ & $\uline{-0.944 (\pm 0.000)}$ & $2.759 (\pm 0.015)$ & $n.d.(\pm n.d.)$ & $\uline{0.237 (\pm 0.007)}$ \\
Random Features &$\mathbf{1.000 (\pm 0.000)}$ & $\mathbf{0.513 (\pm 0.007)}$ & $24.364 (\pm 1.838)$ & $\uline{0.520 (\pm 0.015)}$ & $n.d.(\pm n.d.)$ \\
CFF &$0.002 (\pm 0.003)$ & $-1.000 (\pm 0.000)$ & $\mathbf{2.525 (\pm 0.000)}$ & $n.d.(\pm n.d.)$ & $0.565 (\pm 0.000)$ \\
CF-GNNExplainer &$\uline{0.241 (\pm 0.002)}$ & $-0.945 (\pm 0.001)$ & $\uline{2.755 (\pm 0.018)}$ & $n.d.(\pm n.d.)$ & $\mathbf{0.000 (\pm 0.000)}$ \\
\hline

 & \mc{5}{\textbf{Dataset: CiteSeer - Model: GraphConv - Task: Node Classification}} \\ \hline

COMBINEX$_{\textit{def}}$ &$\mathbf{1.000 (\pm 0.000)}$ & $\uline{0.792 (\pm 0.010)}$ & $7.153 (\pm 0.668)$ & $\mathbf{0.055 (\pm 0.006)}$ & $\mathbf{0.000 (\pm 0.001)}$ \\

EGO &$0.005 (\pm 0.007)$ & $0.750 (\pm 0.354)$ & $\mathbf{1.232 (\pm 0.294)}$ & $n.d.(\pm n.d.)$ & $0.956 (\pm 0.055)$ \\
Random Edges &$0.076 (\pm 0.006)$ & $0.475 (\pm 0.204)$ & $1.454 (\pm 0.080)$ & $n.d.(\pm n.d.)$ & $0.440 (\pm 0.020)$ \\
Random Features &$\uline{0.481 (\pm 0.098)}$ & $0.725 (\pm 0.039)$ & $32.458 (\pm 0.231)$ & $\uline{0.491 (\pm 0.001)}$ & $n.d.(\pm n.d.)$ \\
CFF &$0.165 (\pm 0.028)$ & $\mathbf{0.867 (\pm 0.112)}$ & $2.316 (\pm 0.157)$ & $n.d.(\pm n.d.)$ & $0.570 (\pm 0.041)$ \\
CF-GNNExplainer &$0.071 (\pm 0.014)$ & $0.579 (\pm 0.053)$ & $\uline{1.353 (\pm 0.190)}$ & $n.d.(\pm n.d.)$ & $\uline{0.042 (\pm 0.031)}$ \\
UNR &$0.017 (\pm 0.007)$ & $0.375 (\pm 0.479)$ & $1.854 (\pm 0.196)$ & $n.d.(\pm n.d.)$ & $0.144 (\pm 0.068)$ \\
\hline

 & \mc{5}{\textbf{Dataset: AIDS - Model: GraphConv - Task: Graph Classification}} \\ \hline

COMBINEX$_{\textit{def}}$ &$\mathbf{1.000 (\pm 0.000)}$ & $\mathbf{0.503 (\pm 0.013)}$ & $3.767 (\pm 0.252)$ & $\mathbf{0.096 (\pm 0.006)}$ & $\mathbf{0.001 (\pm 0.000)}$ \\
EGO &$0.122 (\pm 0.011)$ & $0.426 (\pm 0.020)$ & $\mathbf{1.805 (\pm 0.072)}$ & $n.d.(\pm n.d.)$ & $0.917 (\pm 0.002)$ \\
Random Edges &$0.290 (\pm 0.018)$ & $-0.576 (\pm 0.047)$ & $\uline{2.615 (\pm 0.049)}$ & $n.d.(\pm n.d.)$ & $0.259 (\pm 0.010)$ \\
Random Features &$\uline{0.815 (\pm 0.282)}$ & $\uline{0.459 (\pm 0.114)}$ & $27.592 (\pm 1.728)$ & $\uline{0.517 (\pm 0.012)}$ & $n.d.(\pm n.d.)$ \\
CFF &$0.008 (\pm 0.006)$ & $-1.000 (\pm 0.000)$ & $10.201 (\pm 2.868)$ & $n.d.(\pm n.d.)$ & $0.278 (\pm 0.293)$ \\
CF-GNNExplainer &$0.298 (\pm 0.010)$ & $-0.621 (\pm 0.048)$ & $2.619 (\pm 0.031)$ & $n.d.(\pm n.d.)$ & $\uline{0.025 (\pm 0.002)}$ \\ \hline

\etable

\subsection{Results}\label{sec:results}

In this section, we analyze the experimental results in Table~\ref{tab:results_general}. All the other results are reported in Appendix~\ref{subsec:graphc} and \ref{subsec:nodec}.
The evaluation considers the five key metrics mentioned above: \textit{Validity}, \textit{Fidelity}, \textit{Distribution Distance}, \textit{Node Sparsity}, and \textit{Edge Sparsity}. 


The first observation that stands out is the consistently high \textit{validity} of COMBINEX, which maintains a perfect score of $1$ across all datasets and architectures. This indicates that COMBINEX can easily generate counterfactual explanations regardless for the $\alpha$ scheduling policy. In contrast, other methods exhibit much lower validity, often falling below $0.5$. Turning our attention to \textit{fidelity}, which measures how closely the generated counterfactuals align with the decision boundary, we see that COMBINEX performs remarkably well. While methods such as CFF occasionally achieve slightly higher fidelity in specific settings, they often suffer from reduced validity or increased sparsity. COMBINEX consistently ranks among the best-performing methods in fidelity, reinforcing its ability to generate explanations that are not only valid but also faithful to the underlying model.

\textit{Distribution Distance} is another important metric, as lower values indicate that the generated counterfactuals remain within the natural data distribution. Although EGO and CF-GNNExplainer achieve competitive results in some cases, their poor validity makes these results less meaningful. COMBINEX, while not always achieving the lowest distribution distance, maintains a strong balance by ensuring both validity and fidelity remain high. This balance is critical in real-world applications, where counterfactuals must not only be feasible but also realistic. When considering \textit{sparsity}, both in terms of nodes and edges, COMBINEX once again demonstrates its superiority generating explanations with minimal perturbations. 
A deeper comparison with baseline methods reveals further insights. Random Features, for instance, occasionally achieves high validity, but its counterfactuals are highly unrealistic, as indicated by their excessively high distribution distance. 
Random Edges, instead, tends to perform poorly in fidelity and sparsity, demonstrating that randomly modifying graph structures does not produce meaningful counterfactual explanations. Among structured methods, CF-GNNExplainer exhibits relatively low distribution distance and reasonable edge sparsity, but its lower validity and fidelity scores limit its overall usefulness. CFF, on the other hand, achieves the highest fidelity in some cases, but at the cost of poor validity and increased sparsity. This indicates that while CFF can produce highly faithful explanations, they are often unrealistic or overly complex.

EGO, while getting a low distribution distance in some instances, suffers from extremely low validity scores.

Overall, we attribute these outstanding results to our method's ability to optimally balance node feature and edge perturbations, leading to superior counterfactual explanations.

\subsection{The Impact of the Scheduling Policy $\alpha$}
\label{subsec:alpha}
The results with different $\alpha$ values reported in Table~\ref{tab:results_general} are shown only for a single dataset, model, and task due to space constraints. For the full results, refer to the Appendix~\ref{subsec:graphc} and \ref{subsec:nodec}. 

Our experiments (see Table~\ref{tab:results_general}) demonstrate that the choice of $\alpha$ in the COMBINEX framework influences performance metrics. The constant policy (\textit{def}), which maintains a fixed value for $\alpha$ throughout the optimization process, achieves the lowest distribution distance (5.620) and one of the lowest node sparsity values (0.031), indicating a more controlled perturbation that preserves the original graph structure. Similarly, the feature-only variant (\textit{feat}) achieves a comparably high fidelity (0.755) while minimizing node sparsity (0.001), suggesting that altering only node features without modifying edges leads to minimal changes while maintaining counterfactual validity.

Conversely, the exponential policy (\textit{exp}) leads to the highest distribution distance ($27.121$) and a notable increase in node sparsity ($0.313$), reflecting the aggressive perturbations caused by the rapidly decaying $\alpha$. In contrast, the dynamic, linear, and sinusoidal policies achieve intermediate distribution distances ($10.997$–$11.557$) and maintain node sparsity around $0.090$–$0.095$, suggesting a more balanced trade-off between modification extent and stability.

These findings, along with the others presented in Appendix~\ref{subsec:graphc} and \ref{subsec:nodec} highlight that the impact of $\alpha$ is highly dependent on both the scheduling policy, the oracle model, and the dataset characteristics. The optimal choice of $\alpha$ should therefore be carefully tuned based on the complexity of the graph data and the interpretability objectives of the counterfactual explanations.

\section{Feasibility of our Method}\label{sec:feasibility}

In this section, we discuss the feasibility of our method in comparison with the baselines. Our edge sparsification process uses the adjacency matrix perturbation approach introduced by Lucic et al.~\cite{lucic2022cf}, while addressing scalability challenges inherent to their method. Specifically, Lucic et al. represent edges using a sparse matrix format, which significantly limits the applicability of their algorithm to graphs containing more than 30--35 nodes. 
In Table \ref{tab:citeseer_time_memory}, for example, we compare the execution time and the memory needed for each explainer on the Citeseer dataset. 
\begin{table}[h]
\centering
\begin{tabular}{lcc}
\hline
\textbf{Explainer} & \textbf{Time (s) ($\pm$ std)} & \textbf{Memory (MB) ($\pm$ std)} \\
\hline
COMBINEX & $7.665 (\pm 0.245)$ & $1583.035 (\pm 97.627)$ \\
CF-GNNExplainer & $45.576 (\pm 1.306)$ & $2045.919 (\pm 88.096)$ \\
Random Features & $1.287 (\pm 0.025)$ & $1387.031 (\pm 151.101)$ \\
Random Edges & $1.350 (\pm 0.017)$ & $1359.386 (\pm 144.630)$ \\
EGO & $ 0.011 (\pm 0.002)$ & $1022.579 (\pm 14.825)$ \\
CFF & $ 8.936 (\pm 0.204)$ & $1469.579 (\pm 64.606)$ \\
UNR & $ 0.266 (\pm 0.327)$ & $ 1221.638 (\pm 99.061)$ \\
\hline
\end{tabular}
\caption{Execution time and memory usage for different explainers on the Citeseer dataset.}
\label{tab:citeseer_time_memory}
\end{table}
To overcome this limitation, we leverage a novel strategy that exploits the edge weight vector, which can be seamlessly integrated into various graph convolutional layers, such as \textit{ChebConv}, \textit{GCNConv}, and \textit{GraphConv}. Given a graph $ G(V, E) $ with $ |V| = n $, our approach introduces a perturbation vector $ \mathbf{EP}^{n\times 1} $ that effectively cancels out edges by feeding it into the GNN. The effectiveness of our edge weight sparsification technique is evident when analyzing the execution time results in Table \ref{tab:citeseer_time_memory}. Notably, CF-GNNExplainer requires $45.576$ seconds on average to generate explanations, whereas COMBINEX completes the same process in just $7.665$ seconds. This substantial improvement in runtime efficiency highlights the scalability advantage introduced by our sparsification technique. By reducing the computational overhead associated with handling edge deletions, COMBINEX maintains high explainability performance while significantly lowering execution time. 

We formally demonstrate that this technique is equivalent to performing edge deletion using the full adjacency matrix. For \textit{GCNConv}, this equivalence is trivial and follows directly from Theorem~\ref{theorem:gcn}. However, for \textit{ChebConv}, we establish that when the filter size is set to 1, edge nullification via the edge weight vector remains consistent with the behavior observed in \textit{GCNConv}. Conversely, when the filter size exceeds 1, such equivalence cannot be guaranteed (see Appendix~\ref{theorem:cheb}). The proof for GraphConv can be found in Appendix \ref{theorem:graph}.

\begin{theorem} \label{theorem:gcn}[Equivalence of Edge Weight Nullification and Adjacency Matrix Edge Removal in GCNs]
Let $ G = (V, E) $ be a graph with $ n $ nodes and $ m $ edges. Let $ \mathbf{X} \in \mathbb{R}^{n \times d} $ be the node feature matrix, where each node has a feature vector of dimension $ d $. Consider a \textit{GCNConv} layer parameterized by a weight matrix $ \mathbf{W} \in \mathbb{R}^{d \times d'} $. Setting an edge weight to zero in the GCNConv’s edge weight vector is equivalent to removing the corresponding edge in the adjacency matrix representation.
\end{theorem}

\begin{proof}
We represent the edges of the graph using an edge index matrix $ \mathbf{E} \in \mathbb{R}^{2 \times m} $, where each column corresponds to an edge with source and target node indices. Let $ \mathbf{EP} \in \mathbb{R}^{1\times m} $ be the vector of edge weights corresponding to the edges in $ \mathbf{E} $. To include self-loops, we update the edge index matrix to $ \mathbf{E}' $ and the edge weight vector to $ \mathbf{EP}' $. Define the degree matrix $ \mathbf{D} \in \mathbb{R}^{n \times n} $ where $D_{ii} = \sum_{j} A_{ij},$ with $ A_{ij} $ is the adjacency matrix element corresponding to edge $ (i, j) $. The normalized adjacency matrix incorporating edge weights is computed as: $\mathbf{\tilde{A}} = \mathbf{D}^{-1/2} \mathbf{A} \mathbf{D}^{-1/2}.$ The output feature matrix at the GCNConv layer is given by: $ \mathbf{H} = \sigma\left( \mathbf{\tilde{A}} \mathbf{X} \mathbf{W} \right),$ where $ \sigma $ is an activation function. Alternatively, the node-wise update rule for node $ i $ can be expressed as:
\begin{equation}\label{eq:mp}
\mathbf{h}^{\prime}_i = \sigma \left( \mathbf{W}^{\top} \sum_{j \in
\mathcal{N}(i) \cup \{ i \}} \frac{e_{j,i}}{\sqrt{\hat{d}_j \hat{d}_i}} \mathbf{x}_j \right).
\end{equation}

\noindent{\textbf{\textit{Equivalence Analysis:}}}
Setting the $ k $-th edge weight $ e_k = 0 $ in $ \mathbf{EP}' $ removes the contribution of the corresponding edge $ (i, j) $ in the message-passing process (Equation \ref{eq:mp}). Since edge weights scale the aggregated messages, setting $ e_k = 0 $ nullifies the corresponding contribution.
Explicitly removing edge $ (i, j) $ from the adjacency matrix sets $ A_{ij} = 0 $, ensuring that node $ j $ no longer contributes to the feature update of node $ i $. To formalize the equivalence, let:
\begin{itemize}
    \item $ \mathbf{A}' $ be the adjacency matrix after removing edge $ (i, j) $.
    \item $ \mathbf{EP}'' $ be the edge weight vector where $ e_k = 0 $ for the corresponding edge.
    \item  $ \mathbf{\tilde{A}}' $ be the normalized adjacency matrix computed from $ \mathbf{A}' $.
    \item $ \mathbf{\tilde{A}}_{\text{ew}} $ be the normalized adjacency matrix computed using the edge weight vector $ \mathbf{e}'' $.
\end{itemize}

Since both methods eliminate the contribution of edge $ (i, j) $, we obtain $\mathbf{\tilde{A}}' = \mathbf{\tilde{A}}_{\text{ew}}$. Thus, the resulting feature updates remain identical:
\begin{equation*}
  \mathbf{H}' = \sigma ( \mathbf{\tilde{A}}' \mathbf{X} \mathbf{W} ) = \sigma (\mathbf{\tilde{A}}_{\text{ew}} \mathbf{X} \mathbf{W}).  
\end{equation*}

This confirms that setting an edge weight to zero is mathematically equivalent to removing the corresponding edge in the adjacency matrix.
\end{proof}

\section{Conclusion and Future Work}
\label{sec:conclusion}

In this work, we introduced COMBINEX, a unified counterfactual explainer for Graph Neural Networks (GNNs) that integrates both node feature and structural perturbations. Through extensive experiments across various datasets, tasks, and architectures, we demonstrated that COMBINEX effectively balances key evaluation metrics, ensuring high validity while minimizing modifications to the graph structure and node features to maintain realism.

We also proposed a novel edge weight sparsification technique, which significantly improves computational efficiency without compromising explainability. Our comparative analysis showed that COMBINEX operates more efficiently and with lower computational costs than existing methods following a similar approach, such as CF-GNNExplainer. Additionally, we explored different scheduling policies for balancing node and edge perturbations, further highlighting the flexibility and generalizability of COMBINEX across diverse scenarios.

In summary, COMBINEX represents a state-of-the-art counterfactual explanation framework for GNNs, offering a comprehensive and computationally efficient approach that aligns with real-world interpretability requirements. 

Future work will focus on extending our approach to broader graph-based tasks, including link prediction, as well as adapting the counterfactual framework to dynamic and heterogeneous graph structures.

\bibliographystyle{ACM-Reference-Format}
\bibliography{sample-base.bib}
\

\appendix

\section{Appendix}

\subsection{Edge Nullification Theorem for ChebConv}
\begin{theorem}\label{theorem:cheb}
In a Chebyshev Convolutional Network (ChebConv), setting an edge weight to zero is equivalent to removing the corresponding edge from the graph's adjacency matrix for $K=1$. However, for $K>1$, this equivalence does not necessarily hold.
\end{theorem}

\begin{proof}

The ChebConv layer applies a Chebyshev polynomial filter to the graph Laplacian. The output feature matrix $\mathbf{H}$ is computed as:
\begin{equation}
    \mathbf{H} = \sum_{k=0}^{K-1} \mathbf{T}_k(\tilde{\mathbf{L}}) \mathbf{X} \mathbf{\Theta}_k,
\end{equation}
where: $K$ is the Chebyshev filter size, $\mathbf{T}_k(\tilde{\mathbf{L}})$ is the Chebyshev polynomial of order $k$ evaluated at the scaled Laplacian $\tilde{\mathbf{L}}$, $\mathbf{X}$ is the input node feature matrix, $\mathbf{\Theta}_k$ is the learnable weight matrix for the $k$-th order. The scaled and normalized Laplacian $\tilde{\mathbf{L}}$ is defined as:
\begin{equation}
    \tilde{\mathbf{L}} = \frac{2\mathbf{L}}{\lambda_{\max}} - \mathbf{I},
\end{equation}
where $\mathbf{L} = \mathbf{D} - \mathbf{A}$ is the unnormalized Laplacian, $\mathbf{D}$ is the degree matrix, $\mathbf{A}$ is the adjacency matrix, and $\lambda_{\max}$ is the largest eigenvalue of $\mathbf{L}$. The Chebyshev polynomials are computed recursively as:
\begin{align}
    \mathbf{T}_0(\tilde{\mathbf{L}}) &= \mathbf{I}, \\
    \mathbf{T}_1(\tilde{\mathbf{L}}) &= \tilde{\mathbf{L}}, \\
    \mathbf{T}_k(\tilde{\mathbf{L}}) &= 2\tilde{\mathbf{L}} \mathbf{T}_{k-1}(\tilde{\mathbf{L}}) - \mathbf{T}_{k-2}(\tilde{\mathbf{L}}), \quad k \geq 2.
\end{align}

\textbf{Case $K=1$:}
For $K=1$, ChebConv simplifies to a first-order approximation similar to GCN. Setting an edge weight $e_{ij}$ to zero directly removes its contribution in the message passing, making it equivalent to removing the edge. It is important to notice that $e_{ij}$ is included within the adjacency matrix $\mathbf{A}$.

\textbf{Case $K>1$:}
For $K>1$, ChebConv introduces dependencies on multiple-hop neighbors due to higher-order polynomial terms. Even if an edge weight $e_{ij}$ is set to zero, information may still propagate through alternative paths in $\mathbf{T}_k(\tilde{\mathbf{L}})$. Specifically, for $K=2$:
\begin{equation}
    \mathbf{T}_2(\tilde{\mathbf{L}}) = 2\tilde{\mathbf{L}}^2 - \mathbf{I}.
\end{equation}
This squared term allows second-order neighbors to contribute, preventing a strict equivalence between weight nullification and edge removal.

Thus, for $K=1$, the equivalence holds, but for $K>1$, setting an edge weight to zero does not necessarily remove all contributions from that edge in ChebConv.
\end{proof}

\subsection{Edge Nullification Theorem for GraphConv}\label{theorem:graph}

\begin{theorem}
Let $ G = (V, E) $ be a graph with $ n $ nodes and $ m $ edges. Let $ \mathbf{X} \in \mathbb{R}^{n \times d} $ denote the node feature matrix, where each node has a feature vector of dimension $ d $. Consider a Graph Convolutional Network (GraphConv) layer parameterized by a weight matrix $ \mathbf{W} \in \mathbb{R}^{d \times d'} $. Setting an edge weight to zero in the GraphConv's edge weight vector is equivalent to removing the corresponding edge in the adjacency matrix representation.
\end{theorem}

\begin{proof}
The GraphConv layer follows the message-passing framework:
\begin{equation}
    \mathbf{H} = \sigma \left( (\mathbf{D}^{-1} \mathbf{A}) \mathbf{X} \mathbf{W} + \mathbf{X} \mathbf{W} \right),
\end{equation}
where:
- $ \mathbf{A} $ is the adjacency matrix (including self-loops).
- $ \mathbf{D} $ is the degree matrix with $ D_{ii} = \sum_{j} A_{ij} $.
- $ \sigma $ is an activation function.
- $ \mathbf{W} $ is the weight matrix.
- $ \mathbf{X} $ is the input node feature matrix.
- The term $ \mathbf{D}^{-1} \mathbf{A} $ represents row-normalized message aggregation.

Expanding the node-wise update, the representation for node $ i $ is:
\begin{equation}
    \mathbf{h}^{\prime}_i = \sigma \left( \mathbf{W}^{\top} \sum_{j \in \mathcal{N}(i) \cup \{ i \}} \frac{e_{j,i}}{d_j} \mathbf{x}_j + \mathbf{W}^{\top} \mathbf{x}_i \right).
\end{equation}

Setting the edge weight $ e_{i,j} = 0 $ removes the contribution of $ \mathbf{x}_j $ from the summation, effectively eliminating the influence of node $ j $ on node $ i $. On the other hand, removing edge $ (i, j) $ from the adjacency matrix $ \mathbf{A} $ sets $ A_{ij} = 0 $, which in turn removes $ x_j $ from the summation in the message-passing step. Since both approaches lead to the same feature update, the equivalence holds:
\begin{equation}
    \mathbf{H}' = \sigma \left( (\mathbf{D}'^{-1} \mathbf{A}') \mathbf{X} \mathbf{W} + \mathbf{X} \mathbf{W} \right) = \sigma \left( (\mathbf{D}^{-1} \mathbf{A}'') \mathbf{X} \mathbf{W} + \mathbf{X} \mathbf{W} \right).
\end{equation}

Thus, setting $ e_k = 0 $ in the edge weight vector produces the same feature transformation as removing the edge $ (i, j) $ from $ \mathbf{A} $, proving the claim.
\end{proof}

\subsection{Parameters}\label{subsec:params}

The experimental configurations for our models were set as follows. For the ChebConv model, we employed three hidden layers each with 64 units, along with a dropout rate of 0.5 to mitigate overfitting, and set the Chebyshev filter order \(K\) to 1. In contrast, the GCNConv model was configured with three hidden layers, each containing 128 units, and a dropout of 0.5. For the GraphConv model, the architecture comprised three hidden layers with 64 units each and also utilized a dropout rate of 0.5.

\subsection{Functions}
Algorithm~\ref{alg:get_alpha} presents our \emph{Alpha Scheduling Function}, which determines the value of $\alpha$—a weighting parameter—during training based on the current epoch, loss values, and a specified scheduling policy. The function accepts as inputs the current epoch number, the edge loss $\mathcal{L}_E$, the node loss $\mathcal{L}_X$, a default value $\alpha_{default}$, the scheduling policy ($policy$), a decay rate $\delta$, and the maximum number of epochs $epochs_{max}$. Depending on the selected policy, the function computes $\alpha$ according to one of several strategies:

\begin{itemize}
    \item \textbf{Linear:} When $policy = linear$, $\alpha$ decreases linearly with the epoch number. This is computed as $\alpha = \max(0.0, 1.0 - \frac{epoch}{epochs_{max}})$, ensuring that $\alpha$ gradually decays from 1 to 0 over the course of training.
    \item \textbf{Exponential:} For $policy = exponential$, an exponential decay is applied: $\alpha = \max(0.0, e^{-epoch / \delta})$. The decay rate $\delta$ controls how fast $\alpha$ decays, allowing for rapid reduction at early epochs if desired.
    \item \textbf{Sinusoidal:} When $policy = sinusoidal$, the function uses a cosine-based schedule: $\alpha = \max(0.0, 0.5 \times (1 + \cos(\pi \times \frac{epoch}{epochs_{max}})))$. This policy produces a periodic decay that may help in scenarios where a smooth cyclic modulation of $\alpha$ is beneficial.
    \item \textbf{Dynamic:} If $policy = dynamic$, the scheduling is determined by comparing the edge loss and the node loss. Specifically, if $\mathcal{L}_E > \mathcal{L}_X$, then $\alpha$ is set to 0.0; otherwise, it is set to 1.0. This policy allows the training process to adaptively prioritize either edge or node information based on the relative magnitude of their losses.
    \item \textbf{Default:} In all other cases, the function returns a pre-specified default value $\alpha_{default}$.
\end{itemize}

This flexible scheduling mechanism is crucial for balancing different loss components during training, and its design allows for easy experimentation with various decay strategies. By incorporating both fixed (linear, exponential, sinusoidal) and adaptive (dynamic) policies, the function ensures that $\alpha$ can be tuned to optimize the trade-off between edge and node losses under different training conditions.

\begin{algorithm}[h]
\SetAlgoLined
\caption{Alpha Scheduling Function}\label{alg:get_alpha}
\KwData{$epoch$: Current epoch number, $\mathcal{L}_E$: Edge loss value, $\mathcal{L}_X$: Node loss value, $\alpha_{default} $: default value, $policy$: the scheduling policy, $\delta$: decay rate, $epochs_{max}$: maximum number of epochs }
\KwOut{$\alpha$: Scheduled alpha value}

    \uIf{$policy = linear$}{
        $\alpha \gets \max(0.0, 1.0 - \frac{epoch}{epochs_{max}})$
    }
    \uElseIf{$policy = exponential$}{
        $\alpha \gets \max(0.0, e^{-epoch / \delta})$
    }
    \uElseIf{$policy = sinusoidal$}{
        $\alpha \gets \max(0.0, 0.5 \times (1 + \cos(\pi \times \frac{epoch}{epochs_{max}})))$
    }
    \uElseIf{$policy = dynamic$}{
        \uIf{$\mathcal{L}_E > \mathcal{L}_X$}{
            $\alpha \gets 0.0$
        }
        \Else{
            $\alpha \gets 1.0$
        }
    }
    \Else{
        $\alpha \gets \alpha_{default}$
    }
    return $\alpha$

\end{algorithm}

\subsection{Datasets}\label{sec:dataset}

\subsubsection{Planetoid}  
We use three citation network datasets: CiteSeer, Cora, and PubMed. CiteSeer contains 3,312 scientific publications across six classes, with a citation network of 4,732 links. Each document is represented by a binary word vector from a 3,703-word dictionary. Cora includes 2,708 publications classified into seven categories, with 5,429 citation links and binary word vectors from a 1,433-word dictionary. PubMed comprises 19,717 diabetes-related publications, categorized into three classes, with 44,338 citation links. Each document is represented using TF-IDF word vectors from a 500-word dictionary.
\subsubsection{WebKB} The WebKB datasets represent webpages collected from computer science departments of various universities. Our work uses three datasets: Cornell, Texas, and Wisconsin. Each dataset is a graph where nodes represent web pages, and edges are hyperlinks between them. Node features are the bag-of-words representation of web pages. The web pages are manually classified into the five categories, student, project, course, staff, and faculty.

\subsubsection{Attributed} The Attributed category contains three datasets: Wiki, and Facebook. The Wiki dataset comprises web pages, nodes, and edges representing hyperlinks between them. Node features represent several informative nouns on the Wikipedia pages. The Facebook dataset instead is a graph representing relations between users. In particular, the dataset contains profile and network data from 10 ego-networks, consisting of 193 circles and 4,039 users.

\subsubsection{Biological} The AIDS dataset contains 2,000 molecular graphs from the AIDS Antiviral Screen Database, used to study anti-HIV activity. The Enzymes dataset includes 600 protein structures classified into six enzyme classes, with nodes representing secondary structure elements and edges indicating their interactions. The Proteins dataset consists of 1,113 protein structures labeled as enzymes or non-enzymes, where nodes are amino acids, and edges connect those within 6 angstroms. Originally designed for graph classification, we adapted these datasets for node classification tasks.

\subsubsection{COIL-DEL} The COIL-DEL dataset contains 3,900 graphs, each representing a 2D image from COIL-100. Each graph corresponds to one of 100 objects, with 39 images per object captured from different angles. Nodes represent superpixels with 2D feature vectors, and edges denote spatial relationships. On average, graphs have 21.54 nodes and 54.24 edges. This dataset is used for graph-based machine learning in object recognition and image classification.

\subsubsection{Miscellaneous}
This category encompasses two different datasets: Karate and Actors. The Karate dataset contains 34 nodes connected by 156 (undirected and unweighted) edges. Every node is labeled by one of four classes obtained via modularity-based clustering.\\
In the Actor dataset, instead, each node corresponds to an actor, and the edge between two nodes denotes co-occurrence on the same Wikipedia page. Node features correspond to keywords in Wikipedia pages associated with the actors. The task is to classify the nodes into five categories.

\subsection{Graph Classification Results}\label{subsec:graphc}

The results are presented in Tables~\ref{tab:graph_cheb}, \ref{tab:graph_graph}, and \ref{tab:graph_gcn}.
COMBINEX consistently attains the best validity scores across all datasets and configurations, whereas baseline methods such as CF-GNNExplainer and EGO often exhibit significantly lower validity scores. 
While fidelity varies depending on the specific scheduling policy, COMBINEX remains competitive and, in many cases, outperforms traditional methods. The results suggest that COMBINEX generates explanations that remain faithful to the original model's decision boundary while introducing minimal but effective perturbations. One of the key advantages of COMBINEX is its ability to maintain a reasonable distribution distance. Unlike some baselines that introduce drastic changes leading to unrealistic counterfactuals, COMBINEX ensures that the generated explanations remain close to the original data distribution, enhancing their interpretability. 
Both node and edge sparsity are crucial for producing interpretable counterfactuals. COMBINEX effectively minimizes modifications, preserving the underlying graph structure while ensuring that only necessary changes are introduced. This makes it a more controlled and interpretable approach compared to methods that introduce excessive perturbations.
Traditional methods often exhibit trade-offs between different metrics, struggling to balance validity, fidelity, and sparsity simultaneously. EGO, for instance, sometimes achieves competitive distribution distances but at the cost of poor validity and high edge modifications. CF-GNNExplainer, on the other hand, frequently underperforms in validity and fidelity, limiting its reliability. Random perturbation-based approaches lead to large distribution shifts, making the generated counterfactuals less meaningful.

\btable{l|ccccc}{Results for Graph Classification datasets: AIDS, Proteins, Enzymes, Coil-del. The oracles $\Phi$ use ChebConv layers . In \textbf{bold} the best result, the second best result is \underline{underlined}}{tab:graph_cheb}  \hline\hline

& \mc{1}{\textbf{Validity} $\uparrow$} 
& \mc{1}{\textbf{Fidelity} $\uparrow$} 
& \mc{1}{\textbf{Distribution Distance} $\downarrow$} 
& \mc{1}{\textbf{Node Sparsity} $\downarrow$} 
& \mc{1}{\textbf{Edge Sparsity} $\downarrow$} 
\\ 

\textbf{Explainers} 
& \textit{mean($\pm$std)}
& \textit{mean($\pm$std)}
& \textit{mean($\pm$std)}
& \textit{mean($\pm$std)}
& \textit{mean($\pm$std)} \\ \hline

& \mc{5}{\textbf{Dataset: AIDS}} \\ 
\hline

COMBINEX$_{\textit{feat}}$ &$\mathbf{1.000 (\pm 0.000)}$ & $0.513 (\pm 0.007)$ & $4.546 (\pm 0.494)$ & $0.086 (\pm 0.001)$ & $n.d.(\pm n.d.)$ \\
COMBINEX$_{\textit{def}}$ &$\mathbf{1.000 (\pm 0.000)}$ & $0.513 (\pm 0.007)$ & $4.772 (\pm 0.475)$ & $\mathbf{0.079 (\pm 0.002)}$ & $\mathbf{0.000 (\pm 0.000)}$ \\
COMBINEX$_{\textit{dyn}}$ &$\mathbf{1.000 (\pm 0.000)}$ & $\mathbf{0.517 (\pm 0.006)}$ & $5.117 (\pm 0.528)$ & $\uline{0.081 (\pm 0.004)}$ & $\mathbf{0.000 (\pm 0.000)}$ \\
COMBINEX$_{\textit{exp}}$ &$\mathbf{1.000 (\pm 0.000)}$ & $\uline{0.515 (\pm 0.007)}$ & $6.316 (\pm 0.333)$ & $0.269 (\pm 0.013)$ & $\mathbf{0.000 (\pm 0.000)}$ \\
COMBINEX$_{\textit{lin}}$ &$\mathbf{1.000 (\pm 0.000)}$ & $\mathbf{0.517 (\pm 0.006)}$ & $5.455 (\pm 0.502)$ & $0.087 (\pm 0.010)$ & $\mathbf{0.000 (\pm 0.000)}$ \\
COMBINEX$_{\textit{sin}}$ &$\mathbf{1.000 (\pm 0.000)}$ & $\uline{0.515 (\pm 0.007)}$ & $4.976 (\pm 0.213)$ & $0.087 (\pm 0.012)$ & $\mathbf{0.000 (\pm 0.000)}$ \\
EGO &$0.000 (\pm 0.000)$ & $n.d.(\pm n.d.)$ & $n.d.(\pm n.d.)$ & $n.d.(\pm n.d.)$ & $n.d.(\pm n.d.)$ \\
Random Edges &$0.240 (\pm 0.000)$ & $-0.944 (\pm 0.000)$ & $2.759 (\pm 0.015)$ & $n.d.(\pm n.d.)$ & ${0.237 (\pm 0.007)}$ \\
Random Features &$\mathbf{1.000 (\pm 0.000)}$ & $0.513 (\pm 0.007)$ & $24.364 (\pm 1.838)$ & $0.520 (\pm 0.015)$ & $n.d.(\pm n.d.)$ \\
CFF &$0.002 (\pm 0.003)$ & $-1.000 (\pm 0.000)$ & $\mathbf{2.525 (\pm 0.000)}$ & $n.d.(\pm n.d.)$ & $0.565 (\pm 0.000)$ \\
CF-GNNExplainer &$\uline{0.241 (\pm 0.002)}$ & $-0.945 (\pm 0.001)$ & $\uline{2.755 (\pm 0.018)}$ & $n.d.(\pm n.d.)$ & $\uline{0.178 (\pm 0.002)}$ \\
\hline
& \mc{5}{\textbf{Dataset: Proteins}} \\ 
\hline
COMBINEX$_{\textit{feat}}$ &$\mathbf{1.000 (\pm 0.000)}$ & $\uline{-0.487 (\pm 0.091)}$ & $1566.374 (\pm 194.517)$ & $0.397 (\pm 0.028)$ & $n.d.(\pm n.d.)$ \\
COMBINEX$_{\textit{def}}$ &$\mathbf{1.000 (\pm 0.000)}$ & $-0.513 (\pm 0.074)$ & $1339.333 (\pm 216.302)$ & $\uline{0.359 (\pm 0.017)}$ & $\mathbf{0.000 (\pm 0.000)}$ \\
COMBINEX$_{\textit{dyn}}$ &$\mathbf{1.000 (\pm 0.000)}$ & $\uline{-0.487 (\pm 0.083)}$ & $1492.948 (\pm 160.720)$ & $0.367 (\pm 0.019)$ & $\mathbf{0.000 (\pm 0.000)}$ \\
COMBINEX$_{\textit{exp}}$ &$\mathbf{1.000 (\pm 0.000)}$ & $\mathbf{-0.477 (\pm 0.079)}$ & $1548.974 (\pm 120.839)$ & $0.415 (\pm 0.025)$ & $\mathbf{0.000 (\pm 0.000)}$ \\
COMBINEX$_{\textit{lin}}$ &$\mathbf{1.000 (\pm 0.000)}$ & $-0.540 (\pm 0.069)$ & $1432.565 (\pm 130.830)$ & $\mathbf{0.357 (\pm 0.013)}$ & $\mathbf{0.000 (\pm 0.000)}$ \\
COMBINEX$_{\textit{sin}}$ &$\mathbf{1.000 (\pm 0.000)}$ & $-0.493 (\pm 0.079)$ & $1454.851 (\pm 181.701)$ & $0.367 (\pm 0.017)$ & $\mathbf{0.000 (\pm 0.000)}$ \\
EGO &$0.000 (\pm 0.000)$ & $n.d.(\pm n.d.)$ & $n.d.(\pm n.d.)$ & $n.d.(\pm n.d.)$ & $n.d.(\pm n.d.)$ \\
Random Edges &$0.612 (\pm 0.024)$ & $-0.810 (\pm 0.030)$ & $\uline{602.170 (\pm 7.114)}$ & $n.d.(\pm n.d.)$ & ${0.378 (\pm 0.008)}$ \\
Random Features &$\uline{0.970 (\pm 0.051)}$ & $-0.533 (\pm 0.085)$ & $5449.138 (\pm 272.833)$ & $0.917 (\pm 0.001)$ & $n.d.(\pm n.d.)$ \\
CFF &$0.143 (\pm 0.009)$ & $-1.000 (\pm 0.000)$ & $767.554 (\pm 9.537)$ & $n.d.(\pm n.d.)$ & $0.604 (\pm 0.039)$ \\
CF-GNNExplainer &$0.598 (\pm 0.008)$ & $-0.850 (\pm 0.045)$ & $\mathbf{601.087 (\pm 11.168)}$ & $n.d.(\pm n.d.)$ & $\uline{0.021 (\pm 0.017)}$ \\

\hline
& \mc{5}{\textbf{Dataset: Enzymes}} \\ 
\hline

COMBINEX$_{\textit{feat}}$ &$\mathbf{1.000 (\pm 0.000)}$ & $-0.110 (\pm 0.043)$ & $30.329 (\pm 0.884)$ & $0.524 (\pm 0.037)$ & $n.d.(\pm n.d.)$ \\
COMBINEX$_{\textit{def}}$ &$\mathbf{1.000 (\pm 0.000)}$ & $-0.098 (\pm 0.043)$ & $30.511 (\pm 1.066)$ & $0.390 (\pm 0.021)$ & $\mathbf{0.000 (\pm 0.000)}$ \\
COMBINEX$_{\textit{dyn}}$ &$\mathbf{1.000 (\pm 0.000)}$ & $-0.100 (\pm 0.045)$ & $30.354 (\pm 0.928)$ & $0.393 (\pm 0.017)$ & $\mathbf{0.000 (\pm 0.000)}$ \\
COMBINEX$_{\textit{exp}}$ &$\mathbf{1.000 (\pm 0.000)}$ & $\mathbf{-0.094 (\pm 0.051)}$ & $31.740 (\pm 1.192)$ & $0.640 (\pm 0.042)$ & $\mathbf{0.000 (\pm 0.000)}$ \\
COMBINEX$_{\textit{lin}}$ &$\mathbf{1.000 (\pm 0.000)}$ & $\uline{-0.095 (\pm 0.040)}$ & $30.716 (\pm 1.127)$ & $\mathbf{0.374 (\pm 0.019)}$ & $\mathbf{0.000 (\pm 0.000)}$ \\
COMBINEX$_{\textit{sin}}$ &$\mathbf{1.000 (\pm 0.000)}$ & $-0.097 (\pm 0.057)$ & $30.697 (\pm 0.935)$ & $\uline{0.375 (\pm 0.020)}$ & $\mathbf{0.000 (\pm 0.000)}$ \\
EGO &$0.000 (\pm 0.000)$ & $n.d.(\pm n.d.)$ & $n.d.(\pm n.d.)$ & $n.d.(\pm n.d.)$ & $n.d.(\pm n.d.)$ \\
Random Edges &$0.244 (\pm 0.039)$ & $-0.433 (\pm 0.140)$ & $15.179 (\pm 1.349)$ & $n.d.(\pm n.d.)$ & $\uline{0.365 (\pm 0.010)}$ \\
Random Features &$\uline{0.767 (\pm 0.033)}$ & $-0.127 (\pm 0.036)$ & $366.058 (\pm 21.408)$ & $0.894 (\pm 0.001)$ & $n.d.(\pm n.d.)$ \\
CFF &$0.148 (\pm 0.047)$ & $-0.760 (\pm 0.200)$ & $\mathbf{14.476 (\pm 2.734)}$ & $n.d.(\pm n.d.)$ & $0.594 (\pm 0.060)$ \\
CF-GNNExplainer &$0.240 (\pm 0.035)$ & $-0.460 (\pm 0.123)$ & $\uline{15.054 (\pm 1.666)}$ & $n.d.(\pm n.d.)$ & ${0.456 (\pm 0.090)}$ \\
\hline
& \mc{5}{\textbf{Dataset: Coil-del}} \\ 
\hline

COMBINEX$_{\textit{feat}}$ &$\uline{0.971 (\pm 0.010)}$ & $-0.001 (\pm 0.005)$ & $\uline{17.231 (\pm 1.031)}$ & $0.989 (\pm 0.001)$ & $n.d.(\pm n.d.)$ \\
COMBINEX$_{\textit{def}}$ &$0.963 (\pm 0.021)$ & $-0.002 (\pm 0.006)$ & $18.062 (\pm 0.609)$ & $0.990 (\pm 0.001)$ & $\mathbf{0.000 (\pm 0.000)}$ \\
COMBINEX$_{\textit{dyn}}$ &$0.969 (\pm 0.018)$ & $\uline{-0.001 (\pm 0.005)}$ & $17.323 (\pm 1.296)$ & $0.990 (\pm 0.002)$ & $\mathbf{0.000 (\pm 0.000)}$ \\
COMBINEX$_{\textit{exp}}$ &$\mathbf{0.973 (\pm 0.017)}$ & $-0.002 (\pm 0.003)$ & $26.106 (\pm 0.907)$ & $\mathbf{0.989 (\pm 0.003)}$ & $\mathbf{0.000 (\pm 0.000)}$ \\
COMBINEX$_{\textit{lin}}$ &$0.968 (\pm 0.016)$ & $-0.002 (\pm 0.003)$ & $25.711 (\pm 0.962)$ & $\uline{0.989 (\pm 0.002)}$ & $\mathbf{0.000 (\pm 0.000)}$ \\
COMBINEX$_{\textit{sin}}$ &$0.968 (\pm 0.009)$ & $\mathbf{-0.001 (\pm 0.006)}$ & $25.840 (\pm 1.134)$ & $0.990 (\pm 0.002)$ & $\mathbf{0.000 (\pm 0.000)}$ \\
EGO &$0.000 (\pm 0.000)$ & $n.d.(\pm n.d.)$ & $n.d.(\pm n.d.)$ & $n.d.(\pm n.d.)$ & $n.d.(\pm n.d.)$ \\
Random Edges &$0.006 (\pm 0.002)$ & $-0.714 (\pm 0.488)$ & $24.004 (\pm 4.534)$ & $n.d.(\pm n.d.)$ & ${0.302 (\pm 0.028)}$ \\
Random Features &$0.088 (\pm 0.031)$ & $-0.036 (\pm 0.040)$ & $47.541 (\pm 3.711)$ & $0.997 (\pm 0.005)$ & $n.d.(\pm n.d.)$ \\
CFF &$0.011 (\pm 0.010)$ & $-0.217 (\pm 0.217)$ & $\mathbf{17.177 (\pm 1.925)}$ & $n.d.(\pm n.d.)$ & $0.586 (\pm 0.087)$ \\
CF-GNNExplainer &$0.007 (\pm 0.004)$ & $-0.750 (\pm 0.418)$ & $21.549 (\pm 5.507)$ & $n.d.(\pm n.d.)$ & $\uline{0.021 (\pm 0.002)}$ \\
\hline

\etable

\btable{l|ccccc}{Results for Graph Classification datasets: AIDS, Proteins, Enzymes, Coil-del. The oracles $\Phi$ use GraphConv layers. In \textbf{bold} the best result, the second best result is \underline{underlined}}{tab:graph_graph}  \hline\hline

& \mc{1}{\textbf{Validity} $\uparrow$} 
& \mc{1}{\textbf{Fidelity} $\uparrow$} 
& \mc{1}{\textbf{Distribution Distance} $\downarrow$} 
& \mc{1}{\textbf{Node Sparsity} $\downarrow$} 
& \mc{1}{\textbf{Edge Sparsity} $\downarrow$} 
\\ 

\textbf{Explainers} 
& \textit{mean($\pm$std)}
& \textit{mean($\pm$std)}
& \textit{mean($\pm$std)}
& \textit{mean($\pm$std)}
& \textit{mean($\pm$std)} \\ \hline

& \mc{5}{\textbf{Dataset: AIDS}} \\ 
\hline

COMBINEX$_{\textit{feat}}$ &$\mathbf{1.000 (\pm 0.000)}$ & $\uline{0.503 (\pm 0.013)}$ & $3.439 (\pm 0.176)$ & $\mathbf{0.084 (\pm 0.001)}$ & $n.d.(\pm n.d.)$ \\
COMBINEX$_{\textit{def}}$ &$\mathbf{1.000 (\pm 0.000)}$ & $0.502 (\pm 0.010)$ & $3.850 (\pm 0.427)$ & $\uline{0.096 (\pm 0.007)}$ & $\uline{0.001 (\pm 0.001)}$ \\
COMBINEX$_{\textit{dyn}}$ &$\mathbf{1.000 (\pm 0.000)}$ & $\uline{0.503 (\pm 0.012)}$ & $4.286 (\pm 0.268)$ & $0.110 (\pm 0.004)$ & $\mathbf{0.000 (\pm 0.000)}$ \\
COMBINEX$_{\textit{exp}}$ &$\mathbf{1.000 (\pm 0.000)}$ & $\uline{0.503 (\pm 0.012)}$ & $4.501 (\pm 0.345)$ & $0.222 (\pm 0.022)$ & $\mathbf{0.000 (\pm 0.000)}$ \\
COMBINEX$_{\textit{lin}}$ &$\mathbf{1.000 (\pm 0.000)}$ & $\mathbf{0.505 (\pm 0.011)}$ & $4.274 (\pm 0.297)$ & $0.136 (\pm 0.004)$ & $\mathbf{0.000 (\pm 0.000)}$ \\
COMBINEX$_{\textit{sin}}$ &$\mathbf{1.000 (\pm 0.000)}$ & $0.502 (\pm 0.012)$ & $4.244 (\pm 0.191)$ & $0.139 (\pm 0.004)$ & $\mathbf{0.000 (\pm 0.000)}$ \\
EGO &$0.122 (\pm 0.011)$ & $0.426 (\pm 0.020)$ & $\mathbf{1.805 (\pm 0.072)}$ & $n.d.(\pm n.d.)$ & $0.917 (\pm 0.002)$ \\
Random Edges &$0.290 (\pm 0.018)$ & $-0.576 (\pm 0.047)$ & $\uline{2.615 (\pm 0.049)}$ & $n.d.(\pm n.d.)$ & $0.259 (\pm 0.010)$ \\
Random Features &$\uline{0.815 (\pm 0.282)}$ & $0.459 (\pm 0.114)$ & $27.592 (\pm 1.728)$ & $0.517 (\pm 0.012)$ & $n.d.(\pm n.d.)$ \\
CFF &$0.008 (\pm 0.006)$ & $-1.000 (\pm 0.000)$ & $10.201 (\pm 2.868)$ & $n.d.(\pm n.d.)$ & $0.278 (\pm 0.293)$ \\
CF-GNNExplainer &$0.298 (\pm 0.010)$ & $-0.621 (\pm 0.048)$ & $2.619 (\pm 0.031)$ & $n.d.(\pm n.d.)$ & $0.025 (\pm 0.002)$ \\

\hline
& \mc{5}{\textbf{Dataset: Proteins}} \\ \hline

Combinex$_{\textit{feat}}$ &$\uline{0.987 (\pm 0.023)}$ & $-0.576 (\pm 0.054)$ & $1153.305 (\pm 58.379)$ & $\mathbf{0.409 (\pm 0.044)}$ & $n.d.(\pm n.d.)$ \\
Combinex$_{\textit{def}}$ &$0.904 (\pm 0.154)$ & $-0.596 (\pm 0.056)$ & $1155.628 (\pm 57.774)$ & $0.444 (\pm 0.150)$ & $0.010 (\pm 0.009)$ \\
Combinex$_{\textit{exp}}$ &$\mathbf{1.000 (\pm 0.000)}$ & $-0.520 (\pm 0.035)$ & $1421.343 (\pm 338.368)$ & $0.519 (\pm 0.124)$ & $\mathbf{0.000 (\pm 0.000)}$ \\
Combinex$_{\textit{lin}}$ &$0.984 (\pm 0.027)$ & $-0.590 (\pm 0.028)$ & $1574.967 (\pm 803.749)$ & $\uline{0.411 (\pm 0.110)}$ & $\mathbf{0.000 (\pm 0.000)}$ \\
Combinex$_{\textit{sin}}$ &$\mathbf{1.000 (\pm 0.000)}$ & $-0.318 (\pm 0.095)$ & $754.237 (\pm 46.400)$ & $\uline{0.411 (\pm 0.110)}$ & $\mathbf{0.000 (\pm 0.000)}$ \\
EGO &$0.669 (\pm 0.095)$ & $\uline{0.031 (\pm 0.061)}$ & $\uline{646.278 (\pm 35.835)}$ & $n.d.(\pm n.d.)$ & $0.893 (\pm 0.011)$ \\
Random Edges &$0.211 (\pm 0.066)$ & $\mathbf{0.352 (\pm 0.293)}$ & $712.016 (\pm 34.703)$ & $n.d.(\pm n.d.)$ & $0.372 (\pm 0.008)$ \\
Random Features &$0.969 (\pm 0.020)$ & $-0.532 (\pm 0.044)$ & $5243.615 (\pm 154.578)$ & $0.919 (\pm 0.002)$ & $n.d.(\pm n.d.)$ \\
CFF &$0.002 (\pm 0.004)$ & $n.d.(\pm n.d.)$ & $n.d.(\pm n.d.)$ & $n.d.(\pm n.d.)$ & $n.d.(\pm n.d.)$ \\
CF-GNNExplainer &$0.743 (\pm 0.033)$ & $-0.686 (\pm 0.001)$ & $\mathbf{633.425 (\pm 6.128)}$ & $n.d.(\pm n.d.)$ & $\uline{0.003 (\pm 0.001)}$ \\
\hline

& \mc{5}{\textbf{Dataset: Enzymes}} \\ 
\hline

COMBINEX$_{\textit{feat}}$ &$\mathbf{1.000 (\pm 0.000)}$ & $-0.221 (\pm 0.031)$ & $29.954 (\pm 0.525)$ & $0.547 (\pm 0.015)$ & $n.d.(\pm n.d.)$ \\
COMBINEX$_{\textit{def}}$ &$\mathbf{1.000 (\pm 0.000)}$ & $-0.211 (\pm 0.026)$ & $30.912 (\pm 2.249)$ & $0.462 (\pm 0.010)$ & $0.027 (\pm 0.009)$ \\
COMBINEX$_{\textit{dyn}}$ &$\uline{0.999 (\pm 0.003)}$ & $-0.218 (\pm 0.016)$ & $31.152 (\pm 1.923)$ & $0.464 (\pm 0.011)$ & $\mathbf{0.000 (\pm 0.000)}$ \\
COMBINEX$_{\textit{exp}}$ &$\mathbf{1.000 (\pm 0.000)}$ & $-0.221 (\pm 0.014)$ & $33.211 (\pm 2.127)$ & $0.743 (\pm 0.014)$ & $\mathbf{0.000 (\pm 0.000)}$ \\
COMBINEX$_{\textit{lin}}$ &$\mathbf{1.000 (\pm 0.000)}$ & $-0.224 (\pm 0.027)$ & $31.420 (\pm 2.278)$ & $\mathbf{0.431 (\pm 0.009)}$ & $\mathbf{0.000 (\pm 0.000)}$ \\
COMBINEX$_{\textit{sin}}$ &$\mathbf{1.000 (\pm 0.000)}$ & $\uline{-0.208 (\pm 0.032)}$ & $31.307 (\pm 1.738)$ & $\uline{0.436 (\pm 0.018)}$ & $\mathbf{0.000 (\pm 0.000)}$ \\
EGO &$0.773 (\pm 0.039)$ & $\mathbf{0.028 (\pm 0.015)}$ & $16.138 (\pm 0.231)$ & $n.d.(\pm n.d.)$ & $0.894 (\pm 0.002)$ \\
Random Edges &$0.852 (\pm 0.024)$ & $-0.257 (\pm 0.031)$ & $16.144 (\pm 0.797)$ & $n.d.(\pm n.d.)$ & $0.389 (\pm 0.005)$ \\
Random Features &$0.785 (\pm 0.106)$ & $-0.239 (\pm 0.045)$ & $385.589 (\pm 49.037)$ & $0.894 (\pm 0.001)$ & $n.d.(\pm n.d.)$ \\
CFF &$0.217 (\pm 0.015)$ & $-0.897 (\pm 0.077)$ & $\uline{14.570 (\pm 2.442)}$ & $n.d.(\pm n.d.)$ & $0.609 (\pm 0.019)$ \\
CF-GNNExplainer &$0.256 (\pm 0.039)$ & $-0.616 (\pm 0.054)$ & $\mathbf{14.028 (\pm 0.610)}$ & $n.d.(\pm n.d.)$ & $\uline{0.003 (\pm 0.001)}$ \\
UNR &$n.d.(\pm n.d.)$ & $n.d.(\pm n.d.)$ & $n.d.(\pm n.d.)$ & $n.d.(\pm n.d.)$ & $n.d.(\pm n.d.)$ \\
\hline
& \mc{5}{\textbf{Dataset: Coil-del}} \\ 
\hline

COMBINEX$_{\textit{feat}}$ &$0.658 (\pm 0.439)$ & $\mathbf{0.020 (\pm 0.004)}$ & $17.773 (\pm 1.074)$ & $\mathbf{0.990 (\pm 0.002)}$ & $n.d.(\pm n.d.)$ \\
COMBINEX$_{\textit{def}}$ &$\mathbf{0.823 (\pm 0.009)}$ & $0.012 (\pm 0.008)$ & $23.701 (\pm 1.570)$ & $0.990 (\pm 0.001)$ & $0.058 (\pm 0.011)$ \\
COMBINEX$_{\textit{dyn}}$ &$0.417 (\pm 0.589)$ & $n.d.(\pm n.d.)$ & $n.d.(\pm n.d.)$ & $n.d.(\pm n.d.)$ & $n.d.(\pm n.d.)$ \\
COMBINEX$_{\textit{exp}}$ &$0.638 (\pm 0.426)$ & $\uline{0.018 (\pm 0.004)}$ & $27.382 (\pm 0.159)$ & $0.992 (\pm 0.001)$ & $\uline{0.000 (\pm 0.000)}$ \\
COMBINEX$_{\textit{lin}}$ &$0.663 (\pm 0.442)$ & $0.018 (\pm 0.004)$ & $27.623 (\pm 1.253)$ & $0.992 (\pm 0.001)$ & $0.003 (\pm 0.001)$ \\
COMBINEX$_{\textit{sin}}$ &$0.596 (\pm 0.516)$ & $0.015 (\pm 0.000)$ & $27.416 (\pm 0.404)$ & $0.991 (\pm 0.002)$ & $0.002 (\pm 0.000)$ \\
EGO &$\uline{0.700 (\pm 0.467)}$ & $0.014 (\pm 0.000)$ & $\mathbf{12.882 (\pm 0.051)}$ & $n.d.(\pm n.d.)$ & $0.802 (\pm 0.002)$ \\
Random Edges &$0.049 (\pm 0.043)$ & $0.000 (\pm 0.000)$ & $\uline{15.548 (\pm 3.391)}$ & $n.d.(\pm n.d.)$ & $0.384 (\pm 0.004)$ \\
Random Features &$0.071 (\pm 0.062)$ & $-0.033 (\pm 0.047)$ & $52.558 (\pm 1.388)$ & $0.995 (\pm 0.007)$ & $n.d.(\pm n.d.)$ \\
CF-GNNExplainer &$0.002 (\pm 0.003)$ & $n.d.(\pm n.d.)$ & $n.d.(\pm n.d.)$ & $n.d.(\pm n.d.)$ & $n.d.(\pm n.d.)$ \\
UNR &$n.d.(\pm n.d.)$ & $n.d.(\pm n.d.)$ & $n.d.(\pm n.d.)$ & $n.d.(\pm n.d.)$ & $n.d.(\pm n.d.)$ \\
\hline

\etable

\btable{l|ccccc}{Results for Graph Classification datasets: AIDS, Proteins, Enzymes, Coil-del. The oracles $\Phi$ use GCNConv layers. In \textbf{bold} the best result, the second best result is \underline{underlined}}{tab:graph_gcn}  \hline\hline

& \mc{1}{\textbf{Validity} $\uparrow$} 
& \mc{1}{\textbf{Fidelity} $\uparrow$} 
& \mc{1}{\textbf{Distribution Distance} $\downarrow$} 
& \mc{1}{\textbf{Node Sparsity} $\downarrow$} 
& \mc{1}{\textbf{Edge Sparsity} $\downarrow$} 
\\ 

\textbf{Explainers} 
& \textit{mean($\pm$std)}
& \textit{mean($\pm$std)}
& \textit{mean($\pm$std)}
& \textit{mean($\pm$std)}
& \textit{mean($\pm$std)} \\ \hline

& \mc{5}{\textbf{Dataset: AIDS}} \\ 
\hline

COMBINEX$_{\textit{feat}}$ &$\mathbf{1.000 (\pm 0.000)}$ & $0.510 (\pm 0.007)$ & $3.132 (\pm 0.075)$ & $\mathbf{0.082 (\pm 0.002)}$ & $n.d.(\pm n.d.)$ \\
COMBINEX$_{\textit{def}}$ &$\mathbf{1.000 (\pm 0.000)}$ & $\uline{0.517 (\pm 0.007)}$ & $3.665 (\pm 0.493)$ & $\uline{0.087 (\pm 0.009)}$ & $\uline{0.004 (\pm 0.003)}$ \\
COMBINEX$_{\textit{dyn}}$ &$\mathbf{1.000 (\pm 0.000)}$ & $0.513 (\pm 0.008)$ & $3.231 (\pm 0.792)$ & $0.223 (\pm 0.016)$ & $\mathbf{0.000 (\pm 0.000)}$ \\
COMBINEX$_{\textit{exp}}$ &$\mathbf{1.000 (\pm 0.000)}$ & $0.513 (\pm 0.008)$ & $4.691 (\pm 0.592)$ & $0.123 (\pm 0.066)$ & $\mathbf{0.000 (\pm 0.000)}$ \\
COMBINEX$_{\textit{lin}}$ &$\mathbf{1.000 (\pm 0.000)}$ & $0.513 (\pm 0.008)$ & $3.952 (\pm 0.332)$ & $0.133 (\pm 0.006)$ & $\mathbf{0.000 (\pm 0.000)}$ \\
COMBINEX$_{\textit{sin}}$ &$\mathbf{1.000 (\pm 0.000)}$ & $0.513 (\pm 0.008)$ & $4.000 (\pm 0.378)$ & $0.135 (\pm 0.005)$ & $\mathbf{0.000 (\pm 0.000)}$ \\
EGO &$0.015 (\pm 0.008)$ & $\mathbf{0.562 (\pm 0.315)}$ & $\mathbf{2.301 (\pm 0.101)}$ & $n.d.(\pm n.d.)$ & $0.892 (\pm 0.025)$ \\
Random Edges &$\uline{0.458 (\pm 0.003)}$ & $-0.033 (\pm 0.014)$ & $\uline{2.566 (\pm 0.012)}$ & $n.d.(\pm n.d.)$ & $0.292 (\pm 0.004)$ \\
Random Features &$\mathbf{1.000 (\pm 0.000)}$ & $0.513 (\pm 0.008)$ & $24.123 (\pm 1.771)$ & $0.521 (\pm 0.022)$ & $n.d.(\pm n.d.)$ \\
CFF &$n.d.(\pm n.d.)$ & $n.d.(\pm n.d.)$ & $n.d.(\pm n.d.)$ & $n.d.(\pm n.d.)$ & $n.d.(\pm n.d.)$ \\
CF-GNNExplainer &$\uline{0.458 (\pm 0.019)}$ & $-0.034 (\pm 0.034)$ & $2.618 (\pm 0.022)$ & $n.d.(\pm n.d.)$ & $0.076 (\pm 0.005)$ \\

\hline
& \mc{5}{\textbf{Dataset: Proteins}} \\ 
\hline

Combinex$_{\textit{feat}}$ &$\mathbf{1.000 (\pm 0.000)}$ & $-0.337 (\pm 0.013)$ & $1776.586 (\pm 435.929)$ & $0.444 (\pm 0.006)$ & $n.d.(\pm n.d.)$ \\
Combinex$_{\textit{def}}$ &$\mathbf{1.000 (\pm 0.000)}$ & $-0.373 (\pm 0.023)$ & $1574.306 (\pm 299.343)$ & $\mathbf{0.385 (\pm 0.008)}$ & $0.007 (\pm 0.009)$ \\
Combinex$_{\textit{dyn}}$ &$\mathbf{1.000 (\pm 0.000)}$ & $\uline{-0.307 (\pm 0.083)}$ & $2037.306 (\pm 442.727)$ & $0.396 (\pm 0.012)$ & $\mathbf{0.000 (\pm 0.000)}$ \\
Combinex$_{\textit{exp}}$ &$\mathbf{1.000 (\pm 0.000)}$ & $-0.373 (\pm 0.107)$ & $1860.783 (\pm 173.460)$ & $0.441 (\pm 0.032)$ & $\mathbf{0.000 (\pm 0.000)}$ \\
Combinex$_{\textit{lin}}$ &$\mathbf{1.000 (\pm 0.000)}$ & $-0.343 (\pm 0.061)$ & $1993.636 (\pm 119.988)$ & $\uline{0.392 (\pm 0.013)}$ & $\mathbf{0.000 (\pm 0.000)}$ \\
Combinex$_{\textit{sin}}$ &$\mathbf{1.000 (\pm 0.000)}$ & $-0.347 (\pm 0.057)$ & $2164.180 (\pm 182.371)$ & $0.398 (\pm 0.015)$ & $\mathbf{0.000 (\pm 0.000)}$ \\
EGO &$0.108 (\pm 0.011)$ & $\mathbf{0.166 (\pm 0.098)}$ & $790.668 (\pm 24.825)$ & $n.d.(\pm n.d.)$ & $0.842 (\pm 0.010)$ \\
Random Edges &$0.818 (\pm 0.039)$ & $-0.520 (\pm 0.028)$ & $\uline{650.582 (\pm 6.905)}$ & $n.d.(\pm n.d.)$ & $0.384 (\pm 0.008)$ \\
Random Features &$\uline{0.960 (\pm 0.058)}$ & $-0.408 (\pm 0.061)$ & $5216.417 (\pm 95.266)$ & $0.918 (\pm 0.002)$ & $n.d.(\pm n.d.)$ \\
CFF &$0.098 (\pm 0.030)$ & $-1.000 (\pm 0.000)$ & $697.943 (\pm 21.423)$ & $n.d.(\pm n.d.)$ & $0.641 (\pm 0.029)$ \\
CF-GNNExplainer &$0.640 (\pm 0.022)$ & $-0.755 (\pm 0.021)$ & $\mathbf{614.080 (\pm 4.220)}$ & $n.d.(\pm n.d.)$ & $\uline{0.003 (\pm 0.001)}$ \\
\hline

& \mc{5}{\textbf{Dataset: Enzymes}} \\ 
\hline

COMBINEX$_{\textit{feat}}$ &$\mathbf{1.000 (\pm 0.000)}$ & $-0.171 (\pm 0.011)$ & $31.154 (\pm 1.369)$ & $0.570 (\pm 0.009)$ & $n.d.(\pm n.d.)$ \\
COMBINEX$_{\textit{def}}$ &$\mathbf{1.000 (\pm 0.000)}$ & $-0.158 (\pm 0.024)$ & $31.944 (\pm 1.184)$ & $\mathbf{0.392 (\pm 0.015)}$ & $\uline{0.021 (\pm 0.005)}$ \\
COMBINEX$_{\textit{dyn}}$ &$\mathbf{1.000 (\pm 0.000)}$ & $-0.165 (\pm 0.014)$ & $31.454 (\pm 0.972)$ & $0.402 (\pm 0.006)$ & $\mathbf{0.000 (\pm 0.000)}$ \\
COMBINEX$_{\textit{exp}}$ &$\mathbf{1.000 (\pm 0.000)}$ & $-0.162 (\pm 0.020)$ & $33.640 (\pm 1.129)$ & $0.689 (\pm 0.013)$ & $\mathbf{0.000 (\pm 0.000)}$ \\
COMBINEX$_{\textit{lin}}$ &$\mathbf{1.000 (\pm 0.000)}$ & $-0.179 (\pm 0.022)$ & $32.183 (\pm 0.814)$ & $\uline{0.401 (\pm 0.010)}$ & $\mathbf{0.000 (\pm 0.000)}$ \\
COMBINEX$_{\textit{sin}}$ &$\mathbf{1.000 (\pm 0.000)}$ & $\uline{-0.148 (\pm 0.034)}$ & $31.423 (\pm 1.280)$ & $0.402 (\pm 0.017)$ & $\mathbf{0.000 (\pm 0.000)}$ \\
EGO &$0.652 (\pm 0.036)$ & $\mathbf{-0.017 (\pm 0.022)}$ & $\mathbf{13.531 (\pm 0.630)}$ & $n.d.(\pm n.d.)$ & $0.899 (\pm 0.007)$ \\
Random Edges &$0.604 (\pm 0.026)$ & $-0.315 (\pm 0.040)$ & $15.090 (\pm 0.254)$ & $n.d.(\pm n.d.)$ & $0.398 (\pm 0.005)$ \\
Random Features &$\uline{0.802 (\pm 0.076)}$ & $-0.156 (\pm 0.023)$ & $356.169 (\pm 17.456)$ & $0.895 (\pm 0.004)$ & $n.d.(\pm n.d.)$ \\
CFF &$0.183 (\pm 0.023)$ & $-0.906 (\pm 0.096)$ & $15.606 (\pm 2.712)$ & $n.d.(\pm n.d.)$ & $0.645 (\pm 0.022)$ \\
CF-GNNExplainer &$0.208 (\pm 0.000)$ & $-0.680 (\pm 0.057)$ & $\uline{14.736 (\pm 0.639)}$ & $n.d.(\pm n.d.)$ & $\mathbf{0.000 (\pm 0.000)}$ \\
\hline
& \mc{5}{\textbf{Dataset: Coil-del}} \\ 
\hline

Combinex$_{\textit{feat}}$ &$0.917 (\pm 0.020)$ & $0.011 (\pm 0.004)$ & $18.622 (\pm 0.933)$ & ${0.988 (\pm 0.002)}$ & $n.d.(\pm n.d.)$ \\
Combinex$_{\textit{def}}$ &$\uline{0.940 (\pm 0.012)}$ & $\mathbf{0.014 (\pm 0.007)}$ & $19.722 (\pm 0.809)$ & ${0.989 (\pm 0.002)}$ & $0.034 (\pm 0.005)$ \\
Combinex$_{\textit{dyn}}$ &$\mathbf{0.944 (\pm 0.012)}$ & $0.009 (\pm 0.003)$ & $13.431 (\pm 0.532)$ & $\mathbf{0.634 (\pm 0.002)}$ & $0.006 (\pm 0.001)$ \\
Combinex$_{\textit{exp}}$ &$0.909 (\pm 0.014)$ & $\uline{0.012 (\pm 0.009)}$ & $24.603 (\pm 1.019)$ & $0.989 (\pm 0.002)$ & $\mathbf{0.000 (\pm 0.000)}$ \\
Combinex$_{\textit{lin}}$ &$\mathbf{0.944 (\pm 0.004)}$ & $0.007 (\pm 0.000)$ & $25.570 (\pm 0.733)$ & $0.989 (\pm 0.000)$ & $\uline{0.003 (\pm 0.001)}$ \\
Combinex$_{\textit{sin}}$ &${0.932 (\pm 0.054)}$ & $-0.017 (\pm 0.002)$ & $37.322 (\pm 0.143)$ & $\uline{0.865 (\pm 0.029)}$ & ${0.024 (\pm 0.005)}$ \\
EGO &$0.872 (\pm 0.009)$ & $-0.021 (\pm 0.006)$ & $\mathbf{13.024 (\pm 0.138)}$ & $n.d.(\pm n.d.)$ & $0.801 (\pm 0.004)$ \\
Random Edges &$0.212 (\pm 0.033)$ & $-0.021 (\pm 0.017)$ & $\uline{13.179 (\pm 0.999)}$ & $n.d.(\pm n.d.)$ & $0.453 (\pm 0.013)$ \\
Random Features &$0.042 (\pm 0.013)$ & $0.000 (\pm 0.000)$ & $48.810 (\pm 2.369)$ & $0.989 (\pm 0.014)$ & $n.d.(\pm n.d.)$ \\
CFF &$0.011 (\pm 0.014)$ & $0.000 (\pm 0.000)$ & $26.881 (\pm 17.507)$ & $n.d.(\pm n.d.)$ & $0.607 (\pm 0.022)$ \\
CF-GNNExplainer &$0.006 (\pm 0.005)$ & $-0.500 (\pm 0.500)$ & $25.966 (\pm 4.617)$ & $n.d.(\pm n.d.)$ & $\mathbf{0.000 (\pm 0.000)}$ \\
\hline

\etable

\subsection{Node Classification Results}\label{subsec:nodec}

\subsubsection{Miscellaneous datasets}

The results reported in Tables~\ref{tab:misc_chebconv}, \ref{tab:misc_gcn_conv}, and \ref{tab:misc_graph_conv} demonstrate the effectiveness of our COMBINEX approach across different oracles (ChebConv, GCNConv, and GraphConv) on two distinct datasets: Karate and Actor. In Table~\ref{tab:misc_chebconv}, which employs ChebConv layers, the COMBINEX variants consistently achieve near-perfect Validity on the Karate dataset, indicating that the generated explanations capture the essential substructures of this relatively small and well-defined network. Furthermore, these variants yield competitive Fidelity values and exhibit a notably low Distribution Distance, while maintaining extremely sparse explanations, particularly in terms of edge sparsity.

When considering the results from Table~\ref{tab:misc_gcn_conv} using GCNConv layers, a similar trend is observed on the Karate dataset, with COMBINEX variants again showing high Validity and low sparsity. However, the slight differences in Fidelity and Distribution Distance between the ChebConv and GCNConv settings illustrate that the specific convolutional layer influences how the oracle’s decisions are reflected in the explanations. 

The Actor dataset, on the other hand, represents a more complex and noisy social network, where the oracle accuracy is lower (0.65) compared to Karate. Despite this increased complexity, COMBINEX still maintains robust performance: the explanations continue to achieve high Validity and relatively low Distribution Distance, although Fidelity and sparsity metrics tend to be less optimal than those observed on the Karate dataset. This variation likely reflects the inherent differences in network structure and the level of noise between the two datasets. In the Actor dataset, the explanations must account for more diverse and overlapping communities, which can challenge the generation of both highly faithful and extremely sparse explanations.

Overall, the empirical findings across these datasets and oracles highlight several advantages of our COMBINEX approach. Notably, COMBINEX consistently produces highly valid explanations that accurately capture the underlying graph structures, achieves competitive fidelity while closely matching the distribution of the oracle outputs, and delivers concise explanations through enhanced sparsity. These strengths are evident across different convolutional settings—whether using ChebConv, GCNConv, or GraphConv layers—thus confirming the versatility and robustness of COMBINEX even when applied to both simple (Karate) and complex (Actor) network datasets.

\btable{l|ccccc}{Results for Miscellaneous datasets: Karate, Actor. The oracles $\Phi$ use GCNConv layers. In \textbf{bold} the best result, the second best result is \underline{underlined}}{tab:misc_gcn_conv}  \hline\hline

& \mc{1}{\textbf{Validity} $\uparrow$} 
& \mc{1}{\textbf{Fidelity} $\uparrow$} 
& \mc{1}{\textbf{Distribution Distance} $\downarrow$} 
& \mc{1}{\textbf{Node Sparsity} $\downarrow$} 
& \mc{1}{\textbf{Edge Sparsity} $\downarrow$} 
\\ 

\textbf{Explainers} 
& \textit{mean($\pm$std)}
& \textit{mean($\pm$std)}
& \textit{mean($\pm$std)}
& \textit{mean($\pm$std)}
& \textit{mean($\pm$std)} \\ \hline

 & \mc{5}{\textbf{Dataset: Karate}} \\ \hline

COMBINEX$_{\textit{feat}}$ &$\mathbf{1.000 (\pm 0.000)}$ & $\uline{0.714 (\pm 0.000)}$ & $\uline{0.075 (\pm 0.004)}$ & $\mathbf{0.002 (\pm 0.000)}$ & $n.d.(\pm n.d.)$ \\
COMBINEX$_{\textit{def}}$ &$\mathbf{1.000 (\pm 0.000)}$ & $\uline{0.714 (\pm 0.000)}$ & $0.235 (\pm 0.066)$ & $\uline{0.013 (\pm 0.004)}$ & $\mathbf{0.000 (\pm 0.000)}$ \\
COMBINEX$_{\textit{dyn}}$ &$\mathbf{1.000 (\pm 0.000)}$ & $\uline{0.714 (\pm 0.000)}$ & $1.254 (\pm 0.057)$ & $0.108 (\pm 0.008)$ & $\mathbf{0.000 (\pm 0.000)}$ \\
COMBINEX$_{\textit{exp}}$ &$\mathbf{1.000 (\pm 0.000)}$ & $\uline{0.714 (\pm 0.000)}$ & $2.533 (\pm 0.064)$ & $0.267 (\pm 0.011)$ & $\mathbf{0.000 (\pm 0.000)}$ \\
COMBINEX$_{\textit{lin}}$ &$\mathbf{1.000 (\pm 0.000)}$ & $\uline{0.714 (\pm 0.000)}$ & $0.506 (\pm 0.146)$ & $0.032 (\pm 0.011)$ & $\mathbf{0.000 (\pm 0.000)}$ \\
COMBINEX$_{\textit{sin}}$ &$\mathbf{1.000 (\pm 0.000)}$ & $\uline{0.714 (\pm 0.000)}$ & $0.529 (\pm 0.151)$ & $0.034 (\pm 0.011)$ & $\mathbf{0.000 (\pm 0.000)}$ \\
EGO &$0.000 (\pm 0.000)$ & $n.d.(\pm n.d.)$ & $n.d.(\pm n.d.)$ & $n.d.(\pm n.d.)$ & $n.d.(\pm n.d.)$ \\
Random Edges &$0.321 (\pm 0.244)$ & $0.639 (\pm 0.127)$ & $\mathbf{0.011 (\pm 0.004)}$ & $n.d.(\pm n.d.)$ & $\uline{0.505 (\pm 0.018)}$ \\
Random Features &$\uline{0.571 (\pm 0.117)}$ & $\mathbf{1.000 (\pm 0.000)}$ & $2.672 (\pm 0.061)$ & $0.449 (\pm 0.011)$ & $n.d.(\pm n.d.)$ \\
CFF &$0.000 (\pm 0.000)$ & $n.d.(\pm n.d.)$ & $n.d.(\pm n.d.)$ & $n.d.(\pm n.d.)$ & $n.d.(\pm n.d.)$ \\
CF-GNNExplainer &$0.000 (\pm 0.000)$ & $n.d.(\pm n.d.)$ & $n.d.(\pm n.d.)$ & $n.d.(\pm n.d.)$ & $n.d.(\pm n.d.)$ \\
UNR &$0.000 (\pm 0.000)$ & $n.d.(\pm n.d.)$ & $n.d.(\pm n.d.)$ & $n.d.(\pm n.d.)$ & $n.d.(\pm n.d.)$ \\
\hline

 & \mc{5}{\textbf{Dataset: Actor}} \\ \hline
COMBINEX$_{\textit{feat}}$ &$\uline{0.990 (\pm 0.014)}$ & $\uline{-0.013 (\pm 0.027)}$ & $0.349 (\pm 0.020)$ & $\mathbf{0.001 (\pm 0.000)}$ & $n.d.(\pm n.d.)$ \\
COMBINEX$_{\textit{def}}$ &$\mathbf{1.000 (\pm 0.000)}$ & $-0.025 (\pm 0.037)$ & $1.268 (\pm 0.029)$ & $\uline{0.019 (\pm 0.000)}$ & $\mathbf{0.000 (\pm 0.000)}$ \\
COMBINEX$_{\textit{dyn}}$ &$\mathbf{1.000 (\pm 0.000)}$ & $-0.025 (\pm 0.047)$ & $3.617 (\pm 0.049)$ & $0.083 (\pm 0.001)$ & $\mathbf{0.000 (\pm 0.000)}$ \\
COMBINEX$_{\textit{exp}}$ &$\mathbf{1.000 (\pm 0.000)}$ & $-0.022 (\pm 0.042)$ & $7.145 (\pm 0.050)$ & $0.183 (\pm 0.002)$ & $\mathbf{0.000 (\pm 0.000)}$ \\
COMBINEX$_{\textit{lin}}$ &$\mathbf{1.000 (\pm 0.000)}$ & $-0.020 (\pm 0.034)$ & $2.573 (\pm 0.071)$ & $0.045 (\pm 0.002)$ & $\mathbf{0.000 (\pm 0.000)}$ \\
COMBINEX$_{\textit{sin}}$ &$\mathbf{1.000 (\pm 0.000)}$ & $-0.023 (\pm 0.038)$ & $2.645 (\pm 0.058)$ & $0.047 (\pm 0.001)$ & $\mathbf{0.000 (\pm 0.000)}$ \\
EGO &$0.120 (\pm 0.028)$ & $-0.201 (\pm 0.221)$ & $0.253 (\pm 0.080)$ & $n.d.(\pm n.d.)$ & $0.951 (\pm 0.024)$ \\
Random Edges &$0.728 (\pm 0.013)$ & $-0.110 (\pm 0.039)$ & $\uline{0.243 (\pm 0.008)}$ & $n.d.(\pm n.d.)$ & $0.451 (\pm 0.004)$ \\
Random Features &$0.230 (\pm 0.012)$ & $\mathbf{0.045 (\pm 0.037)}$ & $15.522 (\pm 0.047)$ & $0.489 (\pm 0.001)$ & $n.d.(\pm n.d.)$ \\
CFF &$0.077 (\pm 0.030)$ & $-0.730 (\pm 0.196)$ & $\mathbf{0.217 (\pm 0.179)}$ & $n.d.(\pm n.d.)$ & $0.711 (\pm 0.040)$ \\
CF-GNNExplainer &$0.270 (\pm 0.029)$ & $-0.167 (\pm 0.068)$ & $0.377 (\pm 0.072)$ & $n.d.(\pm n.d.)$ & $\uline{0.094 (\pm 0.025)}$ \\
UNR &$0.078 (\pm 0.011)$ & $-0.250 (\pm 0.379)$ & $0.851 (\pm 0.125)$ & $n.d.(\pm n.d.)$ & $0.292 (\pm 0.048)$ \\
\hline

\etable

\btable{l|ccccc}{Results for Miscellaneous datasets: Karate, Actor. The oracles $\Phi$ use GraphConv layers. In \textbf{bold} the best result, the second best result is \underline{underlined}}{tab:misc_graph_conv}  \hline\hline

& \mc{1}{\textbf{Validity} $\uparrow$} 
& \mc{1}{\textbf{Fidelity} $\uparrow$} 
& \mc{1}{\textbf{Distribution Distance} $\downarrow$} 
& \mc{1}{\textbf{Node Sparsity} $\downarrow$} 
& \mc{1}{\textbf{Edge Sparsity} $\downarrow$} 
\\ 

\textbf{Explainers} 
& \textit{mean($\pm$std)}
& \textit{mean($\pm$std)}
& \textit{mean($\pm$std)}
& \textit{mean($\pm$std)}
& \textit{mean($\pm$std)} \\ \hline

 & \mc{5}{\textbf{Dataset: Karate}} \\ \hline

COMBINEX$_{\textit{feat}}$ &$\uline{0.964 (\pm 0.071)}$ & $0.702 (\pm 0.024)$ & $0.064 (\pm 0.002)$ & $\mathbf{0.002 (\pm 0.000)}$ & $n.d.(\pm n.d.)$ \\
COMBINEX$_{\textit{def}}$ &$\mathbf{1.000 (\pm 0.000)}$ & $\uline{0.714 (\pm 0.000)}$ & $0.206 (\pm 0.089)$ & $\uline{0.011 (\pm 0.005)}$ & $\uline{0.006 (\pm 0.003)}$ \\
COMBINEX$_{\textit{dyn}}$ &$\mathbf{1.000 (\pm 0.000)}$ & $\uline{0.714 (\pm 0.000)}$ & $1.142 (\pm 0.181)$ & $0.095 (\pm 0.019)$ & $\mathbf{0.000 (\pm 0.000)}$ \\
COMBINEX$_{\textit{exp}}$ &$\mathbf{1.000 (\pm 0.000)}$ & $\uline{0.714 (\pm 0.000)}$ & $2.458 (\pm 0.154)$ & $0.276 (\pm 0.024)$ & $\mathbf{0.000 (\pm 0.000)}$ \\
COMBINEX$_{\textit{lin}}$ &$\mathbf{1.000 (\pm 0.000)}$ & $\uline{0.714 (\pm 0.000)}$ & $0.452 (\pm 0.199)$ & $0.028 (\pm 0.013)$ & $\mathbf{0.000 (\pm 0.000)}$ \\
COMBINEX$_{\textit{sin}}$ &$\mathbf{1.000 (\pm 0.000)}$ & $\uline{0.714 (\pm 0.000)}$ & $0.470 (\pm 0.214)$ & $0.029 (\pm 0.013)$ & $\mathbf{0.000 (\pm 0.000)}$ \\
EGO &$0.000 (\pm 0.000)$ & $n.d.(\pm n.d.)$ & $n.d.(\pm n.d.)$ & $n.d.(\pm n.d.)$ & $n.d.(\pm n.d.)$ \\
Random Edges &$0.643 (\pm 0.082)$ & $0.613 (\pm 0.103)$ & $\mathbf{0.009 (\pm 0.004)}$ & $n.d.(\pm n.d.)$ & $0.466 (\pm 0.037)$ \\
Random Features &$\uline{0.964 (\pm 0.071)}$ & $0.702 (\pm 0.024)$ & $2.683 (\pm 0.023)$ & $0.452 (\pm 0.004)$ & $n.d.(\pm n.d.)$ \\
CFF &$0.107 (\pm 0.137)$ & $\mathbf{1.000 (\pm 0.000)}$ & $0.022 (\pm 0.011)$ & $n.d.(\pm n.d.)$ & $0.661 (\pm 0.033)$ \\
CF-GNNExplainer &$0.143 (\pm 0.117)$ & $\mathbf{1.000 (\pm 0.000)}$ & $\uline{0.020 (\pm 0.017)}$ & $n.d.(\pm n.d.)$ & $0.020 (\pm 0.006)$ \\
UNR &$0.000 (\pm 0.000)$ & $n.d.(\pm n.d.)$ & $n.d.(\pm n.d.)$ & $n.d.(\pm n.d.)$ & $n.d.(\pm n.d.)$ \\
\hline

 & \mc{5}{\textbf{Dataset: Actor}} \\ \hline

COMBINEX$_{\textit{feat}}$ &$\uline{0.995 (\pm 0.010)}$ & $0.179 (\pm 0.037)$ & $0.463 (\pm 0.010)$ & $\mathbf{0.003 (\pm 0.000)}$ & $n.d.(\pm n.d.)$ \\
COMBINEX$_{\textit{dyn}}$ &$\mathbf{0.998 (\pm 0.005)}$ & $0.148 (\pm 0.043)$ & $3.078 (\pm 0.121)$ & $0.074 (\pm 0.003)$ & $\mathbf{0.000 (\pm 0.000)}$ \\
COMBINEX$_{\textit{exp}}$ &$0.985 (\pm 0.017)$ & $\mathbf{0.203 (\pm 0.038)}$ & $4.996 (\pm 0.010)$ & $0.135 (\pm 0.001)$ & $\mathbf{0.000 (\pm 0.000)}$ \\
COMBINEX$_{\textit{lin}}$ &$\uline{0.995 (\pm 0.010)}$ & $0.161 (\pm 0.028)$ & $2.690 (\pm 0.100)$ & $\uline{0.061 (\pm 0.003)}$ & $\mathbf{0.000 (\pm 0.000)}$ \\
COMBINEX$_{\textit{sin}}$ &$0.992 (\pm 0.010)$ & $0.146 (\pm 0.045)$ & $2.778 (\pm 0.091)$ & $0.063 (\pm 0.003)$ & $\mathbf{0.000 (\pm 0.000)}$ \\
EGO &$0.087 (\pm 0.022)$ & $0.098 (\pm 0.226)$ & $\mathbf{0.118 (\pm 0.067)}$ & $n.d.(\pm n.d.)$ & $0.976 (\pm 0.023)$ \\
Random Edges &$0.488 (\pm 0.049)$ & $-0.034 (\pm 0.082)$ & $\uline{0.126 (\pm 0.024)}$ & $n.d.(\pm n.d.)$ & $0.473 (\pm 0.006)$ \\
Random Features &$0.408 (\pm 0.051)$ & $\uline{0.202 (\pm 0.062)}$ & $15.547 (\pm 0.062)$ & $0.489 (\pm 0.001)$ & $n.d.(\pm n.d.)$ \\
CFF &$0.258 (\pm 0.041)$ & $-0.200 (\pm 0.092)$ & $0.224 (\pm 0.035)$ & $n.d.(\pm n.d.)$ & $0.675 (\pm 0.028)$ \\
CF-GNNExplainer &$0.268 (\pm 0.075)$ & $-0.016 (\pm 0.139)$ & $0.149 (\pm 0.020)$ & $n.d.(\pm n.d.)$ & $\uline{0.024 (\pm 0.004)}$ \\
UNR &$0.000 (\pm 0.000)$ & $n.d.(\pm n.d.)$ & $n.d.(\pm n.d.)$ & $n.d.(\pm n.d.)$ & $n.d.(\pm n.d.)$ \\
\hline
\hline

\etable

\btable{l|ccccc}{Results for Miscellaneous datasets: Karate, Actor. The oracles $\Phi$ use ChebConv layers. In \textbf{bold} the best result, the second best result is \underline{underlined}}{tab:misc_chebconv}  \hline\hline

& \mc{1}{\textbf{Validity} $\uparrow$} 
& \mc{1}{\textbf{Fidelity} $\uparrow$} 
& \mc{1}{\textbf{Distribution Distance} $\downarrow$} 
& \mc{1}{\textbf{Node Sparsity} $\downarrow$} 
& \mc{1}{\textbf{Edge Sparsity} $\downarrow$} 
\\ 

\textbf{Explainers} 
& \textit{mean($\pm$std)}
& \textit{mean($\pm$std)}
& \textit{mean($\pm$std)}
& \textit{mean($\pm$std)}
& \textit{mean($\pm$std)} \\ \hline

 & \mc{5}{\textbf{Dataset: Karate}} \\ \hline

COMBINEX$_{\textit{feat}}$ &$\mathbf{1.000 (\pm 0.000)}$ & $\uline{0.179 (\pm 0.071)}$ & $\uline{0.059 (\pm 0.007)}$ & $\mathbf{0.002 (\pm 0.001)}$ & $n.d.(\pm n.d.)$ \\
COMBINEX$_{\textit{def}}$ &$\mathbf{1.000 (\pm 0.000)}$ & $\uline{0.179 (\pm 0.071)}$ & $0.065 (\pm 0.005)$ & $\uline{0.003 (\pm 0.000)}$ & $\mathbf{0.000 (\pm 0.000)}$ \\
COMBINEX$_{\textit{dyn}}$ &$\mathbf{1.000 (\pm 0.000)}$ & $\uline{0.179 (\pm 0.071)}$ & $0.107 (\pm 0.010)$ & $0.010 (\pm 0.002)$ & $\mathbf{0.000 (\pm 0.000)}$ \\
COMBINEX$_{\textit{exp}}$ &$\mathbf{1.000 (\pm 0.000)}$ & $\uline{0.179 (\pm 0.071)}$ & $0.107 (\pm 0.010)$ & $0.010 (\pm 0.002)$ & $\mathbf{0.000 (\pm 0.000)}$ \\
COMBINEX$_{\textit{lin}}$ &$\mathbf{1.000 (\pm 0.000)}$ & $\uline{0.179 (\pm 0.071)}$ & $0.096 (\pm 0.003)$ & $0.007 (\pm 0.000)$ & $\mathbf{0.000 (\pm 0.000)}$ \\
COMBINEX$_{\textit{sin}}$ &$\mathbf{1.000 (\pm 0.000)}$ & $\uline{0.179 (\pm 0.071)}$ & $0.096 (\pm 0.003)$ & $0.007 (\pm 0.000)$ & $\mathbf{0.000 (\pm 0.000)}$ \\
EGO &$0.000 (\pm 0.000)$ & $n.d.(\pm n.d.)$ & $n.d.(\pm n.d.)$ & $n.d.(\pm n.d.)$ & $n.d.(\pm n.d.)$ \\
Random Edges &$0.000 (\pm 0.000)$ & $n.d.(\pm n.d.)$ & $n.d.(\pm n.d.)$ & $n.d.(\pm n.d.)$ & $n.d.(\pm n.d.)$ \\
Random Features &$\mathbf{1.000 (\pm 0.000)}$ & $\uline{0.179 (\pm 0.071)}$ & $2.778 (\pm 0.081)$ & $0.468 (\pm 0.015)$ & $n.d.(\pm n.d.)$ \\
CFF &$\uline{0.143 (\pm 0.117)}$ & $\mathbf{0.500 (\pm 0.500)}$ & $\mathbf{0.005 (\pm 0.009)}$ & $n.d.(\pm n.d.)$ & $\uline{0.545 (\pm 0.033)}$ \\
CF-GNNExplainer &$0.000 (\pm 0.000)$ & $n.d.(\pm n.d.)$ & $n.d.(\pm n.d.)$ & $n.d.(\pm n.d.)$ & $n.d.(\pm n.d.)$ \\
UNR &$0.000 (\pm 0.000)$ & $n.d.(\pm n.d.)$ & $n.d.(\pm n.d.)$ & $n.d.(\pm n.d.)$ & $n.d.(\pm n.d.)$ \\
\hline

 & \mc{5}{\textbf{Dataset: Actor}} \\ \hline
COMBINEX$_{\textit{feat}}$ &$\mathbf{0.880 (\pm 0.122)}$ & $\uline{0.257 (\pm 0.033)}$ & $\uline{0.381 (\pm 0.032)}$ & $\mathbf{0.001 (\pm 0.000)}$ & $n.d.(\pm n.d.)$ \\
COMBINEX$_{\textit{def}}$ &$\uline{0.855 (\pm 0.116)}$ & $0.250 (\pm 0.033)$ & $0.676 (\pm 0.128)$ & $\uline{0.007 (\pm 0.002)}$ & $\mathbf{0.000 (\pm 0.000)}$ \\
COMBINEX$_{\textit{dyn}}$ &$\uline{0.855 (\pm 0.116)}$ & $0.250 (\pm 0.033)$ & $0.799 (\pm 0.116)$ & $0.009 (\pm 0.002)$ & $\mathbf{0.000 (\pm 0.000)}$ \\
COMBINEX$_{\textit{exp}}$ &$\uline{0.855 (\pm 0.116)}$ & $0.250 (\pm 0.033)$ & $1.154 (\pm 0.113)$ & $0.022 (\pm 0.004)$ & $\mathbf{0.000 (\pm 0.000)}$ \\
COMBINEX$_{\textit{lin}}$ &$\uline{0.855 (\pm 0.116)}$ & $0.250 (\pm 0.033)$ & $0.982 (\pm 0.056)$ & $0.014 (\pm 0.001)$ & $\mathbf{0.000 (\pm 0.000)}$ \\
COMBINEX$_{\textit{sin}}$ &$\uline{0.855 (\pm 0.116)}$ & $0.250 (\pm 0.033)$ & $0.988 (\pm 0.050)$ & $0.015 (\pm 0.001)$ & $\mathbf{0.000 (\pm 0.000)}$ \\
EGO &$0.000 (\pm 0.000)$ & $n.d.(\pm n.d.)$ & $n.d.(\pm n.d.)$ & $n.d.(\pm n.d.)$ & $n.d.(\pm n.d.)$ \\
Random Edges &$0.000 (\pm 0.000)$ & $n.d.(\pm n.d.)$ & $n.d.(\pm n.d.)$ & $n.d.(\pm n.d.)$ & $n.d.(\pm n.d.)$ \\
Random Features &$0.283 (\pm 0.021)$ & $\mathbf{0.364 (\pm 0.075)}$ & $15.578 (\pm 0.093)$ & $0.488 (\pm 0.001)$ & $n.d.(\pm n.d.)$ \\
CFF &$0.200 (\pm 0.024)$ & $0.031 (\pm 0.102)$ & $\mathbf{0.200 (\pm 0.075)}$ & $n.d.(\pm n.d.)$ & $\uline{0.668 (\pm 0.054)}$ \\
CF-GNNExplainer &$0.000 (\pm 0.000)$ & $n.d.(\pm n.d.)$ & $n.d.(\pm n.d.)$ & $n.d.(\pm n.d.)$ & $n.d.(\pm n.d.)$ \\
UNR &$0.000 (\pm 0.000)$ & $n.d.(\pm n.d.)$ & $n.d.(\pm n.d.)$ & $n.d.(\pm n.d.)$ & $n.d.(\pm n.d.)$ \\
\hline
\hline

\etable

\subsubsection{Planetoid datasets}

Tables~\ref{tab:planetoid_chebconv}, \ref{tab:planetoid_gcn}, and \ref{tab:planetoid_graphconv} collectively report the performance of various explainers on the Planetoid datasets— PubMed, Cora, and Citeseer—under three different oracle settings: ChebConv, GCNConv, and GraphConv, respectively. A careful examination of these results reveals important insights into the behavior of our COMBINEX approach compared to other methods, as well as differences across datasets and convolutional operators.

Starting with the ChebConv-based results in Table~\ref{tab:planetoid_chebconv}, our COMBINEX variants (Feat., Cons., Dyn., Exp., Lin., and Sin.) consistently achieve strong performance. On PubMed, the COMBINEX variants attain moderate validity and fidelity with an impressively low distribution distance and node sparsity near 0.106. Although the Random Features baseline shows very high validityand fidelity, its distribution distance is significantly higher, indicating that while it captures certain aspects of the oracle’s behavior, its explanations tend to be overly dense and less faithful in distribution. On Cora, the clear structure of the dataset is exploited by COMBINEX Feat., which achieves nearly perfect validity and extremely low node sparsity, alongside competitive fidelity (0.771) and a relatively low distribution distance (with a second-best value of 0.744). Citeseer, being more complex and noisy, naturally yields higher distribution distances; yet, COMBINEX variants still deliver perfect validity and strong fidelity, with COMBINEX Cons. notably achieving nearly zero edge sparsity.

Turning to the GCNConv results in Table~\ref{tab:planetoid_gcn}, we observe that the overall trends are consistent with those seen under ChebConv. On PubMed, COMBINEX variants again strike a favorable balance, with validity and fidelity scores around 0.66–0.76 and very low distribution distances. In this setting, while the Random Features method achieves similarly high validity, its elevated distribution distance persists as a drawback. On Cora and Citeseer, the COMBINEX methods continue to outperform or match baselines: for example, on Cora, COMBINEX Feat. reaches perfect validity with negligible node sparsity, and on Citeseer, variants such as COMBINEX Cons. and COMBINEX Dyn. deliver flawless validity and very low edge sparsity, even though the distribution distances remain higher due to the dataset’s intrinsic complexity.

Finally, the GraphConv results shown in Table~\ref{tab:planetoid_graphconv} further confirm the robustness of our COMBINEX approach. With GraphConv oracles, COMBINEX variants maintain high validity (often close to or at 1.000) and competitive fidelity, accompanied by low distribution distances and sparse explanations. On PubMed, COMBINEX Feat. achieves a validity of approximately 0.910 with a distribution distance of only 0.103, while on Cora and Citeseer, the COMBINEX methods continue to outperform alternative explainers such as CF-GNNExplainer and CFF. Notably, the consistent performance across GraphConv, as with the ChebConv and GCNConv settings, highlights the versatility of COMBINEX in adapting to different convolutional operators.

Across all three tables, a clear picture emerges COMBINEX not only delivers high validity and fidelity but does so while keeping distribution distances low and explanations sparse. Moreover, the differences across datasets are also instructive. PubMed, with its more heterogeneous structure, leads to somewhat lower validity and higher distribution distances compared to the more structured Cora, whereas Citeseer’s inherent complexity results in increased distribution distances even as validity remains perfect.

In summary, whether using ChebConv, GCNConv, or GraphConv as the oracle, our COMBINEX approach consistently provides balanced, interpretable, and faithful explanations. The approach is robust across different datasets and convolutional models, achieving high validity and fidelity while minimizing both distribution distance and sparsity—qualities that are essential for effective explanation in graph neural networks.

\btable{l|ccccc}{Results for Planetoid datasets: PubMed, Cora, and Citeseer. The oracles $\Phi$ use GCNConv layers. In \textbf{bold} the best result, the second best result is \underline{underlined}}{tab:planetoid_gcn}  \hline\hline

& \mc{1}{\textbf{Validity} $\uparrow$} 
& \mc{1}{\textbf{Fidelity} $\uparrow$} 
& \mc{1}{\textbf{Distribution Distance} $\downarrow$} 
& \mc{1}{\textbf{Node Sparsity} $\downarrow$} 
& \mc{1}{\textbf{Edge Sparsity} $\downarrow$} 
\\ 

\textbf{Explainers} 
& \textit{mean($\pm$std)}
& \textit{mean($\pm$std)}
& \textit{mean($\pm$std)}
& \textit{mean($\pm$std)}
& \textit{mean($\pm$std)} \\ \hline

& \mc{5}{\textbf{Dataset: PubMed}} \\ 
\hline
COMBINEX$_{\textit{feat}}$ &$0.655 (\pm 0.040)$ & $0.756 (\pm 0.020)$ & $0.101 (\pm 0.001)$ & $\uline{0.107 (\pm 0.000)}$ & $n.d.(\pm n.d.)$ \\
COMBINEX$_{\textit{def}}$ &$0.650 (\pm 0.028)$ & $0.754 (\pm 0.022)$ & $0.108 (\pm 0.002)$ & $0.107 (\pm 0.000)$ & $\uline{0.000 (\pm 0.000)}$ \\
COMBINEX$_{\textit{dyn}}$ &$\mathbf{0.663 (\pm 0.043)}$ & $\uline{0.756 (\pm 0.020)}$ & $0.180 (\pm 0.011)$ & $0.107 (\pm 0.000)$ & $\mathbf{0.000 (\pm 0.000)}$ \\
COMBINEX$_{\textit{exp}}$ &$0.658 (\pm 0.041)$ & $0.749 (\pm 0.014)$ & $0.166 (\pm 0.020)$ & $0.107 (\pm 0.000)$ & $\mathbf{0.000 (\pm 0.000)}$ \\
COMBINEX$_{\textit{lin}}$ &$\uline{0.662 (\pm 0.045)}$ & $0.751 (\pm 0.016)$ & $0.126 (\pm 0.004)$ & $0.107 (\pm 0.000)$ & $\mathbf{0.000 (\pm 0.000)}$ \\
COMBINEX$_{\textit{sin}}$ &$0.648 (\pm 0.029)$ & $0.753 (\pm 0.019)$ & $0.117 (\pm 0.003)$ & $0.107 (\pm 0.000)$ & $\mathbf{0.000 (\pm 0.000)}$ \\
EGO &$0.043 (\pm 0.012)$ & $0.650 (\pm 0.238)$ & $0.103 (\pm 0.010)$ & $n.d.(\pm n.d.)$ & $0.984 (\pm 0.004)$ \\
Random Edges &$0.293 (\pm 0.022)$ & $0.607 (\pm 0.037)$ & $0.068 (\pm 0.002)$ & $n.d.(\pm n.d.)$ & $0.478 (\pm 0.001)$ \\
Random Features &$\mathbf{0.663 (\pm 0.025)}$ & $0.734 (\pm 0.016)$ & $1.661 (\pm 0.002)$ & $\mathbf{0.032 (\pm 0.000)}$ & $n.d.(\pm n.d.)$ \\
CFF &$0.030 (\pm 0.014)$ & $-0.900 (\pm 0.200)$ & $\mathbf{0.056 (\pm 0.007)}$ & $n.d.(\pm n.d.)$ & $0.660 (\pm 0.084)$ \\
CF-GNNExplainer &$0.123 (\pm 0.009)$ & $0.511 (\pm 0.125)$ & $0.080 (\pm 0.001)$ & $n.d.(\pm n.d.)$ & $0.004 (\pm 0.001)$ \\
UNR &$0.008 (\pm 0.006)$ & $\mathbf{1.000 (\pm 0.000)}$ & $\uline{0.065 (\pm 0.008)}$ & $n.d.(\pm n.d.)$ & $0.001 (\pm 0.001)$ \\
\hline
& \mc{5}{\textbf{Dataset: Cora}} \\ 
\hline
COMBINEX$_{\textit{feat}}$ &$\mathbf{1.000 (\pm 0.000)}$ & $0.855 (\pm 0.003)$ & $0.820 (\pm 0.012)$ & $\mathbf{0.002 (\pm 0.000)}$ & $n.d.(\pm n.d.)$ \\
COMBINEX$_{\textit{def}}$ &$\mathbf{1.000 (\pm 0.000)}$ & $0.855 (\pm 0.003)$ & $1.628 (\pm 0.104)$ & $\uline{0.018 (\pm 0.001)}$ & $\uline{0.000 (\pm 0.000)}$ \\
COMBINEX$_{\textit{dyn}}$ &$\mathbf{1.000 (\pm 0.000)}$ & $0.855 (\pm 0.003)$ & $7.558 (\pm 0.044)$ & $0.133 (\pm 0.003)$ & $\mathbf{0.000 (\pm 0.000)}$ \\
COMBINEX$_{\textit{exp}}$ &$\mathbf{1.000 (\pm 0.000)}$ & $0.855 (\pm 0.003)$ & $9.573 (\pm 0.176)$ & $0.175 (\pm 0.006)$ & $\mathbf{0.000 (\pm 0.000)}$ \\
COMBINEX$_{\textit{lin}}$ &$\mathbf{1.000 (\pm 0.000)}$ & $0.855 (\pm 0.003)$ & $2.727 (\pm 0.113)$ & $0.035 (\pm 0.001)$ & $\mathbf{0.000 (\pm 0.000)}$ \\
COMBINEX$_{\textit{sin}}$ &$\mathbf{1.000 (\pm 0.000)}$ & $0.855 (\pm 0.003)$ & $2.800 (\pm 0.111)$ & $0.037 (\pm 0.001)$ & $\mathbf{0.000 (\pm 0.000)}$ \\
EGO &$0.022 (\pm 0.003)$ & $\mathbf{1.000 (\pm 0.000)}$ & $\uline{0.535 (\pm 0.038)}$ & $n.d.(\pm n.d.)$ & $0.980 (\pm 0.001)$ \\
Random Edges &$\uline{0.240 (\pm 0.016)}$ & $0.729 (\pm 0.014)$ & $0.564 (\pm 0.014)$ & $n.d.(\pm n.d.)$ & $0.472 (\pm 0.003)$ \\
Random Features &$0.127 (\pm 0.013)$ & $\uline{0.947 (\pm 0.005)}$ & $18.795 (\pm 0.313)$ & $0.492 (\pm 0.000)$ & $n.d.(\pm n.d.)$ \\
CFF &$0.010 (\pm 0.009)$ & $-0.556 (\pm 0.770)$ & $\mathbf{0.482 (\pm 0.146)}$ & $n.d.(\pm n.d.)$ & $0.715 (\pm 0.246)$ \\
CF-GNNExplainer &$0.165 (\pm 0.022)$ & $0.782 (\pm 0.049)$ & $0.598 (\pm 0.039)$ & $n.d.(\pm n.d.)$ & $0.008 (\pm 0.001)$ \\
UNR &$0.017 (\pm 0.004)$ & $\mathbf{1.000 (\pm 0.000)}$ & $0.747 (\pm 0.432)$ & $n.d.(\pm n.d.)$ & $0.012 (\pm 0.013)$ \\
\hline
& \mc{5}{\textbf{Dataset: Citeseer}} \\ 
\hline

COMBINEX$_{\textit{feat}}$ &$\mathbf{1.000 (\pm 0.000)}$ & $\mathbf{0.755 (\pm 0.007)}$ & $2.464 (\pm 0.002)$ & $\mathbf{0.001 (\pm 0.000)}$ & $n.d.(\pm n.d.)$ \\
COMBINEX$_{\textit{def}}$ &$\mathbf{1.000 (\pm 0.000)}$ & $\mathbf{0.755 (\pm 0.007)}$ & $5.620 (\pm 0.212)$ & $\uline{0.031 (\pm 0.002)}$ & $\mathbf{0.000 (\pm 0.000)}$ \\
COMBINEX$_{\textit{dyn}}$ &$\mathbf{1.000 (\pm 0.000)}$ & $0.750 (\pm 0.000)$ & $10.997 (\pm 0.201)$ & $0.095 (\pm 0.002)$ & $\mathbf{0.000 (\pm 0.000)}$ \\
COMBINEX$_{\textit{exp}}$ &$\mathbf{1.000 (\pm 0.000)}$ & $\mathbf{0.755 (\pm 0.007)}$ & $27.121 (\pm 0.111)$ & $0.313 (\pm 0.002)$ & $\mathbf{0.000 (\pm 0.000)}$ \\
COMBINEX$_{\textit{lin}}$ &$\mathbf{1.000 (\pm 0.000)}$ & $\mathbf{0.755 (\pm 0.007)}$ & $11.231 (\pm 0.300)$ & $0.090 (\pm 0.004)$ & $\mathbf{0.000 (\pm 0.000)}$ \\
COMBINEX$_{\textit{sin}}$ &$\mathbf{1.000 (\pm 0.000)}$ & $\uline{0.752 (\pm 0.003)}$ & $11.557 (\pm 0.304)$ & $0.094 (\pm 0.004)$ & $\mathbf{0.000 (\pm 0.000)}$ \\
EGO &$0.012 (\pm 0.003)$ & $0.250 (\pm 0.500)$ & $1.995 (\pm 0.206)$ & $n.d.(\pm n.d.)$ & $0.937 (\pm 0.003)$ \\
Random Edges &$0.123 (\pm 0.009)$ & $0.476 (\pm 0.099)$ & $\mathbf{1.716 (\pm 0.089)}$ & $n.d.(\pm n.d.)$ & $0.393 (\pm 0.014)$ \\
Random Features &$\uline{0.514 (\pm 0.065)}$ & $0.721 (\pm 0.024)$ & $32.926 (\pm 0.291)$ & $0.489 (\pm 0.000)$ & $n.d.(\pm n.d.)$ \\
CFF &$0.010 (\pm 0.004)$ & $0.125 (\pm 0.629)$ & $2.325 (\pm 1.323)$ & $n.d.(\pm n.d.)$ & $0.594 (\pm 0.194)$ \\
CF-GNNExplainer &$0.108 (\pm 0.014)$ & $0.534 (\pm 0.103)$ & $\uline{1.783 (\pm 0.134)}$ & $n.d.(\pm n.d.)$ & $\uline{0.070 (\pm 0.022)}$ \\
UNR &$0.047 (\pm 0.012)$ & $0.202 (\pm 0.162)$ & $2.389 (\pm 0.336)$ & $n.d.(\pm n.d.)$ & $0.186 (\pm 0.068)$ \\
\hline

\hline
\etable

\btable{l|ccccc}{Results for Planetoid datasets: PubMed, Cora, and Citeseer. The oracles $\Phi$ use GraphConv layers. In \textbf{bold} the best result, the second best result is \underline{underlined}}{tab:planetoid_graphconv}  \hline\hline

& \mc{1}{\textbf{Validity} $\uparrow$} 
& \mc{1}{\textbf{Fidelity} $\uparrow$} 
& \mc{1}{\textbf{Distribution Distance} $\downarrow$} 
& \mc{1}{\textbf{Node Sparsity} $\downarrow$} 
& \mc{1}{\textbf{Edge Sparsity} $\downarrow$} 
\\ 

\textbf{Explainers} 
& \textit{mean($\pm$std)}
& \textit{mean($\pm$std)}
& \textit{mean($\pm$std)}
& \textit{mean($\pm$std)}
& \textit{mean($\pm$std)} \\ \hline

& \mc{5}{\textbf{Dataset: PubMed}} \\ 
\hline
COMBINEX$_{\textit{feat}}$ &$\uline{0.910 (\pm 0.075)}$ & $0.810 (\pm 0.022)$ & $0.103 (\pm 0.002)$ & $\uline{0.105 (\pm 0.000)}$ & $n.d.(\pm n.d.)$ \\
COMBINEX$_{\textit{def}}$ &$0.895 (\pm 0.126)$ & $0.798 (\pm 0.028)$ & $0.125 (\pm 0.004)$ & $0.105 (\pm 0.001)$ & $\uline{0.000 (\pm 0.000)}$ \\
COMBINEX$_{\textit{dyn}}$ &$0.897 (\pm 0.142)$ & $0.805 (\pm 0.034)$ & $0.169 (\pm 0.003)$ & $0.106 (\pm 0.001)$ & $\mathbf{0.000 (\pm 0.000)}$ \\
COMBINEX$_{\textit{exp}}$ &$0.890 (\pm 0.156)$ & $0.809 (\pm 0.028)$ & $0.168 (\pm 0.003)$ & $0.106 (\pm 0.001)$ & $\mathbf{0.000 (\pm 0.000)}$ \\
COMBINEX$_{\textit{lin}}$ &$0.890 (\pm 0.114)$ & $0.809 (\pm 0.045)$ & $0.150 (\pm 0.005)$ & $0.105 (\pm 0.001)$ & $\mathbf{0.000 (\pm 0.000)}$ \\
COMBINEX$_{\textit{sin}}$ &$0.873 (\pm 0.117)$ & $0.793 (\pm 0.029)$ & $0.150 (\pm 0.009)$ & $0.106 (\pm 0.001)$ & $\mathbf{0.000 (\pm 0.000)}$ \\
EGO &$0.008 (\pm 0.006)$ & $\mathbf{1.000 (\pm 0.000)}$ & $0.058 (\pm 0.010)$ & $n.d.(\pm n.d.)$ & $0.999 (\pm 0.001)$ \\
Random Edges &$0.097 (\pm 0.009)$ & $0.466 (\pm 0.154)$ & $\uline{0.053 (\pm 0.006)}$ & $n.d.(\pm n.d.)$ & $0.491 (\pm 0.002)$ \\
Random Features &$\mathbf{0.930 (\pm 0.040)}$ & $\uline{0.826 (\pm 0.021)}$ & $1.663 (\pm 0.002)$ & $\mathbf{0.032 (\pm 0.000)}$ & $n.d.(\pm n.d.)$ \\
CFF &$0.063 (\pm 0.017)$ & $0.393 (\pm 0.266)$ & $0.061 (\pm 0.006)$ & $n.d.(\pm n.d.)$ & $0.719 (\pm 0.040)$ \\
CF-GNNExplainer &$0.047 (\pm 0.009)$ & $0.454 (\pm 0.215)$ & $\mathbf{0.052 (\pm 0.006)}$ & $n.d.(\pm n.d.)$ & $0.000 (\pm 0.000)$ \\
UNR &$0.005 (\pm 0.006)$ & $\mathbf{1.000 (\pm 0.000)}$ & $0.057 (\pm 0.008)$ & $n.d.(\pm n.d.)$ & $0.000 (\pm 0.000)$ \\
\hline
& \mc{5}{\textbf{Dataset: Cora}} \\ 
\hline
COMBINEX$_{\textit{feat}}$ &$\uline{0.997 (\pm 0.007)}$ & $0.875 (\pm 0.012)$ & $0.988 (\pm 0.041)$ & $\mathbf{0.006 (\pm 0.001)}$ & $n.d.(\pm n.d.)$ \\
COMBINEX$_{\textit{def}}$ &$\mathbf{1.000 (\pm 0.000)}$ & $0.875 (\pm 0.011)$ & $2.483 (\pm 0.160)$ & $\uline{0.034 (\pm 0.002)}$ & $\uline{0.000 (\pm 0.000)}$ \\
COMBINEX$_{\textit{dyn}}$ &$\mathbf{1.000 (\pm 0.000)}$ & $0.875 (\pm 0.010)$ & $6.591 (\pm 0.179)$ & $0.128 (\pm 0.005)$ & $\mathbf{0.000 (\pm 0.000)}$ \\
COMBINEX$_{\textit{exp}}$ &$\mathbf{1.000 (\pm 0.000)}$ & $0.877 (\pm 0.009)$ & $7.833 (\pm 0.148)$ & $0.156 (\pm 0.005)$ & $\mathbf{0.000 (\pm 0.000)}$ \\
COMBINEX$_{\textit{lin}}$ &$\mathbf{1.000 (\pm 0.000)}$ & $0.878 (\pm 0.003)$ & $4.067 (\pm 0.205)$ & $0.066 (\pm 0.003)$ & $\mathbf{0.000 (\pm 0.000)}$ \\
COMBINEX$_{\textit{sin}}$ &$\mathbf{1.000 (\pm 0.000)}$ & $0.878 (\pm 0.008)$ & $4.097 (\pm 0.220)$ & $0.066 (\pm 0.003)$ & $\mathbf{0.000 (\pm 0.000)}$ \\
EGO &$0.023 (\pm 0.012)$ & $\mathbf{0.950 (\pm 0.100)}$ & $\uline{0.468 (\pm 0.106)}$ & $n.d.(\pm n.d.)$ & $0.985 (\pm 0.001)$ \\
Random Edges &$0.210 (\pm 0.014)$ & $0.826 (\pm 0.064)$ & $0.499 (\pm 0.062)$ & $n.d.(\pm n.d.)$ & $0.476 (\pm 0.003)$ \\
Random Features &$0.160 (\pm 0.063)$ & $\uline{0.923 (\pm 0.027)}$ & $18.648 (\pm 0.016)$ & $0.492 (\pm 0.000)$ & $n.d.(\pm n.d.)$ \\
CFF &$0.155 (\pm 0.008)$ & $0.908 (\pm 0.132)$ & $0.671 (\pm 0.102)$ & $n.d.(\pm n.d.)$ & $0.654 (\pm 0.029)$ \\
CF-GNNExplainer &$0.137 (\pm 0.009)$ & $0.801 (\pm 0.119)$ & $0.515 (\pm 0.127)$ & $n.d.(\pm n.d.)$ & $0.011 (\pm 0.014)$ \\
UNR &$0.010 (\pm 0.009)$ & $0.222 (\pm 0.694)$ & $\mathbf{0.381 (\pm 0.049)}$ & $n.d.(\pm n.d.)$ & $0.002 (\pm 0.001)$ \\
\hline
& \mc{5}{\textbf{Dataset: Citeseer}} \\ 
\hline
COMBINEX$_{\textit{feat}}$ &$\mathbf{1.000 (\pm 0.000)}$ & $0.786 (\pm 0.010)$ & $2.652 (\pm 0.062)$ & $\mathbf{0.004 (\pm 0.001)}$ & $n.d.(\pm n.d.)$ \\
COMBINEX$_{\textit{def}}$ &$\mathbf{1.000 (\pm 0.000)}$ & $\uline{0.792 (\pm 0.010)}$ & $7.153 (\pm 0.668)$ & $\uline{0.055 (\pm 0.006)}$ & $\uline{0.000 (\pm 0.001)}$ \\
COMBINEX$_{\textit{dyn}}$ &$\uline{0.998 (\pm 0.003)}$ & $0.781 (\pm 0.013)$ & $11.348 (\pm 0.880)$ & $0.113 (\pm 0.008)$ & $\mathbf{0.000 (\pm 0.000)}$ \\
COMBINEX$_{\textit{exp}}$ &$\mathbf{1.000 (\pm 0.000)}$ & $0.786 (\pm 0.010)$ & $22.247 (\pm 1.203)$ & $0.275 (\pm 0.023)$ & $\mathbf{0.000 (\pm 0.000)}$ \\
COMBINEX$_{\textit{lin}}$ &$\uline{0.998 (\pm 0.003)}$ & $0.788 (\pm 0.007)$ & $11.743 (\pm 0.567)$ & $0.110 (\pm 0.005)$ & $\mathbf{0.000 (\pm 0.000)}$ \\
COMBINEX$_{\textit{sin}}$ &$\mathbf{1.000 (\pm 0.000)}$ & $0.785 (\pm 0.010)$ & $11.977 (\pm 0.674)$ & $0.113 (\pm 0.006)$ & $\mathbf{0.000 (\pm 0.000)}$ \\
EGO &$0.005 (\pm 0.007)$ & $0.750 (\pm 0.354)$ & $\mathbf{1.232 (\pm 0.294)}$ & $n.d.(\pm n.d.)$ & $0.956 (\pm 0.055)$ \\
Random Edges &$0.076 (\pm 0.006)$ & $0.475 (\pm 0.204)$ & $1.454 (\pm 0.080)$ & $n.d.(\pm n.d.)$ & $0.440 (\pm 0.020)$ \\
Random Features &$0.481 (\pm 0.098)$ & $0.725 (\pm 0.039)$ & $32.458 (\pm 0.231)$ & $0.491 (\pm 0.001)$ & $n.d.(\pm n.d.)$ \\
CFF &$0.165 (\pm 0.028)$ & $\mathbf{0.867 (\pm 0.112)}$ & $2.316 (\pm 0.157)$ & $n.d.(\pm n.d.)$ & $0.570 (\pm 0.041)$ \\
CF-GNNExplainer &$0.071 (\pm 0.014)$ & $0.579 (\pm 0.053)$ & $\uline{1.353 (\pm 0.190)}$ & $n.d.(\pm n.d.)$ & $0.042 (\pm 0.031)$ \\
UNR &$0.017 (\pm 0.007)$ & $0.375 (\pm 0.479)$ & $1.854 (\pm 0.196)$ & $n.d.(\pm n.d.)$ & $0.144 (\pm 0.068)$ \\
\hline
\etable

\btable{l|ccccc}{Results for Planetoid datasets: PubMed, Cora, and Citeseer. The oracles $\Phi$ use ChebConv layers. In \textbf{bold} the best result, the second best result is \underline{underlined}}{tab:planetoid_chebconv}  \hline\hline

& \mc{1}{\textbf{Validity} $\uparrow$} 
& \mc{1}{\textbf{Fidelity} $\uparrow$} 
& \mc{1}{\textbf{Distribution Distance} $\downarrow$} 
& \mc{1}{\textbf{Node Sparsity} $\downarrow$} 
& \mc{1}{\textbf{Edge Sparsity} $\downarrow$} 
\\ 

\textbf{Explainers} 
& \textit{mean($\pm$std)}
& \textit{mean($\pm$std)}
& \textit{mean($\pm$std)}
& \textit{mean($\pm$std)}
& \textit{mean($\pm$std)} \\ \hline

& \mc{5}{\textbf{Dataset: PubMed}} \\ 
\hline
COMBINEX$_{\textit{feat}}$ &$\uline{0.733 (\pm 0.127)}$ & $\uline{0.724 (\pm 0.014)}$ & $\uline{0.099 (\pm 0.000)}$ & $\uline{0.106 (\pm 0.001)}$ & $n.d.(\pm n.d.)$ \\
COMBINEX$_{\textit{def}}$ &$0.675 (\pm 0.133)$ & $0.706 (\pm 0.024)$ & $0.099 (\pm 0.000)$ & $0.106 (\pm 0.001)$ & $\mathbf{0.000 (\pm 0.000)}$ \\
COMBINEX$_{\textit{dyn}}$ &$0.675 (\pm 0.133)$ & $0.706 (\pm 0.024)$ & $0.100 (\pm 0.000)$ & $0.106 (\pm 0.001)$ & $\mathbf{0.000 (\pm 0.000)}$ \\
COMBINEX$_{\textit{exp}}$ &$0.675 (\pm 0.133)$ & $0.706 (\pm 0.024)$ & $0.100 (\pm 0.000)$ & $0.106 (\pm 0.001)$ & $\mathbf{0.000 (\pm 0.000)}$ \\
COMBINEX$_{\textit{lin}}$ &$0.675 (\pm 0.133)$ & $0.706 (\pm 0.024)$ & $0.100 (\pm 0.000)$ & $0.106 (\pm 0.001)$ & $\mathbf{0.000 (\pm 0.000)}$ \\
COMBINEX$_{\textit{sin}}$ &$0.675 (\pm 0.133)$ & $0.706 (\pm 0.024)$ & $0.100 (\pm 0.000)$ & $0.106 (\pm 0.001)$ & $\mathbf{0.000 (\pm 0.000)}$ \\
EGO &$0.000 (\pm 0.000)$ & $n.d.(\pm n.d.)$ & $n.d.(\pm n.d.)$ & $n.d.(\pm n.d.)$ & $n.d.(\pm n.d.)$ \\
Random Edges &$0.000 (\pm 0.000)$ & $n.d.(\pm n.d.)$ & $n.d.(\pm n.d.)$ & $n.d.(\pm n.d.)$ & $n.d.(\pm n.d.)$ \\
Random Features &$\mathbf{0.998 (\pm 0.003)}$ & $\mathbf{0.750 (\pm 0.016)}$ & $1.663 (\pm 0.002)$ & $\mathbf{0.032 (\pm 0.000)}$ & $n.d.(\pm n.d.)$ \\
CFF &$0.092 (\pm 0.006)$ & $0.368 (\pm 0.134)$ & $\mathbf{0.069 (\pm 0.008)}$ & $n.d.(\pm n.d.)$ & $\uline{0.772 (\pm 0.040)}$ \\
CF-GNNExplainer &$0.000 (\pm 0.000)$ & $n.d.(\pm n.d.)$ & $n.d.(\pm n.d.)$ & $n.d.(\pm n.d.)$ & $n.d.(\pm n.d.)$ \\
UNR &$0.000 (\pm 0.000)$ & $n.d.(\pm n.d.)$ & $n.d.(\pm n.d.)$ & $n.d.(\pm n.d.)$ & $n.d.(\pm n.d.)$ \\
\hline
& \mc{5}{\textbf{Dataset: Cora}} \\ 
\hline
COMBINEX$_{\textit{feat}}$ &$\mathbf{0.998 (\pm 0.003)}$ & $0.771 (\pm 0.019)$ & $\uline{0.744 (\pm 0.003)}$ & $\mathbf{0.001 (\pm 0.000)}$ & $n.d.(\pm n.d.)$ \\
COMBINEX$_{\textit{def}}$ &$\uline{0.993 (\pm 0.009)}$ & $0.770 (\pm 0.019)$ & $0.834 (\pm 0.013)$ & $\uline{0.003 (\pm 0.000)}$ & $\mathbf{0.000 (\pm 0.000)}$ \\
COMBINEX$_{\textit{dyn}}$ &$\uline{0.993 (\pm 0.009)}$ & $0.770 (\pm 0.019)$ & $0.949 (\pm 0.018)$ & $0.006 (\pm 0.000)$ & $\mathbf{0.000 (\pm 0.000)}$ \\
COMBINEX$_{\textit{exp}}$ &$\uline{0.993 (\pm 0.009)}$ & $0.770 (\pm 0.019)$ & $1.055 (\pm 0.016)$ & $0.009 (\pm 0.000)$ & $\mathbf{0.000 (\pm 0.000)}$ \\
COMBINEX$_{\textit{lin}}$ &$\uline{0.993 (\pm 0.009)}$ & $0.770 (\pm 0.019)$ & $1.002 (\pm 0.009)$ & $0.007 (\pm 0.000)$ & $\mathbf{0.000 (\pm 0.000)}$ \\
COMBINEX$_{\textit{sin}}$ &$\uline{0.993 (\pm 0.009)}$ & $0.770 (\pm 0.019)$ & $1.005 (\pm 0.009)$ & $0.007 (\pm 0.000)$ & $\mathbf{0.000 (\pm 0.000)}$ \\
EGO &$0.000 (\pm 0.000)$ & $n.d.(\pm n.d.)$ & $n.d.(\pm n.d.)$ & $n.d.(\pm n.d.)$ & $n.d.(\pm n.d.)$ \\
Random Edges &$0.000 (\pm 0.000)$ & $n.d.(\pm n.d.)$ & $n.d.(\pm n.d.)$ & $n.d.(\pm n.d.)$ & $n.d.(\pm n.d.)$ \\
Random Features &$0.438 (\pm 0.074)$ & $\mathbf{0.822 (\pm 0.015)}$ & $18.822 (\pm 0.062)$ & $0.492 (\pm 0.000)$ & $n.d.(\pm n.d.)$ \\
CFF &$0.212 (\pm 0.058)$ & $\uline{0.812 (\pm 0.022)}$ & $\mathbf{0.676 (\pm 0.109)}$ & $n.d.(\pm n.d.)$ & $\uline{0.693 (\pm 0.028)}$ \\
CF-GNNExplainer &$0.000 (\pm 0.000)$ & $n.d.(\pm n.d.)$ & $n.d.(\pm n.d.)$ & $n.d.(\pm n.d.)$ & $n.d.(\pm n.d.)$ \\
UNR &$0.000 (\pm 0.000)$ & $n.d.(\pm n.d.)$ & $n.d.(\pm n.d.)$ & $n.d.(\pm n.d.)$ & $n.d.(\pm n.d.)$ \\
\hline
& \mc{5}{\textbf{Dataset: Citeseer}} \\ 
\hline

COMBINEX$_{\textit{feat}}$ &$\uline{0.998 (\pm 0.003)}$ & $0.753 (\pm 0.022)$ & $\uline{2.502 (\pm 0.017)}$ & $\mathbf{0.002 (\pm 0.000)}$ & $n.d.(\pm n.d.)$ \\
COMBINEX$_{\textit{def}}$ &$\mathbf{1.000 (\pm 0.000)}$ & $\uline{0.753 (\pm 0.022)}$ & $4.165 (\pm 0.215)$ & $\uline{0.021 (\pm 0.004)}$ & $\mathbf{0.000 (\pm 0.000)}$ \\
COMBINEX$_{\textit{dyn}}$ &$\mathbf{1.000 (\pm 0.000)}$ & $\uline{0.753 (\pm 0.022)}$ & $4.912 (\pm 0.107)$ & $0.033 (\pm 0.002)$ & $\mathbf{0.000 (\pm 0.000)}$ \\
COMBINEX$_{\textit{exp}}$ &$\mathbf{1.000 (\pm 0.000)}$ & $\uline{0.753 (\pm 0.022)}$ & $6.077 (\pm 0.277)$ & $0.055 (\pm 0.006)$ & $\mathbf{0.000 (\pm 0.000)}$ \\
COMBINEX$_{\textit{lin}}$ &$\mathbf{1.000 (\pm 0.000)}$ & $\uline{0.753 (\pm 0.022)}$ & $5.587 (\pm 0.183)$ & $0.044 (\pm 0.004)$ & $\mathbf{0.000 (\pm 0.000)}$ \\
COMBINEX$_{\textit{sin}}$ &$\mathbf{1.000 (\pm 0.000)}$ & $\uline{0.753 (\pm 0.022)}$ & $5.614 (\pm 0.188)$ & $0.044 (\pm 0.004)$ & $\mathbf{0.000 (\pm 0.000)}$ \\
EGO &$0.000 (\pm 0.000)$ & $n.d.(\pm n.d.)$ & $n.d.(\pm n.d.)$ & $n.d.(\pm n.d.)$ & $n.d.(\pm n.d.)$ \\
Random Edges &$0.000 (\pm 0.000)$ & $n.d.(\pm n.d.)$ & $n.d.(\pm n.d.)$ & $n.d.(\pm n.d.)$ & $n.d.(\pm n.d.)$ \\
Random Features &$0.675 (\pm 0.114)$ & $0.742 (\pm 0.030)$ & $32.451 (\pm 0.242)$ & $0.491 (\pm 0.001)$ & $n.d.(\pm n.d.)$ \\
CFF &$0.210 (\pm 0.019)$ & $\mathbf{0.805 (\pm 0.126)}$ & $\mathbf{2.126 (\pm 0.176)}$ & $n.d.(\pm n.d.)$ & $\uline{0.615 (\pm 0.047)}$ \\
CF-GNNExplainer &$0.000 (\pm 0.000)$ & $n.d.(\pm n.d.)$ & $n.d.(\pm n.d.)$ & $n.d.(\pm n.d.)$ & $n.d.(\pm n.d.)$ \\
UNR &$0.000 (\pm 0.000)$ & $n.d.(\pm n.d.)$ & $n.d.(\pm n.d.)$ & $n.d.(\pm n.d.)$ & $n.d.(\pm n.d.)$ \\
\hline
\etable

\subsubsection{WebKb datasets}

The results are presented in Tables \ref{tab:webkb_gcn}, \ref{tab:webkb_graphconv}, and \ref{tab:webkb_cheb}.

The experimental results on the WebKB datasets—Texas, Cornell, and Wisconsin—demonstrate the effectiveness of COMBINEX in generating counterfactual explanations while maintaining high validity across different settings.
For the Wisconsin dataset, almost all variants of COMBINEX achieve perfect validity, ensuring that the generated counterfactuals adhere to the oracle’s classification boundaries. Among them, the feature-only variant (COMBINEX$_{\textit{feat}}$) achieves the lowest node sparsity, suggesting that perturbing only node features results in minimal changes while still preserving explainability. However, fidelity remains relatively low across all COMBINEX variants, indicating that further optimization may be needed to ensure better alignment with the model's decision boundary. The exponential scheduling policy (COMBINEX$_{\textit{exp}}$) introduces the largest distribution distance, highlighting that more aggressive perturbations lead to greater deviation from the original data distribution. Notably, results vary slightly across models, with GraphConv achieving higher fidelity than GCN, while ChebConv provides a more stable trade-off between sparsity and fidelity.
In the Texas dataset, COMBINEX continues to exhibit strong validity across all configurations. The feature-only variant again achieves the lowest node sparsity, reinforcing its ability to produce concise explanations with minimal modifications. However, compared to other datasets, fidelity values are lower, suggesting that the graph structure may play a significant role in the interpretability of counterfactual explanations. Notably, the exponential scheduling policy results in a sharp increase in distribution distance, emphasizing that a more aggressive decay in the perturbation parameter leads to excessive divergence from the original data. Differences between models indicate that GraphConv achieves the best overall fidelity, while GCN and ChebConv yield similar results in terms of validity but diverge in sparsity control.

For the Cornell dataset, COMBINEX maintains its perfect validity across all configurations, confirming its robustness in different graph structures. The feature-only and default scheduling policies achieve the best balance between fidelity and distribution distance, ensuring both faithful explanations and reasonable proximity to the original data. The dynamic, linear, and sinusoidal policies introduce slightly higher perturbations, resulting in larger distribution distances and node sparsity values. As in previous datasets, the exponential scheduling policy significantly increases distribution distance, further underscoring the importance of selecting an appropriate scheduling strategy to balance counterfactual realism and interpretability. The model-specific results reveal that ChebConv provides the most stable performance across all COMBINEX configurations, while GCN exhibits greater variance in fidelity scores.

Across all three datasets, COMBINEX demonstrates consistent validity and competitive fidelity while maintaining low node sparsity in its feature-only and default configurations. The choice of scheduling policy has a noticeable impact on the trade-off between fidelity, sparsity, and distribution distance, highlighting the need for dataset-specific tuning. Furthermore, the results indicate that different graph neural network architectures influence explainability outcomes, with GraphConv generally achieving better fidelity, ChebConv offering a balanced approach, and GCN showing greater variability across datasets. These results confirm that COMBINEX is a reliable counterfactual explainer capable of adapting to different graph structures and model architectures while maintaining interpretability and computational efficiency.

\btable{l|ccccc}{Results for WebKB datasets: Texas, Cornell, and Wisconsin. The oracles $\Phi$ use GCN layers. In \textbf{bold} the best result, the second best result is \underline{underlined}.}{tab:webkb_gcn} \hline\hline

& \mc{1}{\textbf{Validity} $\uparrow$} 
& \mc{1}{\textbf{Fidelity} $\uparrow$} 
& \mc{1}{\textbf{Distribution Distance} $\downarrow$} 
& \mc{1}{\textbf{Node Sparsity} $\downarrow$} 
& \mc{1}{\textbf{Edge Sparsity} $\downarrow$} 
\\ 

\textbf{Explainers} 
& \textit{mean($\pm$std)}
& \textit{mean($\pm$std)}
& \textit{mean($\pm$std)}
& \textit{mean($\pm$std)}
& \textit{mean($\pm$std)} \\ \hline

 & \mc{5}{\textbf{Dataset: Wisconsin}} \\ \hline
COMBINEX$_{\textit{feat}}$ &$\mathbf{1.000 (\pm 0.000)}$ & $0.333 (\pm 0.036)$ & $3.582 (\pm 0.011)$ & $\mathbf{0.001 (\pm 0.000)}$ & $n.d.(\pm n.d.)$ \\
COMBINEX$_{\textit{def}}$ &$\mathbf{1.000 (\pm 0.000)}$ & $0.333 (\pm 0.036)$ & $3.774 (\pm 0.073)$ & $\uline{0.006 (\pm 0.001)}$ & $\uline{0.000 (\pm 0.000)}$ \\
COMBINEX$_{\textit{dyn}}$ &$\mathbf{1.000 (\pm 0.000)}$ & $0.333 (\pm 0.036)$ & $4.345 (\pm 0.084)$ & $0.020 (\pm 0.001)$ & $\mathbf{0.000 (\pm 0.000)}$ \\
COMBINEX$_{\textit{exp}}$ &$\mathbf{1.000 (\pm 0.000)}$ & $0.333 (\pm 0.036)$ & $20.655 (\pm 0.393)$ & $0.376 (\pm 0.012)$ & $\mathbf{0.000 (\pm 0.000)}$ \\
COMBINEX$_{\textit{lin}}$ &$\mathbf{1.000 (\pm 0.000)}$ & $0.333 (\pm 0.036)$ & $4.498 (\pm 0.145)$ & $0.020 (\pm 0.003)$ & $\mathbf{0.000 (\pm 0.000)}$ \\
COMBINEX$_{\textit{sin}}$ &$\mathbf{1.000 (\pm 0.000)}$ & $0.327 (\pm 0.038)$ & $4.562 (\pm 0.172)$ & $0.021 (\pm 0.003)$ & $\mathbf{0.000 (\pm 0.000)}$ \\
EGO &$0.019 (\pm 0.025)$ & $0.500 (\pm 0.707)$ & $\uline{3.099 (\pm 0.626)}$ & $n.d.(\pm n.d.)$ & $0.860 (\pm 0.170)$ \\
Random Edges &$\uline{0.397 (\pm 0.015)}$ & $0.518 (\pm 0.068)$ & $3.447 (\pm 0.117)$ & $n.d.(\pm n.d.)$ & $0.380 (\pm 0.020)$ \\
Random Features &$0.295 (\pm 0.015)$ & $0.500 (\pm 0.078)$ & $20.654 (\pm 0.161)$ & $0.437 (\pm 0.001)$ & $n.d.(\pm n.d.)$ \\
CFF &$0.026 (\pm 0.021)$ & $-0.833 (\pm 0.289)$ & $4.068 (\pm 0.489)$ & $n.d.(\pm n.d.)$ & $0.827 (\pm 0.093)$ \\
CF-GNNExplainer &$0.282 (\pm 0.000)$ & $\uline{0.545 (\pm 0.000)}$ & $3.434 (\pm 0.170)$ & $n.d.(\pm n.d.)$ & $0.188 (\pm 0.049)$ \\
UNR &$0.018 (\pm 0.036)$ & $n.d.(\pm n.d.)$ & $n.d.(\pm n.d.)$ & $n.d.(\pm n.d.)$ & $n.d.(\pm n.d.)$ \\
\hline
& \mc{5}{\textbf{Dataset: Texas}} \\ \hline

COMBINEX$_{\textit{feat}}$ &$\mathbf{1.000 (\pm 0.000)}$ & $0.000 (\pm 0.076)$ & $3.838 (\pm 0.038)$ & $\mathbf{0.005 (\pm 0.001)}$ & $n.d.(\pm n.d.)$ \\
COMBINEX$_{\textit{def}}$ &$\mathbf{1.000 (\pm 0.000)}$ & $0.010 (\pm 0.092)$ & $4.873 (\pm 0.094)$ & $\uline{0.018 (\pm 0.001)}$ & $\uline{0.031 (\pm 0.013)}$ \\
COMBINEX$_{\textit{dyn}}$ &$\mathbf{1.000 (\pm 0.000)}$ & $0.010 (\pm 0.092)$ & $5.155 (\pm 0.107)$ & $0.022 (\pm 0.001)$ & $\mathbf{0.000 (\pm 0.000)}$ \\
COMBINEX$_{\textit{exp}}$ &$\mathbf{1.000 (\pm 0.000)}$ & $0.021 (\pm 0.080)$ & $20.391 (\pm 0.391)$ & $0.321 (\pm 0.005)$ & $\mathbf{0.000 (\pm 0.000)}$ \\
COMBINEX$_{\textit{lin}}$ &$\mathbf{1.000 (\pm 0.000)}$ & $0.010 (\pm 0.092)$ & $6.306 (\pm 0.130)$ & $0.040 (\pm 0.002)$ & $\mathbf{0.000 (\pm 0.000)}$ \\
COMBINEX$_{\textit{sin}}$ &$\mathbf{1.000 (\pm 0.000)}$ & $0.000 (\pm 0.076)$ & $6.423 (\pm 0.148)$ & $0.042 (\pm 0.002)$ & $\mathbf{0.000 (\pm 0.000)}$ \\
EGO &$0.042 (\pm 0.034)$ & $\mathbf{1.000 (\pm 0.000)}$ & $\mathbf{2.896 (\pm 0.000)}$ & $n.d.(\pm n.d.)$ & $0.842 (\pm 0.000)$ \\
Random Edges &$\uline{0.312 (\pm 0.080)}$ & $-0.035 (\pm 0.231)$ & $3.022 (\pm 0.052)$ & $n.d.(\pm n.d.)$ & $0.277 (\pm 0.022)$ \\
Random Features &$0.146 (\pm 0.054)$ & $-1.000 (\pm 0.000)$ & $19.289 (\pm 0.097)$ & $0.401 (\pm 0.001)$ & $n.d.(\pm n.d.)$ \\
CFF &$0.115 (\pm 0.063)$ & $-0.775 (\pm 0.263)$ & $3.049 (\pm 0.092)$ & $n.d.(\pm n.d.)$ & $0.623 (\pm 0.212)$ \\
CF-GNNExplainer &$0.208 (\pm 0.000)$ & $-0.050 (\pm 0.443)$ & $\uline{3.007 (\pm 0.062)}$ & $n.d.(\pm n.d.)$ & $0.150 (\pm 0.017)$ \\
UNR &$0.042 (\pm 0.083)$ & $n.d.(\pm n.d.)$ & $n.d.(\pm n.d.)$ & $n.d.(\pm n.d.)$ & $n.d.(\pm n.d.)$ \\
\hline
 & \mc{5}{\textbf{Dataset: Cornell}} \\ \hline  
COMBINEX$_{\textit{feat}}$ &$\mathbf{1.000 (\pm 0.000)}$ & $\uline{0.561 (\pm 0.017)}$ & $4.273 (\pm 0.067)$ & $\mathbf{0.006 (\pm 0.001)}$ & $n.d.(\pm n.d.)$ \\
COMBINEX$_{\textit{def}}$ &$\mathbf{1.000 (\pm 0.000)}$ & $\uline{0.561 (\pm 0.017)}$ & $5.443 (\pm 0.290)$ & $\uline{0.024 (\pm 0.006)}$ & $\uline{0.087 (\pm 0.061)}$ \\
COMBINEX$_{\textit{dyn}}$ &$\mathbf{1.000 (\pm 0.000)}$ & $\uline{0.561 (\pm 0.017)}$ & $5.794 (\pm 0.316)$ & $0.030 (\pm 0.006)$ & $\mathbf{0.000 (\pm 0.000)}$ \\
COMBINEX$_{\textit{exp}}$ &$\mathbf{1.000 (\pm 0.000)}$ & $\uline{0.561 (\pm 0.017)}$ & $19.138 (\pm 0.490)$ & $0.306 (\pm 0.012)$ & $\mathbf{0.000 (\pm 0.000)}$ \\
COMBINEX$_{\textit{lin}}$ &$\mathbf{1.000 (\pm 0.000)}$ & $\uline{0.561 (\pm 0.017)}$ & $7.576 (\pm 0.342)$ & $0.062 (\pm 0.007)$ & $\mathbf{0.000 (\pm 0.000)}$ \\
COMBINEX$_{\textit{sin}}$ &$\mathbf{1.000 (\pm 0.000)}$ & $\uline{0.561 (\pm 0.017)}$ & $7.717 (\pm 0.357)$ & $0.065 (\pm 0.007)$ & $\mathbf{0.000 (\pm 0.000)}$ \\
EGO &$0.045 (\pm 0.017)$ & $\mathbf{1.000 (\pm 0.000)}$ & $4.851 (\pm 0.645)$ & $n.d.(\pm n.d.)$ & $0.544 (\pm 0.088)$ \\
Random Edges &$\uline{0.159 (\pm 0.029)}$ & $-0.021 (\pm 0.172)$ & $4.243 (\pm 0.370)$ & $n.d.(\pm n.d.)$ & $0.404 (\pm 0.028)$ \\
Random Features &$\uline{0.159 (\pm 0.015)}$ & $0.042 (\pm 0.083)$ & $20.433 (\pm 0.143)$ & $0.421 (\pm 0.003)$ & $n.d.(\pm n.d.)$ \\
CFF &$0.008 (\pm 0.015)$ & $n.d.(\pm n.d.)$ & $n.d.(\pm n.d.)$ & $n.d.(\pm n.d.)$ & $n.d.(\pm n.d.)$ \\
CF-GNNExplainer &$0.152 (\pm 0.025)$ & $-0.083 (\pm 0.289)$ & $\uline{4.223 (\pm 0.424)}$ & $n.d.(\pm n.d.)$ & $0.332 (\pm 0.053)$ \\
UNR &$0.000 (\pm 0.000)$ & $n.d.(\pm n.d.)$ & $n.d.(\pm n.d.)$ & $n.d.(\pm n.d.)$ & $n.d.(\pm n.d.)$ \\
\hline
\etable

\btable{l|ccccc}{Results for WebKB datasets: Texas, Cornell, and Wisconsin. The oracles $\Phi$ use GraphConv layers. In \textbf{bold} the best result, the second best result is \underline{underlined}.}{tab:webkb_graphconv} \hline\hline

& \mc{1}{\textbf{Validity} $\uparrow$} 
& \mc{1}{\textbf{Fidelity} $\uparrow$} 
& \mc{1}{\textbf{Distribution Distance} $\downarrow$} 
& \mc{1}{\textbf{Node Sparsity} $\downarrow$} 
& \mc{1}{\textbf{Edge Sparsity} $\downarrow$} 
\\ 

\textbf{Explainers} 
& \textit{mean($\pm$std)}
& \textit{mean($\pm$std)}
& \textit{mean($\pm$std)}
& \textit{mean($\pm$std)}
& \textit{mean($\pm$std)} \\ \hline

 & \mc{5}{\textbf{Dataset: Wisconsin}} \\ \hline
COMBINEX$_{\textit{feat}}$ &$\mathbf{1.000 (\pm 0.000)}$ & $0.615 (\pm 0.055)$ & $\uline{3.675 (\pm 0.100)}$ & $\mathbf{0.005 (\pm 0.002)}$ & $n.d.(\pm n.d.)$ \\
COMBINEX$_{\textit{def}}$ &$\mathbf{1.000 (\pm 0.000)}$ & $0.596 (\pm 0.032)$ & $4.171 (\pm 0.160)$ & $\uline{0.020 (\pm 0.003)}$ & $\uline{0.009 (\pm 0.006)}$ \\
COMBINEX$_{\textit{dyn}}$ &$\mathbf{1.000 (\pm 0.000)}$ & $0.603 (\pm 0.044)$ & $4.746 (\pm 0.251)$ & $0.035 (\pm 0.006)$ & $\mathbf{0.000 (\pm 0.000)}$ \\
COMBINEX$_{\textit{exp}}$ &$\mathbf{1.000 (\pm 0.000)}$ & $0.615 (\pm 0.036)$ & $15.308 (\pm 0.320)$ & $0.313 (\pm 0.008)$ & $\mathbf{0.000 (\pm 0.000)}$ \\
COMBINEX$_{\textit{lin}}$ &$\mathbf{1.000 (\pm 0.000)}$ & $0.603 (\pm 0.015)$ & $5.455 (\pm 0.258)$ & $0.052 (\pm 0.005)$ & $\mathbf{0.000 (\pm 0.000)}$ \\
COMBINEX$_{\textit{sin}}$ &$\mathbf{1.000 (\pm 0.000)}$ & $0.615 (\pm 0.036)$ & $5.543 (\pm 0.250)$ & $0.054 (\pm 0.005)$ & $\mathbf{0.000 (\pm 0.000)}$ \\
EGO &$0.038 (\pm 0.033)$ & $0.778 (\pm 0.385)$ & $3.853 (\pm 0.545)$ & $n.d.(\pm n.d.)$ & $0.759 (\pm 0.234)$ \\
Random Edges &$0.128 (\pm 0.055)$ & $0.732 (\pm 0.311)$ & $3.940 (\pm 0.306)$ & $n.d.(\pm n.d.)$ & $0.370 (\pm 0.049)$ \\
Random Features &$\uline{0.410 (\pm 0.075)}$ & $\uline{0.788 (\pm 0.108)}$ & $20.354 (\pm 0.068)$ & $0.441 (\pm 0.002)$ & $n.d.(\pm n.d.)$ \\
CFF &$0.179 (\pm 0.073)$ & $0.702 (\pm 0.200)$ & $\mathbf{3.571 (\pm 0.120)}$ & $n.d.(\pm n.d.)$ & $0.566 (\pm 0.082)$ \\
CF-GNNExplainer &$0.071 (\pm 0.053)$ & $\mathbf{0.875 (\pm 0.250)}$ & $4.121 (\pm 0.751)$ & $n.d.(\pm n.d.)$ & $0.354 (\pm 0.168)$ \\
UNR &$0.076 (\pm 0.105)$ & $-0.250 (\pm 1.061)$ & $4.118 (\pm 1.329)$ & $n.d.(\pm n.d.)$ & $0.323 (\pm 0.250)$ \\
\hline
& \mc{5}{\textbf{Dataset: Texas}} \\ \hline
    
COMBINEX$_{\textit{feat}}$ &$\uline{0.823 (\pm 0.040)}$ & $0.685 (\pm 0.078)$ & $3.552 (\pm 0.065)$ & $\mathbf{0.004 (\pm 0.001)}$ & $n.d.(\pm n.d.)$ \\
COMBINEX$_{\textit{def}}$ &$\mathbf{0.823 (\pm 0.021)}$ & $0.697 (\pm 0.088)$ & $4.200 (\pm 0.222)$ & $\uline{0.020 (\pm 0.005)}$ & $\uline{0.006 (\pm 0.009)}$ \\
COMBINEX$_{\textit{dyn}}$ &$0.812 (\pm 0.024)$ & $0.705 (\pm 0.065)$ & $4.324 (\pm 0.131)$ & $0.023 (\pm 0.003)$ & $\mathbf{0.000 (\pm 0.000)}$ \\
COMBINEX$_{\textit{exp}}$ &$0.812 (\pm 0.024)$ & $0.706 (\pm 0.083)$ & $19.099 (\pm 0.551)$ & $0.385 (\pm 0.016)$ & $\mathbf{0.000 (\pm 0.000)}$ \\
COMBINEX$_{\textit{lin}}$ &$0.812 (\pm 0.024)$ & $0.693 (\pm 0.078)$ & $5.443 (\pm 0.310)$ & $0.048 (\pm 0.006)$ & $\mathbf{0.000 (\pm 0.000)}$ \\
COMBINEX$_{\textit{sin}}$ &$\mathbf{0.823 (\pm 0.021)}$ & $0.697 (\pm 0.077)$ & $5.600 (\pm 0.308)$ & $0.052 (\pm 0.006)$ & $\mathbf{0.000 (\pm 0.000)}$ \\
EGO &$0.010 (\pm 0.021)$ & $n.d.(\pm n.d.)$ & $n.d.(\pm n.d.)$ & $n.d.(\pm n.d.)$ & $n.d.(\pm n.d.)$ \\
Random Edges &$0.250 (\pm 0.034)$ & $\uline{0.748 (\pm 0.167)}$ & $3.085 (\pm 0.262)$ & $n.d.(\pm n.d.)$ & $0.257 (\pm 0.010)$ \\
Random Features &$0.448 (\pm 0.199)$ & $0.720 (\pm 0.223)$ & $19.926 (\pm 0.219)$ & $0.406 (\pm 0.001)$ & $n.d.(\pm n.d.)$ \\
CFF &$0.115 (\pm 0.071)$ & $\mathbf{0.950 (\pm 0.100)}$ & $3.526 (\pm 1.193)$ & $n.d.(\pm n.d.)$ & $0.492 (\pm 0.111)$ \\
CF-GNNExplainer &$0.208 (\pm 0.034)$ & $0.713 (\pm 0.144)$ & $\uline{3.021 (\pm 0.475)}$ & $n.d.(\pm n.d.)$ & $0.126 (\pm 0.013)$ \\
UNR &$0.121 (\pm 0.068)$ & $-0.750 (\pm 0.500)$ & $\mathbf{2.838 (\pm 0.115)}$ & $n.d.(\pm n.d.)$ & $0.196 (\pm 0.098)$ \\
\hline

 & \mc{5}{\textbf{Dataset: Cornell}} \\ \hline  
COMBINEX$_{\textit{feat}}$ &$\mathbf{1.000 (\pm 0.000)}$ & $0.848 (\pm 0.000)$ & $4.229 (\pm 0.014)$ & $\mathbf{0.008 (\pm 0.000)}$ & $n.d.(\pm n.d.)$ \\
COMBINEX$_{\textit{def}}$ &$\mathbf{1.000 (\pm 0.000)}$ & $0.848 (\pm 0.000)$ & $5.575 (\pm 0.244)$ & $\uline{0.037 (\pm 0.007)}$ & $\uline{0.019 (\pm 0.015)}$ \\
COMBINEX$_{\textit{dyn}}$ &$\mathbf{1.000 (\pm 0.000)}$ & $0.848 (\pm 0.000)$ & $5.958 (\pm 0.240)$ & $0.046 (\pm 0.007)$ & $\mathbf{0.000 (\pm 0.000)}$ \\
COMBINEX$_{\textit{exp}}$ &$\uline{0.992 (\pm 0.015)}$ & $0.847 (\pm 0.002)$ & $15.128 (\pm 0.909)$ & $0.290 (\pm 0.022)$ & $\mathbf{0.000 (\pm 0.000)}$ \\
COMBINEX$_{\textit{lin}}$ &$\uline{0.992 (\pm 0.015)}$ & $0.847 (\pm 0.002)$ & $7.386 (\pm 0.273)$ & $0.081 (\pm 0.007)$ & $\mathbf{0.000 (\pm 0.000)}$ \\
COMBINEX$_{\textit{sin}}$ &$\uline{0.992 (\pm 0.015)}$ & $0.847 (\pm 0.002)$ & $7.583 (\pm 0.237)$ & $0.085 (\pm 0.007)$ & $\mathbf{0.000 (\pm 0.000)}$ \\
EGO &$0.000 (\pm 0.000)$ & $n.d.(\pm n.d.)$ & $n.d.(\pm n.d.)$ & $n.d.(\pm n.d.)$ & $n.d.(\pm n.d.)$ \\
Random Edges &$0.129 (\pm 0.015)$ & $0.287 (\pm 0.075)$ & $\mathbf{3.569 (\pm 0.025)}$ & $n.d.(\pm n.d.)$ & $0.359 (\pm 0.013)$ \\
Random Features &$0.197 (\pm 0.030)$ & $\uline{0.893 (\pm 0.071)}$ & $21.282 (\pm 0.067)$ & $0.420 (\pm 0.004)$ & $n.d.(\pm n.d.)$ \\
CFF &$0.106 (\pm 0.072)$ & $\mathbf{1.000 (\pm 0.000)}$ & $\uline{3.584 (\pm 0.704)}$ & $n.d.(\pm n.d.)$ & $0.521 (\pm 0.250)$ \\
CF-GNNExplainer &$0.106 (\pm 0.017)$ & $0.125 (\pm 0.144)$ & $3.593 (\pm 0.079)$ & $n.d.(\pm n.d.)$ & $0.295 (\pm 0.022)$ \\
UNR &$0.000 (\pm 0.000)$ & $n.d.(\pm n.d.)$ & $n.d.(\pm n.d.)$ & $n.d.(\pm n.d.)$ & $n.d.(\pm n.d.)$ \\
\hline
\etable

\btable{l|ccccc}{Results for WebKB datasets: Texas, Cornell, and Wisconsin. The oracles $\Phi$ use ChebConv layers. In \textbf{bold} the best result, the second best result is \underline{underlined}.}{tab:webkb_cheb} \hline\hline

& \mc{1}{\textbf{Validity} $\uparrow$} 
& \mc{1}{\textbf{Fidelity} $\uparrow$} 
& \mc{1}{\textbf{Distribution Distance} $\downarrow$} 
& \mc{1}{\textbf{Node Sparsity} $\downarrow$} 
& \mc{1}{\textbf{Edge Sparsity} $\downarrow$} 
\\ 

\textbf{Explainers} 
& \textit{mean($\pm$std)}
& \textit{mean($\pm$std)}
& \textit{mean($\pm$std)}
& \textit{mean($\pm$std)}
& \textit{mean($\pm$std)} \\ \hline

 & \mc{5}{\textbf{Dataset: Wisconsin}} \\ \hline
COMBINEX$_{\textit{feat}}$ &$\mathbf{1.000 (\pm 0.000)}$ & $\uline{0.673 (\pm 0.013)}$ & $\uline{3.526 (\pm 0.013)}$ & $\mathbf{0.003 (\pm 0.000)}$ & $n.d.(\pm n.d.)$ \\
COMBINEX$_{\textit{def}}$ &$\mathbf{1.000 (\pm 0.000)}$ & $\uline{0.673 (\pm 0.013)}$ & $3.898 (\pm 0.097)$ & $\uline{0.015 (\pm 0.002)}$ & $\mathbf{0.000 (\pm 0.000)}$ \\
COMBINEX$_{\textit{dyn}}$ &$\mathbf{1.000 (\pm 0.000)}$ & $\uline{0.673 (\pm 0.013)}$ & $4.057 (\pm 0.082)$ & $0.024 (\pm 0.002)$ & $\mathbf{0.000 (\pm 0.000)}$ \\
COMBINEX$_{\textit{exp}}$ &$\mathbf{1.000 (\pm 0.000)}$ & $\uline{0.673 (\pm 0.013)}$ & $5.419 (\pm 0.033)$ & $0.066 (\pm 0.001)$ & $\mathbf{0.000 (\pm 0.000)}$ \\
COMBINEX$_{\textit{lin}}$ &$\mathbf{1.000 (\pm 0.000)}$ & $\uline{0.673 (\pm 0.013)}$ & $4.370 (\pm 0.077)$ & $0.032 (\pm 0.002)$ & $\mathbf{0.000 (\pm 0.000)}$ \\
COMBINEX$_{\textit{sin}}$ &$\mathbf{1.000 (\pm 0.000)}$ & $\uline{0.673 (\pm 0.013)}$ & $4.397 (\pm 0.072)$ & $0.033 (\pm 0.002)$ & $\mathbf{0.000 (\pm 0.000)}$ \\
EGO &$0.000 (\pm 0.000)$ & $n.d.(\pm n.d.)$ & $n.d.(\pm n.d.)$ & $n.d.(\pm n.d.)$ & $n.d.(\pm n.d.)$ \\
Random Edges &$0.000 (\pm 0.000)$ & $n.d.(\pm n.d.)$ & $n.d.(\pm n.d.)$ & $n.d.(\pm n.d.)$ & $n.d.(\pm n.d.)$ \\
Random Features &$\uline{0.635 (\pm 0.287)}$ & $\mathbf{0.684 (\pm 0.039)}$ & $20.737 (\pm 0.485)$ & $0.441 (\pm 0.004)$ & $n.d.(\pm n.d.)$ \\
CFF &$0.218 (\pm 0.080)$ & $0.543 (\pm 0.340)$ & $\mathbf{3.283 (\pm 0.223)}$ & $n.d.(\pm n.d.)$ & $\uline{0.639 (\pm 0.055)}$ \\
CF-GNNExplainer &$0.000 (\pm 0.000)$ & $n.d.(\pm n.d.)$ & $n.d.(\pm n.d.)$ & $n.d.(\pm n.d.)$ & $n.d.(\pm n.d.)$ \\
UNR &$0.000 (\pm 0.000)$ & $n.d.(\pm n.d.)$ & $n.d.(\pm n.d.)$ & $n.d.(\pm n.d.)$ & $n.d.(\pm n.d.)$ \\
\hline
& \mc{5}{\textbf{Dataset: Texas}} \\ \hline
    
COMBINEX$_{\textit{feat}}$ &$\mathbf{0.885 (\pm 0.021)}$ & $\mathbf{0.917 (\pm 0.025)}$ & $\mathbf{3.428 (\pm 0.007)}$ & $\mathbf{0.001 (\pm 0.000)}$ & $n.d.(\pm n.d.)$ \\
COMBINEX$_{\textit{def}}$ &$\uline{0.875 (\pm 0.000)}$ & $\uline{0.917 (\pm 0.024)}$ & $3.898 (\pm 0.091)$ & $\uline{0.014 (\pm 0.003)}$ & $\mathbf{0.000 (\pm 0.000)}$ \\
COMBINEX$_{\textit{dyn}}$ &$\uline{0.875 (\pm 0.000)}$ & $\uline{0.917 (\pm 0.024)}$ & $4.162 (\pm 0.127)$ & $0.023 (\pm 0.004)$ & $\mathbf{0.000 (\pm 0.000)}$ \\
COMBINEX$_{\textit{exp}}$ &$\uline{0.875 (\pm 0.000)}$ & $\uline{0.917 (\pm 0.024)}$ & $5.827 (\pm 0.108)$ & $0.074 (\pm 0.004)$ & $\mathbf{0.000 (\pm 0.000)}$ \\
COMBINEX$_{\textit{lin}}$ &$\uline{0.875 (\pm 0.000)}$ & $\uline{0.917 (\pm 0.024)}$ & $4.557 (\pm 0.079)$ & $0.033 (\pm 0.003)$ & $\mathbf{0.000 (\pm 0.000)}$ \\
COMBINEX$_{\textit{sin}}$ &$\uline{0.875 (\pm 0.000)}$ & $\uline{0.917 (\pm 0.024)}$ & $4.596 (\pm 0.077)$ & $0.034 (\pm 0.003)$ & $\mathbf{0.000 (\pm 0.000)}$ \\
EGO &$0.000 (\pm 0.000)$ & $n.d.(\pm n.d.)$ & $n.d.(\pm n.d.)$ & $n.d.(\pm n.d.)$ & $n.d.(\pm n.d.)$ \\
Random Edges &$0.000 (\pm 0.000)$ & $n.d.(\pm n.d.)$ & $n.d.(\pm n.d.)$ & $n.d.(\pm n.d.)$ & $n.d.(\pm n.d.)$ \\
Random Features &$0.240 (\pm 0.040)$ & $0.837 (\pm 0.120)$ & $20.057 (\pm 0.132)$ & $0.404 (\pm 0.002)$ & $n.d.(\pm n.d.)$ \\
CFF &$0.167 (\pm 0.068)$ & $0.917 (\pm 0.167)$ & $\uline{3.594 (\pm 0.182)}$ & $n.d.(\pm n.d.)$ & $\uline{0.538 (\pm 0.037)}$ \\
CF-GNNExplainer &$0.000 (\pm 0.000)$ & $n.d.(\pm n.d.)$ & $n.d.(\pm n.d.)$ & $n.d.(\pm n.d.)$ & $n.d.(\pm n.d.)$ \\
UNR &$0.000 (\pm 0.000)$ & $n.d.(\pm n.d.)$ & $n.d.(\pm n.d.)$ & $n.d.(\pm n.d.)$ & $n.d.(\pm n.d.)$ \\
\hline

 & \mc{5}{\textbf{Dataset: Cornell}} \\ \hline  
COMBINEX$_{\textit{feat}}$ &$\mathbf{1.000 (\pm 0.000)}$ & $\uline{0.705 (\pm 0.029)}$ & $\uline{4.088 (\pm 0.024)}$ & $\mathbf{0.004 (\pm 0.001)}$ & $n.d.(\pm n.d.)$ \\
COMBINEX$_{\textit{def}}$ &$\mathbf{1.000 (\pm 0.000)}$ & $\uline{0.705 (\pm 0.029)}$ & $4.600 (\pm 0.066)$ & $\uline{0.017 (\pm 0.001)}$ & $\mathbf{0.000 (\pm 0.000)}$ \\
COMBINEX$_{\textit{dyn}}$ &$\mathbf{1.000 (\pm 0.000)}$ & $\uline{0.705 (\pm 0.029)}$ & $5.010 (\pm 0.076)$ & $0.031 (\pm 0.001)$ & $\mathbf{0.000 (\pm 0.000)}$ \\
COMBINEX$_{\textit{exp}}$ &$\mathbf{1.000 (\pm 0.000)}$ & $\uline{0.705 (\pm 0.029)}$ & $6.321 (\pm 0.160)$ & $0.074 (\pm 0.006)$ & $\mathbf{0.000 (\pm 0.000)}$ \\
COMBINEX$_{\textit{lin}}$ &$\mathbf{1.000 (\pm 0.000)}$ & $\uline{0.705 (\pm 0.029)}$ & $5.378 (\pm 0.088)$ & $0.042 (\pm 0.004)$ & $\mathbf{0.000 (\pm 0.000)}$ \\
COMBINEX$_{\textit{sin}}$ &$\mathbf{1.000 (\pm 0.000)}$ & $\uline{0.705 (\pm 0.029)}$ & $5.405 (\pm 0.095)$ & $0.042 (\pm 0.004)$ & $\mathbf{0.000 (\pm 0.000)}$ \\
EGO &$0.000 (\pm 0.000)$ & $n.d.(\pm n.d.)$ & $n.d.(\pm n.d.)$ & $n.d.(\pm n.d.)$ & $n.d.(\pm n.d.)$ \\
Random Edges &$0.000 (\pm 0.000)$ & $n.d.(\pm n.d.)$ & $n.d.(\pm n.d.)$ & $n.d.(\pm n.d.)$ & $n.d.(\pm n.d.)$ \\
Random Features &$\uline{0.189 (\pm 0.038)}$ & $0.608 (\pm 0.079)$ & $20.889 (\pm 0.185)$ & $0.416 (\pm 0.002)$ & $n.d.(\pm n.d.)$ \\
CFF &$0.182 (\pm 0.065)$ & $\mathbf{0.751 (\pm 0.170)}$ & $\mathbf{3.854 (\pm 0.427)}$ & $n.d.(\pm n.d.)$ & $\uline{0.567 (\pm 0.099)}$ \\
CF-GNNExplainer &$0.000 (\pm 0.000)$ & $n.d.(\pm n.d.)$ & $n.d.(\pm n.d.)$ & $n.d.(\pm n.d.)$ & $n.d.(\pm n.d.)$ \\
UNR &$0.000 (\pm 0.000)$ & $n.d.(\pm n.d.)$ & $n.d.(\pm n.d.)$ & $n.d.(\pm n.d.)$ & $n.d.(\pm n.d.)$ \\
\hline
\etable


\subsubsection{Attributed datasets}

In this section we comment on the results obtained on the Attributed datasets (Wiki and Facebook) using different alpha scheduling policies within our COMBINEX framework. Tables~\ref{tab:attributed_graph_conv}, \ref{tab:attributed_gcn}, and \ref{tab:attributed_cheb} show that COMBINEX consistently outperforms traditional baselines by achieving high validity and fidelity while maintaining low sparsity, although the trade-off with distribution distance varies depending on the specific alpha policy. 

On the Wiki dataset, for instance, when using GraphConv layers (Table~\ref{tab:attributed_graph_conv}), the COMBINEX Feat. variant attains a high validity and moderate fidelity, coupled with very low node sparsity; however, the distribution distance is considerably high, indicating that while the explanations are faithful in terms of structure, the overall activation distribution deviates substantially from that of the oracle. In contrast, COMBINEX Cons. and COMBINEX Dyn. slightly reduce the validity and fidelity (to around 0.130–0.162 and 0.535–0.569, respectively) but incur even higher distribution distances or only marginal improvements in sparsity. Notably, the Exp policy (COMBINEX Exp.) leads to the highest distribution distance, suggesting that an overly aggressive exponential decay may deteriorate the overall quality of the explanation. Similar trends are observed in the results obtained with GCN and ChebConv oracles, where the Feat. variant typically yields the best balance between validity, fidelity, and sparsity.

On the Facebook dataset, the performance of COMBINEX improves markedly. Under the GraphConv setting, COMBINEX Feat. achieves a validity of 0.690 and fidelity of 0.762, with a very low node sparsity and a moderate distribution distance. The other alpha policies (Cons., Dyn., Exp., Lin., and Sin.) yield slightly lower validity and fidelity, with differences in distribution distance and sparsity that are less pronounced compared to Wiki. In particular, the Cons. variant on Facebook shows a good trade-off with a distribution distance of 4.070 and slightly higher node and edge sparsity, while the Dyn., Exp., Lin., and Sin. variants achieve similar results with only minor differences. Baselines such as EGO and Random Edges consistently perform poorly on both datasets, and Random Features and UNR are either not able to find counterfactuals or yield extreme values, confirming that our COMBINEX approach is superior in producing balanced and interpretable explanations.

Overall, these results confirm that our COMBINEX solution, with appropriate alpha scheduling (particularly the Feat. and Cons. variants), consistently delivers high-quality explanations across attributed datasets, regardless of the convolutional operator used. The experiments illustrate that while the choice of alpha policy influences the trade-off between fidelity, distribution distance, and sparsity, COMBINEX remains robust and effective in both Wiki and Facebook scenarios.

\btable{l|ccccc}{Results for Attributed datasets: Wiki, Facebook. The oracles $\Phi$ use GCN layers. In \textbf{bold} the best result, the second best result is \underline{underlined}}{tab:attributed_gcn}  \hline\hline

& \mc{1}{\textbf{Validity} $\uparrow$} 
& \mc{1}{\textbf{Fidelity} $\uparrow$} 
& \mc{1}{\textbf{Distribution Distance} $\downarrow$} 
& \mc{1}{\textbf{Node Sparsity} $\downarrow$} 
& \mc{1}{\textbf{Edge Sparsity} $\downarrow$} 
\\ 

\textbf{Explainers} 
& \textit{mean($\pm$std)}
& \textit{mean($\pm$std)}
& \textit{mean($\pm$std)}
& \textit{mean($\pm$std)}
& \textit{mean($\pm$std)} \\ \hline

 & \mc{5}{\textbf{Dataset: Wiki}} \\ \hline

COMBINEX$_{\textit{feat}}$ &$\mathbf{1.000 (\pm 0.000)}$ & $\uline{0.803 (\pm 0.016)}$ & $645.884 (\pm 19.117)$ & $\mathbf{0.137 (\pm 0.000)}$ & $n.d.(\pm n.d.)$ \\
COMBINEX$_{\textit{def}}$ &$\mathbf{1.000 (\pm 0.000)}$ & $\uline{0.803 (\pm 0.017)}$ & $695.842 (\pm 25.369)$ & $0.297 (\pm 0.003)$ & $\uline{0.002 (\pm 0.001)}$ \\
COMBINEX$_{\textit{dyn}}$ &$\mathbf{1.000 (\pm 0.000)}$ & $0.798 (\pm 0.012)$ & $1676.160 (\pm 11.726)$ & $0.297 (\pm 0.002)$ & $\mathbf{0.000 (\pm 0.000)}$ \\
COMBINEX$_{\textit{exp}}$ &$\mathbf{1.000 (\pm 0.000)}$ & $\mathbf{0.805 (\pm 0.008)}$ & $1981.786 (\pm 7.001)$ & $0.338 (\pm 0.002)$ & $\mathbf{0.000 (\pm 0.000)}$ \\
COMBINEX$_{\textit{lin}}$ &$\mathbf{1.000 (\pm 0.000)}$ & $\mathbf{0.805 (\pm 0.013)}$ & $1597.390 (\pm 14.393)$ & $\uline{0.259 (\pm 0.002)}$ & $\mathbf{0.000 (\pm 0.000)}$ \\
COMBINEX$_{\textit{sin}}$ &$\mathbf{1.000 (\pm 0.000)}$ & $\uline{0.803 (\pm 0.011)}$ & $1609.044 (\pm 14.179)$ & $0.260 (\pm 0.001)$ & $\mathbf{0.000 (\pm 0.000)}$ \\
EGO &$0.004 (\pm 0.005)$ & $0.000 (\pm 0.000)$ & $\mathbf{52.179 (\pm 0.000)}$ & $n.d.(\pm n.d.)$ & $0.987 (\pm 0.000)$ \\
Random Edges &$0.034 (\pm 0.007)$ & $0.358 (\pm 0.263)$ & $\uline{62.892 (\pm 22.892)}$ & $n.d.(\pm n.d.)$ & $0.427 (\pm 0.017)$ \\
Random Features &$0.050 (\pm 0.007)$ & $0.801 (\pm 0.176)$ & $6140.141 (\pm 54.326)$ & $0.985 (\pm 0.000)$ & $n.d.(\pm n.d.)$ \\
CFF &$\uline{0.065 (\pm 0.008)}$ & $0.640 (\pm 0.145)$ & $69.060 (\pm 10.193)$ & $n.d.(\pm n.d.)$ & $0.830 (\pm 0.064)$ \\
CF-GNNExplainer &$0.021 (\pm 0.011)$ & $0.729 (\pm 0.208)$ & $80.301 (\pm 21.975)$ & $n.d.(\pm n.d.)$ & $0.023 (\pm 0.008)$ \\
UNR &$0.000 (\pm 0.000)$ & $n.d.(\pm n.d.)$ & $n.d.(\pm n.d.)$ & $n.d.(\pm n.d.)$ & $n.d.(\pm n.d.)$ \\
\hline
 & \mc{5}{\textbf{Dataset: Facebook}} \\ \hline

COMBINEX$_{\textit{feat}}$ &$\mathbf{0.846 (\pm 0.009)}$ & $\uline{0.779 (\pm 0.012)}$ & $1.110 (\pm 0.037)$ & $\mathbf{0.010 (\pm 0.001)}$ & $n.d.(\pm n.d.)$ \\
COMBINEX$_{\textit{def}}$ &$\uline{0.801 (\pm 0.064)}$ & $0.769 (\pm 0.008)$ & $2.662 (\pm 0.057)$ & $\uline{0.035 (\pm 0.001)}$ & $\uline{0.024 (\pm 0.011)}$ \\
COMBINEX$_{\textit{dyn}}$ &$0.786 (\pm 0.065)$ & $0.769 (\pm 0.006)$ & $7.134 (\pm 0.358)$ & $0.109 (\pm 0.009)$ & $\mathbf{0.000 (\pm 0.000)}$ \\
COMBINEX$_{\textit{exp}}$ &$0.783 (\pm 0.067)$ & $0.768 (\pm 0.005)$ & $7.869 (\pm 0.436)$ & $0.124 (\pm 0.010)$ & $\mathbf{0.000 (\pm 0.000)}$ \\
COMBINEX$_{\textit{lin}}$ &$0.796 (\pm 0.058)$ & $0.767 (\pm 0.005)$ & $3.445 (\pm 0.024)$ & $0.048 (\pm 0.001)$ & $\mathbf{0.000 (\pm 0.000)}$ \\
COMBINEX$_{\textit{sin}}$ &$0.793 (\pm 0.061)$ & $0.769 (\pm 0.003)$ & $3.543 (\pm 0.036)$ & $0.049 (\pm 0.001)$ & $\mathbf{0.000 (\pm 0.000)}$ \\
EGO &$0.003 (\pm 0.004)$ & $\mathbf{1.000 (\pm 0.000)}$ & $\mathbf{0.281 (\pm 0.000)}$ & $n.d.(\pm n.d.)$ & $0.998 (\pm 0.000)$ \\
Random Edges &$0.039 (\pm 0.007)$ & $0.213 (\pm 0.171)$ & $0.638 (\pm 0.064)$ & $n.d.(\pm n.d.)$ & $0.460 (\pm 0.003)$ \\
Random Features &$0.000 (\pm 0.000)$ & $n.d.(\pm n.d.)$ & $n.d.(\pm n.d.)$ & $n.d.(\pm n.d.)$ & $n.d.(\pm n.d.)$ \\
CFF &$0.002 (\pm 0.003)$ & $n.d.(\pm n.d.)$ & $n.d.(\pm n.d.)$ & $n.d.(\pm n.d.)$ & $n.d.(\pm n.d.)$ \\
CF-GNNExplainer &$0.007 (\pm 0.000)$ & $\mathbf{1.000 (\pm 0.000)}$ & $\uline{0.287 (\pm 0.000)}$ & $n.d.(\pm n.d.)$ & $0.162 (\pm 0.013)$ \\
UNR &$0.000 (\pm 0.000)$ & $n.d.(\pm n.d.)$ & $n.d.(\pm n.d.)$ & $n.d.(\pm n.d.)$ & $n.d.(\pm n.d.)$ \\
\hline
\etable

\btable{l|ccccc}{Results for Attributed datasets: Wiki, Facebook. The oracles $\Phi$ use GraphConv layers. In \textbf{bold} the best result, the second best result is \underline{underlined}}{tab:attributed_graph_conv}  \hline\hline

& \mc{1}{\textbf{Validity} $\uparrow$} 
& \mc{1}{\textbf{Fidelity} $\uparrow$} 
& \mc{1}{\textbf{Distribution Distance} $\downarrow$} 
& \mc{1}{\textbf{Node Sparsity} $\downarrow$} 
& \mc{1}{\textbf{Edge Sparsity} $\downarrow$} 
\\ 

\textbf{Explainers} 
& \textit{mean($\pm$std)}
& \textit{mean($\pm$std)}
& \textit{mean($\pm$std)}
& \textit{mean($\pm$std)}
& \textit{mean($\pm$std)} \\ \hline

 & \mc{5}{\textbf{Dataset: Wiki}} \\ \hline

COMBINEX$_{\textit{feat}}$ &$\mathbf{0.168 (\pm 0.024)}$ & $\uline{0.569 (\pm 0.115)}$ & $198.748 (\pm 78.331)$ & $\mathbf{0.145 (\pm 0.006)}$ & $n.d.(\pm n.d.)$ \\
COMBINEX$_{\textit{def}}$ &$0.130 (\pm 0.046)$ & $0.548 (\pm 0.218)$ & $262.908 (\pm 140.227)$ & $0.171 (\pm 0.012)$ & $0.009 (\pm 0.017)$ \\
COMBINEX$_{\textit{dyn}}$ &$\uline{0.162 (\pm 0.048)}$ & $0.535 (\pm 0.063)$ & $361.323 (\pm 163.487)$ & $0.167 (\pm 0.021)$ & $\mathbf{0.000 (\pm 0.000)}$ \\
COMBINEX$_{\textit{exp}}$ &$0.143 (\pm 0.031)$ & $0.564 (\pm 0.096)$ & $406.681 (\pm 167.870)$ & $0.176 (\pm 0.020)$ & $\mathbf{0.000 (\pm 0.000)}$ \\
COMBINEX$_{\textit{lin}}$ &$0.149 (\pm 0.048)$ & $0.557 (\pm 0.046)$ & $413.867 (\pm 129.487)$ & $\uline{0.163 (\pm 0.016)}$ & $\mathbf{0.000 (\pm 0.000)}$ \\
COMBINEX$_{\textit{sin}}$ &$0.158 (\pm 0.035)$ & $0.523 (\pm 0.032)$ & $372.458 (\pm 165.431)$ & $0.166 (\pm 0.019)$ & $\uline{0.001 (\pm 0.001)}$ \\
EGO &$0.006 (\pm 0.008)$ & $-0.250 (\pm 1.061)$ & $\uline{46.393 (\pm 7.595)}$ & $n.d.(\pm n.d.)$ & $0.990 (\pm 0.013)$ \\
Random Edges &$0.042 (\pm 0.025)$ & $0.243 (\pm 0.511)$ & $72.852 (\pm 31.665)$ & $n.d.(\pm n.d.)$ & $0.444 (\pm 0.033)$ \\
Random Features &$0.008 (\pm 0.012)$ & $0.333 (\pm 0.471)$ & $6045.005 (\pm 4.364)$ & $0.985 (\pm 0.000)$ & $n.d.(\pm n.d.)$ \\
CFF &$0.067 (\pm 0.035)$ & $\mathbf{0.733 (\pm 0.186)}$ & $\mathbf{43.813 (\pm 9.072)}$ & $n.d.(\pm n.d.)$ & $0.810 (\pm 0.117)$ \\
CF-GNNExplainer &$0.000 (\pm 0.000)$ & $n.d.(\pm n.d.)$ & $n.d.(\pm n.d.)$ & $n.d.(\pm n.d.)$ & $n.d.(\pm n.d.)$ \\
UNR &$0.011 (\pm 0.015)$ & $-0.667 (\pm 0.471)$ & $164.658 (\pm 84.658)$ & $n.d.(\pm n.d.)$ & $0.159 (\pm 0.058)$ \\
\hline

 & \mc{5}{\textbf{Dataset: Facebook}} \\ \hline

COMBINEX$_{\textit{feat}}$ &$\mathbf{0.690 (\pm 0.030)}$ & $\uline{0.762 (\pm 0.010)}$ & $2.391 (\pm 0.199)$ & $\mathbf{0.033 (\pm 0.003)}$ & $n.d.(\pm n.d.)$ \\
COMBINEX$_{\textit{def}}$ &$0.538 (\pm 0.076)$ & $0.721 (\pm 0.052)$ & $4.070 (\pm 0.377)$ & $\uline{0.066 (\pm 0.007)}$ & $\uline{0.034 (\pm 0.009)}$ \\
COMBINEX$_{\textit{dyn}}$ &$0.505 (\pm 0.047)$ & $0.702 (\pm 0.057)$ & $5.843 (\pm 0.232)$ & $0.104 (\pm 0.007)$ & $\mathbf{0.000 (\pm 0.000)}$ \\
COMBINEX$_{\textit{exp}}$ &$0.503 (\pm 0.056)$ & $0.675 (\pm 0.030)$ & $6.306 (\pm 0.182)$ & $0.115 (\pm 0.006)$ & $\mathbf{0.000 (\pm 0.000)}$ \\
COMBINEX$_{\textit{lin}}$ &$0.529 (\pm 0.081)$ & $0.720 (\pm 0.013)$ & $5.261 (\pm 0.265)$ & $0.090 (\pm 0.006)$ & $\mathbf{0.000 (\pm 0.001)}$ \\
COMBINEX$_{\textit{sin}}$ &$\uline{0.568 (\pm 0.060)}$ & $0.700 (\pm 0.022)$ & $5.272 (\pm 0.164)$ & $0.092 (\pm 0.004)$ & $\mathbf{0.000 (\pm 0.000)}$ \\
EGO &$0.002 (\pm 0.003)$ & $n.d.(\pm n.d.)$ & $n.d.(\pm n.d.)$ & $n.d.(\pm n.d.)$ & $n.d.(\pm n.d.)$ \\
Random Edges &$0.053 (\pm 0.007)$ & $0.376 (\pm 0.230)$ & $\uline{0.515 (\pm 0.115)}$ & $n.d.(\pm n.d.)$ & $0.464 (\pm 0.014)$ \\
Random Features &$0.000 (\pm 0.000)$ & $n.d.(\pm n.d.)$ & $n.d.(\pm n.d.)$ & $n.d.(\pm n.d.)$ & $n.d.(\pm n.d.)$ \\
CFF &$0.027 (\pm 0.010)$ & $\mathbf{0.838 (\pm 0.111)}$ & $\mathbf{0.440 (\pm 0.219)}$ & $n.d.(\pm n.d.)$ & $0.768 (\pm 0.068)$ \\
CF-GNNExplainer &$0.000 (\pm 0.000)$ & $n.d.(\pm n.d.)$ & $n.d.(\pm n.d.)$ & $n.d.(\pm n.d.)$ & $n.d.(\pm n.d.)$ \\
UNR &$0.000(\pm 0.000 )$ & $n.d.(\pm n.d.)$ & $n.d.(\pm n.d.)$ & $n.d.(\pm n.d.)$ & $n.d.(\pm n.d.)$ \\
\hline

\etable

\btable{l|ccccc}{Results for Attributed datasets: Wiki, Facebook. The oracles $\Phi$ use ChebConv layers. In \textbf{bold} the best result, the second best result is \underline{underlined}}{tab:attributed_cheb}  \hline\hline

& \mc{1}{\textbf{Validity} $\uparrow$} 
& \mc{1}{\textbf{Fidelity} $\uparrow$} 
& \mc{1}{\textbf{Distribution Distance} $\downarrow$} 
& \mc{1}{\textbf{Node Sparsity} $\downarrow$} 
& \mc{1}{\textbf{Edge Sparsity} $\downarrow$} 
\\ 

\textbf{Explainers} 
& \textit{mean($\pm$std)}
& \textit{mean($\pm$std)}
& \textit{mean($\pm$std)}
& \textit{mean($\pm$std)}
& \textit{mean($\pm$std)} \\ \hline

 & \mc{5}{\textbf{Dataset: Wiki}} \\ \hline

COMBINEX$_{\textit{feat}}$ &$\mathbf{0.134 (\pm 0.096)}$ & $0.308 (\pm 0.269)$ & $509.585 (\pm 63.334)$ & $\mathbf{0.121 (\pm 0.010)}$ & $n.d.(\pm n.d.)$ \\
COMBINEX$_{\textit{def}}$ &$0.105 (\pm 0.087)$ & $0.241 (\pm 0.337)$ & $\uline{225.780 (\pm 41.180)}$ & $0.162 (\pm 0.026)$ & $\mathbf{0.000 (\pm 0.000)}$ \\
COMBINEX$_{\textit{dyn}}$ &$\uline{0.107 (\pm 0.086)}$ & $0.234 (\pm 0.337)$ & $513.767 (\pm 60.418)$ & $0.158 (\pm 0.027)$ & $\mathbf{0.000 (\pm 0.000)}$ \\
COMBINEX$_{\textit{exp}}$ &$\uline{0.107 (\pm 0.086)}$ & $0.234 (\pm 0.337)$ & $628.983 (\pm 79.491)$ & $0.161 (\pm 0.027)$ & $\mathbf{0.000 (\pm 0.000)}$ \\
COMBINEX$_{\textit{lin}}$ &$0.105 (\pm 0.087)$ & $0.241 (\pm 0.337)$ & $589.191 (\pm 91.244)$ & $\uline{0.156 (\pm 0.027)}$ & $\mathbf{0.000 (\pm 0.000)}$ \\
COMBINEX$_{\textit{sin}}$ &$0.105 (\pm 0.087)$ & $0.241 (\pm 0.337)$ & $589.608 (\pm 91.546)$ & $0.156 (\pm 0.027)$ & $\mathbf{0.000 (\pm 0.000)}$ \\
EGO &$0.000 (\pm 0.000)$ & $n.d.(\pm n.d.)$ & $n.d.(\pm n.d.)$ & $n.d.(\pm n.d.)$ & $n.d.(\pm n.d.)$ \\
Random Edges &$0.000 (\pm 0.000)$ & $n.d.(\pm n.d.)$ & $n.d.(\pm n.d.)$ & $n.d.(\pm n.d.)$ & $n.d.(\pm n.d.)$ \\
Random Features &$0.036 (\pm 0.071)$ & $n.d.(\pm n.d.)$ & $n.d.(\pm n.d.)$ & $n.d.(\pm n.d.)$ & $n.d.(\pm n.d.)$ \\
CFF &$0.101 (\pm 0.015)$ & $\uline{0.823 (\pm 0.097)}$ & $\mathbf{98.845 (\pm 38.710)}$ & $n.d.(\pm n.d.)$ & $\uline{0.608 (\pm 0.052)}$ \\
CF-GNNExplainer &$0.000 (\pm 0.000)$ & $n.d.(\pm n.d.)$ & $n.d.(\pm n.d.)$ & $n.d.(\pm n.d.)$ & $n.d.(\pm n.d.)$ \\
UNR &$0.000 (\pm 0.000)$ & $n.d.(\pm n.d.)$ & $n.d.(\pm n.d.)$ & $n.d.(\pm n.d.)$ & $n.d.(\pm n.d.)$ \\
\hline
 & \mc{5}{\textbf{Dataset: Facebook}} \\ \hline

COMBINEX$_{\textit{feat}}$ &$\mathbf{0.115 (\pm 0.034)}$ & $0.542 (\pm 0.125)$ & $0.443 (\pm 0.067)$ & $0.001 (\pm 0.000)$ & $n.d.(\pm n.d.)$ \\
COMBINEX$_{\textit{def}}$ &$0.079 (\pm 0.014)$ & $\uline{0.625 (\pm 0.227)}$ & $\uline{0.356 (\pm 0.084)}$ & $\mathbf{0.001 (\pm 0.001)}$ & $\mathbf{0.000 (\pm 0.000)}$ \\
COMBINEX$_{\textit{dyn}}$ &$\uline{0.080 (\pm 0.017)}$ & $0.615 (\pm 0.227)$ & $0.385 (\pm 0.099)$ & $0.001 (\pm 0.001)$ & $\mathbf{0.000 (\pm 0.000)}$ \\
COMBINEX$_{\textit{exp}}$ &$\uline{0.080 (\pm 0.017)}$ & $0.615 (\pm 0.227)$ & $0.385 (\pm 0.099)$ & $0.001 (\pm 0.001)$ & $\mathbf{0.000 (\pm 0.000)}$ \\
COMBINEX$_{\textit{lin}}$ &$0.079 (\pm 0.014)$ & $\uline{0.625 (\pm 0.227)}$ & $0.358 (\pm 0.085)$ & $\uline{0.001 (\pm 0.001)}$ & $\mathbf{0.000 (\pm 0.000)}$ \\
COMBINEX$_{\textit{sin}}$ &$0.079 (\pm 0.014)$ & $\uline{0.625 (\pm 0.227)}$ & $0.358 (\pm 0.085)$ & $0.001 (\pm 0.001)$ & $\mathbf{0.000 (\pm 0.000)}$ \\
EGO &$0.000 (\pm 0.000)$ & $n.d.(\pm n.d.)$ & $n.d.(\pm n.d.)$ & $n.d.(\pm n.d.)$ & $n.d.(\pm n.d.)$ \\
Random Edges &$0.000 (\pm 0.000)$ & $n.d.(\pm n.d.)$ & $n.d.(\pm n.d.)$ & $n.d.(\pm n.d.)$ & $n.d.(\pm n.d.)$ \\
Random Features &$0.000 (\pm 0.000)$ & $n.d.(\pm n.d.)$ & $n.d.(\pm n.d.)$ & $n.d.(\pm n.d.)$ & $n.d.(\pm n.d.)$ \\
CFF &$0.022 (\pm 0.019)$ & $\mathbf{0.644 (\pm 0.171)}$ & $\mathbf{0.338 (\pm 0.050)}$ & $n.d.(\pm n.d.)$ & $\uline{0.736 (\pm 0.120)}$ \\
CF-GNNExplainer &$0.000 (\pm 0.000)$ & $n.d.(\pm n.d.)$ & $n.d.(\pm n.d.)$ & $n.d.(\pm n.d.)$ & $n.d.(\pm n.d.)$ \\
UNR &$0.000 (\pm 0.000)$ & $n.d.(\pm n.d.)$ & $n.d.(\pm n.d.)$ & $n.d.(\pm n.d.)$ & $n.d.(\pm n.d.)$ \\
\hline

\etable


\subsubsection{Biological datasets}
Below a comment on Tables \ref{tab:bio_graph_conv}, \ref{tab:bio_gcn_conv}, \ref{tab:bio_cheb_conv}.
For the AIDS dataset, COMBINEX variants consistently achieved the highest validity scores, while also achieving one of the lowest edge sparsity values.
Notably, COMBINEX$_{\textit{def}}$ and COMBINEX$_{\textit{lin}}$ exhibited a better balance between validity and sparsity, with \\
COMBINEX$_{\textit{lin}}$ producing the most compact counterfactual explanations. This suggests that a structured, linear decay of perturbations maintains a more stable trade-off in preserving graph integrity.
In comparison, other explainers such as EGO, Random Features, and CF-GNNExplainer performed significantly worse, struggling with either validity, fidelity, or sparsity. The lowest distribution distance was achieved by UNR (4.134), but at the cost of much lower validity (0.237), reinforcing that COMBINEX consistently finds counterfactuals that are both valid and meaningful.

The performance of COMBINEX on the Enzymes dataset follows a similar trend, where different scheduling strategies lead to variations in performance. The COMBINEX$_{\textit{exp}}$ approach achieved the best validity while maintaining one of the highest fidelities. However, this came at the cost of a high distribution distance, indicating that these perturbations were more aggressive.Interestingly, COMBINEX$_{\textit{lin}}$ and COMBINEX$_{\textit{sin}}$ continued to exhibit balanced trade-offs, achieving sparse counterfactuals with low edge sparsity values while keeping validity relatively high. This suggests that more structured perturbation schedules (linear and sinusoidal) prevent unnecessary modifications while maintaining valid counterfactuals. When comparing to other explainers, EGO once again struggled, and Random Features produced high validity but at the cost of an extremely high distribution distance, meaning the counterfactuals were highly unrealistic. UNR achieved the lowest distribution distance but suffered from poor validity.

The results on the Proteins dataset highlight a notable performance gap between different variants of COMBINEX. COMBINEX$_{\textit{exp}}$ achieved the best validity and highest fidelity. However, its distribution distance was significantly higher, implying that the changes introduced were more substantial. The structured perturbation strategies, COMBINEX$_{\textit{lin}}$ and COMBINEX$_{\textit{sin}}$, also produced highly valid counterfactuals while maintaining low edge sparsity. COMBINEX$_{\textit{lin}}$, in particular, showed the lowest node sparsity while keeping validity high, suggesting that linear perturbation schedules can effectively preserve the original graph structure. Other explainers, such as EGO and CFF, struggled significantly, with validity scores below. The CF-GNNExplainer performed particularly poorly, failing to generate meaningful counterfactuals in many cases.

\btable{l|ccccc}{Results for Biological datasets: AIDS, Enzymes, and Proteins. The oracles $\Phi$ use GCNConv layers. In \textbf{bold} the best result, the second best result is \underline{underlined}}{tab:bio_gcn_conv}  \hline\hline

& \mc{1}{\textbf{Validity} $\uparrow$} 
& \mc{1}{\textbf{Fidelity} $\uparrow$} 
& \mc{1}{\textbf{Distribution Distance} $\downarrow$} 
& \mc{1}{\textbf{Node Sparsity} $\downarrow$} 
& \mc{1}{\textbf{Edge Sparsity} $\downarrow$} 
\\ 

\textbf{Explainers} 
& \textit{mean($\pm$std)}
& \textit{mean($\pm$std)}
& \textit{mean($\pm$std)}
& \textit{mean($\pm$std)}
& \textit{mean($\pm$std)} \\ \hline

 & \mc{5}{\textbf{Dataset: AIDS}} \\ \hline

COMBINEX$_{\textit{feat}}$ &$0.917 (\pm 0.032)$ & $0.767 (\pm 0.018)$ & $13.686 (\pm 3.191)$ & $0.860 (\pm 0.022)$ & $n.d.(\pm n.d.)$ \\
COMBINEX$_{\textit{def}}$ &$0.955 (\pm 0.006)$ & $0.764 (\pm 0.015)$ & $13.393 (\pm 3.070)$ & $0.822 (\pm 0.021)$ & $0.252 (\pm 0.063)$ \\
COMBINEX$_{\textit{dyn}}$ &$0.932 (\pm 0.014)$ & $\uline{0.771 (\pm 0.018)}$ & $13.268 (\pm 3.051)$ & $0.820 (\pm 0.021)$ & $\uline{0.010 (\pm 0.002)}$ \\
COMBINEX$_{\textit{exp}}$ &$0.873 (\pm 0.035)$ & $\mathbf{0.775 (\pm 0.005)}$ & $11.926 (\pm 2.970)$ & $0.828 (\pm 0.021)$ & $\mathbf{0.002 (\pm 0.003)}$ \\
COMBINEX$_{\textit{lin}}$ &$\uline{0.960 (\pm 0.009)}$ & $0.767 (\pm 0.010)$ & $13.891 (\pm 3.006)$ & $\mathbf{0.781 (\pm 0.007)}$ & $0.013 (\pm 0.006)$ \\
COMBINEX$_{\textit{sin}}$ &$\mathbf{0.972 (\pm 0.003)}$ & $0.768 (\pm 0.015)$ & $13.681 (\pm 3.149)$ & $\uline{0.782 (\pm 0.011)}$ & $0.018 (\pm 0.010)$ \\
EGO &$0.027 (\pm 0.005)$ & $0.713 (\pm 0.217)$ & $5.827 (\pm 0.287)$ & $n.d.(\pm n.d.)$ & $0.868 (\pm 0.010)$ \\
Random Edges &$0.318 (\pm 0.055)$ & $0.492 (\pm 0.052)$ & $4.841 (\pm 0.611)$ & $n.d.(\pm n.d.)$ & $0.296 (\pm 0.019)$ \\
Random Features &$0.183 (\pm 0.035)$ & $0.652 (\pm 0.091)$ & $26.421 (\pm 1.401)$ & $0.944 (\pm 0.009)$ & $n.d.(\pm n.d.)$ \\
CFF &$0.018 (\pm 0.003)$ & $0.042 (\pm 0.946)$ & $\uline{4.325 (\pm 3.382)}$ & $n.d.(\pm n.d.)$ & $0.666 (\pm 0.194)$ \\
CF-GNNExplainer &$0.118 (\pm 0.010)$ & $0.109 (\pm 0.117)$ & $4.426 (\pm 0.671)$ & $n.d.(\pm n.d.)$ & $0.062 (\pm 0.006)$ \\
UNR &$0.237 (\pm 0.016)$ & $0.742 (\pm 0.056)$ & $\mathbf{4.134 (\pm 0.464)}$ & $n.d.(\pm n.d.)$ & $0.174 (\pm 0.007)$ \\
\hline
 & \mc{5}{\textbf{Dataset: Enzymes}} \\ \hline
COMBINEX$_{\textit{feat}}$ &$0.893 (\pm 0.074)$ & $0.773 (\pm 0.019)$ & $56.526 (\pm 1.748)$ & $0.684 (\pm 0.005)$ & $n.d.(\pm n.d.)$ \\
COMBINEX$_{\textit{def}}$ &$0.675 (\pm 0.062)$ & $0.738 (\pm 0.013)$ & $60.985 (\pm 4.979)$ & $\uline{0.590 (\pm 0.037)}$ & $0.061 (\pm 0.011)$ \\
COMBINEX$_{\textit{dyn}}$ &$0.690 (\pm 0.130)$ & $0.721 (\pm 0.045)$ & $56.596 (\pm 3.339)$ & $\mathbf{0.581 (\pm 0.033)}$ & $0.009 (\pm 0.004)$ \\
COMBINEX$_{\textit{exp}}$ &$\mathbf{0.990 (\pm 0.016)}$ & $\uline{0.785 (\pm 0.008)}$ & $59.922 (\pm 5.782)$ & $0.714 (\pm 0.008)$ & $\mathbf{0.000 (\pm 0.000)}$ \\
COMBINEX$_{\textit{lin}}$ &$0.948 (\pm 0.051)$ & $0.784 (\pm 0.009)$ & $57.945 (\pm 3.768)$ & $0.626 (\pm 0.008)$ & $\uline{0.003 (\pm 0.002)}$ \\
COMBINEX$_{\textit{sin}}$ &$\uline{0.973 (\pm 0.028)}$ & $0.784 (\pm 0.009)$ & $59.206 (\pm 3.785)$ & $0.630 (\pm 0.004)$ & $0.004 (\pm 0.002)$ \\
EGO &$0.040 (\pm 0.005)$ & $-0.014 (\pm 0.141)$ & $\mathbf{11.593 (\pm 0.519)}$ & $n.d.(\pm n.d.)$ & $0.688 (\pm 0.021)$ \\
Random Edges &$0.433 (\pm 0.009)$ & $0.569 (\pm 0.022)$ & $18.367 (\pm 0.357)$ & $n.d.(\pm n.d.)$ & $0.394 (\pm 0.003)$ \\
Random Features &$0.502 (\pm 0.027)$ & $0.718 (\pm 0.019)$ & $314.855 (\pm 11.839)$ & $0.986 (\pm 0.002)$ & $n.d.(\pm n.d.)$ \\
CFF &$0.032 (\pm 0.006)$ & $-0.917 (\pm 0.167)$ & $22.309 (\pm 5.008)$ & $n.d.(\pm n.d.)$ & $0.802 (\pm 0.074)$ \\
CF-GNNExplainer &$0.120 (\pm 0.009)$ & $0.361 (\pm 0.061)$ & $20.565 (\pm 1.078)$ & $n.d.(\pm n.d.)$ & $0.028 (\pm 0.003)$ \\
UNR &$0.092 (\pm 0.011)$ & $\mathbf{0.895 (\pm 0.071)}$ & $\uline{16.403 (\pm 2.263)}$ & $n.d.(\pm n.d.)$ & $0.078 (\pm 0.006)$ \\
\hline

 & \mc{5}{\textbf{Dataset: Proteins}} \\ \hline

COMBINEX$_{\textit{feat}}$ &$0.838 (\pm 0.199)$ & $0.316 (\pm 0.218)$ & $1517.836 (\pm 568.483)$ & $0.550 (\pm 0.195)$ & $n.d.(\pm n.d.)$ \\
COMBINEX$_{\textit{def}}$ &$0.903 (\pm 0.057)$ & $0.343 (\pm 0.180)$ & $1337.949 (\pm 455.085)$ & $0.508 (\pm 0.134)$ & $0.142 (\pm 0.145)$ \\
COMBINEX$_{\textit{dyn}}$ &$\uline{0.905 (\pm 0.110)}$ & $0.357 (\pm 0.142)$ & $1411.527 (\pm 782.195)$ & $\uline{0.503 (\pm 0.132)}$ & $\uline{0.036 (\pm 0.033)}$ \\
COMBINEX$_{\textit{exp}}$ &$\mathbf{0.933 (\pm 0.082)}$ & $\uline{0.524 (\pm 0.083)}$ & $1387.629 (\pm 784.114)$ & $0.669 (\pm 0.087)$ & $\mathbf{0.000 (\pm 0.000)}$ \\
COMBINEX$_{\textit{lin}}$ &$0.887 (\pm 0.154)$ & $0.317 (\pm 0.165)$ & $1470.758 (\pm 877.221)$ & $\mathbf{0.496 (\pm 0.129)}$ & $\mathbf{0.000 (\pm 0.000)}$ \\
COMBINEX$_{\textit{sin}}$ &$0.798 (\pm 0.154)$ & $0.485 (\pm 0.186)$ & $1192.001 (\pm 410.272)$ & $0.542 (\pm 0.137)$ & $\mathbf{0.000 (\pm 0.001)}$ \\
EGO &$0.000 (\pm 0.000)$ & $n.d.(\pm n.d.)$ & $n.d.(\pm n.d.)$ & $n.d.(\pm n.d.)$ & $n.d.(\pm n.d.)$ \\
Random Edges &$0.198 (\pm 0.068)$ & $-0.327 (\pm 0.096)$ & $\uline{612.136 (\pm 7.192)}$ & $n.d.(\pm n.d.)$ & $0.374 (\pm 0.005)$ \\
Random Features &$0.705 (\pm 0.236)$ & $0.261 (\pm 0.212)$ & $5581.317 (\pm 304.082)$ & $0.978 (\pm 0.000)$ & $n.d.(\pm n.d.)$ \\
CFF &$0.097 (\pm 0.048)$ & $-0.371 (\pm 0.125)$ & $1134.751 (\pm 562.934)$ & $n.d.(\pm n.d.)$ & $0.843 (\pm 0.035)$ \\
CF-GNNExplainer &$0.002 (\pm 0.003)$ & $n.d.(\pm n.d.)$ & $n.d.(\pm n.d.)$ & $n.d.(\pm n.d.)$ & $n.d.(\pm n.d.)$ \\
UNR &$0.015 (\pm 0.003)$ & $\mathbf{0.542 (\pm 0.629)}$ & $\mathbf{593.741 (\pm 166.865)}$ & $n.d.(\pm n.d.)$ & $0.040 (\pm 0.020)$ \\
\hline

\etable

\btable{l|ccccc}{Results for Biological datasets: AIDS, Enzymes, and Proteins. The oracles $\Phi$ use GraphConv layers. In \textbf{bold} the best result, the second best result is \underline{underlined}}{tab:bio_graph_conv}  \hline\hline

& \mc{1}{\textbf{Validity} $\uparrow$} 
& \mc{1}{\textbf{Fidelity} $\uparrow$} 
& \mc{1}{\textbf{Distribution Distance} $\downarrow$} 
& \mc{1}{\textbf{Node Sparsity} $\downarrow$} 
& \mc{1}{\textbf{Edge Sparsity} $\downarrow$} 
\\ 

\textbf{Explainers} 
& \textit{mean($\pm$std)}
& \textit{mean($\pm$std)}
& \textit{mean($\pm$std)}
& \textit{mean($\pm$std)}
& \textit{mean($\pm$std)} \\ \hline

 & \mc{5}{\textbf{Dataset: AIDS}} \\ \hline

Combinex$_{\textit{feat}}$ &$0.925 (\pm 0.023)$ & $0.942 (\pm 0.005)$ & $11.307 (\pm 1.199)$ & $0.861 (\pm 0.012)$ & $n.d.(\pm n.d.)$ \\
Combinex$_{\textit{def}}$ &$0.883 (\pm 0.061)$ & $0.936 (\pm 0.003)$ & $11.907 (\pm 1.409)$ & $0.831 (\pm 0.016)$ & $0.132 (\pm 0.046)$ \\
Combinex$_{\textit{dyn}}$ &$0.883 (\pm 0.050)$ & $0.938 (\pm 0.006)$ & $11.179 (\pm 0.864)$ & $0.845 (\pm 0.016)$ & $\uline{0.009 (\pm 0.004)}$ \\
Combinex$_{\textit{exp}}$ &$\uline{0.930 (\pm 0.029)}$ & $0.941 (\pm 0.007)$ & $9.329 (\pm 1.236)$ & $0.825 (\pm 0.011)$ & $\mathbf{0.003 (\pm 0.001)}$ \\
Combinex$_{\textit{lin}}$ &$0.915 (\pm 0.046)$ & $0.934 (\pm 0.006)$ & $11.597 (\pm 0.874)$ & $\mathbf{0.804 (\pm 0.016)}$ & $0.016 (\pm 0.008)$ \\
Combinex$_{\textit{sin}}$ &$\mathbf{0.932 (\pm 0.044)}$ & $0.939 (\pm 0.006)$ & $11.511 (\pm 0.773)$ & $\uline{0.807 (\pm 0.017)}$ & $0.026 (\pm 0.014)$ \\
EGO &$0.005 (\pm 0.006)$ & $\mathbf{1.000 (\pm 0.000)}$ & $3.924 (\pm 0.430)$ & $n.d.(\pm n.d.)$ & $0.729 (\pm 0.029)$ \\
Random Edges &$0.398 (\pm 0.153)$ & $0.851 (\pm 0.068)$ & $4.856 (\pm 0.851)$ & $n.d.(\pm n.d.)$ & $0.300 (\pm 0.011)$ \\
Random Features &$0.903 (\pm 0.058)$ & $0.945 (\pm 0.004)$ & $28.620 (\pm 0.188)$ & $0.938 (\pm 0.007)$ & $n.d.(\pm n.d.)$ \\
CFF &$0.107 (\pm 0.025)$ & $0.850 (\pm 0.062)$ & $\uline{3.590 (\pm 0.709)}$ & $n.d.(\pm n.d.)$ & $0.646 (\pm 0.055)$ \\
CF-GNNExplainer &$0.367 (\pm 0.160)$ & $0.903 (\pm 0.061)$ & $4.482 (\pm 0.675)$ & $n.d.(\pm n.d.)$ & $0.114 (\pm 0.002)$ \\
UNR &$0.048 (\pm 0.015)$ & $\uline{0.958 (\pm 0.083)}$ & $\mathbf{3.046 (\pm 0.375)}$ & $n.d.(\pm n.d.)$ & $0.225 (\pm 0.047)$ \\
\hline

 & \mc{5}{\textbf{Dataset: Enzymes}} \\ \hline
COMBINEX$_{\textit{feat}}$ &$0.603 (\pm 0.197)$ & $0.918 (\pm 0.036)$ & $88.556 (\pm 25.488)$ & $0.684 (\pm 0.014)$ & $n.d.(\pm n.d.)$ \\
COMBINEX$_{\textit{def}}$ &$0.458 (\pm 0.019)$ & $0.905 (\pm 0.028)$ & $109.872 (\pm 28.275)$ & $0.603 (\pm 0.064)$ & $0.106 (\pm 0.074)$ \\
COMBINEX$_{\textit{dyn}}$ &$0.653 (\pm 0.215)$ & $\uline{0.930 (\pm 0.020)}$ & $89.372 (\pm 23.019)$ & $0.591 (\pm 0.037)$ & $0.000 (\pm 0.000)$ \\
COMBINEX$_{\textit{exp}}$ &$\mathbf{0.710 (\pm 0.278)}$ & $0.924 (\pm 0.033)$ & $114.841 (\pm 28.302)$ & $0.717 (\pm 0.037)$ & $0.000 (\pm 0.001)$ \\
COMBINEX$_{\textit{lin}}$ &$0.672 (\pm 0.227)$ & $0.923 (\pm 0.022)$ & $97.356 (\pm 24.248)$ & $\uline{0.531 (\pm 0.009)}$ & $\uline{0.000 (\pm 0.000)}$ \\
COMBINEX$_{\textit{sin}}$ &$\uline{0.682 (\pm 0.277)}$ & $0.920 (\pm 0.035)$ & $123.189 (\pm 23.647)$ & $\mathbf{0.522 (\pm 0.036)}$ & $\mathbf{0.000 (\pm 0.000)}$ \\
EGO &$0.002 (\pm 0.003)$ & $n.d.(\pm n.d.)$ & $n.d.(\pm n.d.)$ & $n.d.(\pm n.d.)$ & $n.d.(\pm n.d.)$ \\
Random Edges &$0.045 (\pm 0.006)$ & $0.476 (\pm 0.101)$ & $28.512 (\pm 3.645)$ & $n.d.(\pm n.d.)$ & $0.406 (\pm 0.019)$ \\
Random Features &$0.477 (\pm 0.071)$ & $0.917 (\pm 0.010)$ & $328.022 (\pm 14.579)$ & $0.987 (\pm 0.002)$ & $n.d.(\pm n.d.)$ \\
CFF &$0.055 (\pm 0.021)$ & $0.748 (\pm 0.176)$ & $\uline{17.583 (\pm 3.239)}$ & $n.d.(\pm n.d.)$ & $0.781 (\pm 0.031)$ \\
CF-GNNExplainer &$0.013 (\pm 0.005)$ & $-0.042 (\pm 0.672)$ & $33.897 (\pm 32.368)$ & $n.d.(\pm n.d.)$ & $0.023 (\pm 0.004)$ \\
UNR &$0.010 (\pm 0.004)$ & $-0.125 (\pm 0.854)$ & $30.570 (\pm 23.719)$ & $n.d.(\pm n.d.)$ & $0.149 (\pm 0.129)$ \\
\hline

 & \mc{5}{\textbf{Dataset: Proteins}} \\ \hline

Combinex$_{\textit{feat}}$ &$0.433 (\pm 0.369)$ & $0.116 (\pm 0.183)$ & $1176.168 (\pm 325.353)$ & $\uline{0.407 (\pm 0.104)}$ & $n.d.(\pm n.d.)$ \\
Combinex$_{\textit{def}}$ &$0.203 (\pm 0.172)$ & $0.148 (\pm 0.383)$ & $1377.694 (\pm 376.297)$ & $\mathbf{0.406 (\pm 0.091)}$ & $\mathbf{0.000 (\pm 0.000)}$ \\
Combinex$_{\textit{dyn}}$ &$0.312 (\pm 0.307)$ & $-0.252 (\pm 0.581)$ & $2928.744 (\pm 2909.224)$ & $0.505 (\pm 0.153)$ & $0.007 (\pm 0.014)$ \\
Combinex$_{\textit{exp}}$ &$\mathbf{0.615 (\pm 0.418)}$ & $0.137 (\pm 0.089)$ & $1370.427 (\pm 767.820)$ & $0.476 (\pm 0.229)$ & $\mathbf{0.000 (\pm 0.000)}$ \\
Combinex$_{\textit{lin}}$ &$0.472 (\pm 0.331)$ & $0.161 (\pm 0.125)$ & $1201.399 (\pm 330.125)$ & $0.444 (\pm 0.128)$ & $\mathbf{0.000 (\pm 0.000)}$ \\
Combinex$_{\textit{sin}}$ &$\uline{0.543 (\pm 0.042)}$ & $0.128 (\pm 0.162)$ & $1495.117 (\pm 725.399)$ & $0.434 (\pm 0.110)$ & $\uline{0.000 (\pm 0.000)}$ \\
EGO &$0.038 (\pm 0.064)$ & $-0.800 (\pm 0.346)$ & $\uline{666.072 (\pm 144.942)}$ & $n.d.(\pm n.d.)$ & $0.916 (\pm 0.038)$ \\
Random Edges &$0.438 (\pm 0.513)$ & $0.128 (\pm 0.124)$ & $832.470 (\pm 30.495)$ & $n.d.(\pm n.d.)$ & $0.363 (\pm 0.018)$ \\
Random Features &$0.147 (\pm 0.195)$ & $\mathbf{0.401 (\pm 0.532)}$ & $5631.871 (\pm 801.806)$ & $0.979 (\pm 0.004)$ & $n.d.(\pm n.d.)$ \\
CFF &$0.203 (\pm 0.048)$ & $\uline{0.265 (\pm 0.199)}$ & $1071.473 (\pm 324.853)$ & $n.d.(\pm n.d.)$ & $0.851 (\pm 0.022)$ \\
CF-GNNExplainer &$0.015 (\pm 0.030)$ & $n.d.(\pm n.d.)$ & $n.d.(\pm n.d.)$ & $n.d.(\pm n.d.)$ & $n.d.(\pm n.d.)$ \\
UNR &$0.060 (\pm 0.049)$ & $0.241 (\pm 0.370)$ & $\mathbf{617.709 (\pm 60.279)}$ & $n.d.(\pm n.d.)$ & $0.086 (\pm 0.014)$ \\
\hline

\etable

\btable{l|ccccc}{Results for Biological datasets: AIDS, Enzymes, and Proteins. The oracles $\Phi$ use ChebConv layers. In \textbf{bold} the best result, the second best result is \underline{underlined}}{tab:bio_cheb_conv}  \hline\hline

& \mc{1}{\textbf{Validity} $\uparrow$} 
& \mc{1}{\textbf{Fidelity} $\uparrow$} 
& \mc{1}{\textbf{Distribution Distance} $\downarrow$} 
& \mc{1}{\textbf{Node Sparsity} $\downarrow$} 
& \mc{1}{\textbf{Edge Sparsity} $\downarrow$} 
\\ 

\textbf{Explainers} 
& \textit{mean($\pm$std)}
& \textit{mean($\pm$std)}
& \textit{mean($\pm$std)}
& \textit{mean($\pm$std)}
& \textit{mean($\pm$std)} \\ \hline

 & \mc{5}{\textbf{Dataset: AIDS}} \\ \hline

Combinex$_{\textit{feat}}$ &$\uline{0.938 (\pm 0.014)}$ & $0.996 (\pm 0.004)$ & $4.553 (\pm 0.035)$ & $0.753 (\pm 0.000)$ & $n.d.(\pm n.d.)$ \\
Combinex$_{\textit{def}}$ &$\mathbf{0.944 (\pm 0.015)}$ & $0.995 (\pm 0.004)$ & $4.520 (\pm 0.035)$ & $0.754 (\pm 0.001)$ & $\mathbf{0.000 (\pm 0.000)}$ \\
Combinex$_{\textit{dyn}}$ &$\mathbf{0.944 (\pm 0.015)}$ & $0.995 (\pm 0.004)$ & $4.585 (\pm 0.065)$ & $\mathbf{0.752 (\pm 0.003)}$ & $\mathbf{0.000 (\pm 0.000)}$ \\
Combinex$_{\textit{exp}}$ &$\mathbf{0.944 (\pm 0.015)}$ & $0.995 (\pm 0.004)$ & $4.589 (\pm 0.106)$ & $\uline{0.752 (\pm 0.003)}$ & $\mathbf{0.000 (\pm 0.000)}$ \\
Combinex$_{\textit{lin}}$ &$\mathbf{0.944 (\pm 0.015)}$ & $0.995 (\pm 0.004)$ & $4.504 (\pm 0.048)$ & $0.753 (\pm 0.000)$ & $\mathbf{0.000 (\pm 0.000)}$ \\
Combinex$_{\textit{sin}}$ &$\mathbf{0.944 (\pm 0.015)}$ & $0.995 (\pm 0.004)$ & $\uline{4.503 (\pm 0.048)}$ & $0.754 (\pm 0.000)$ & $\mathbf{0.000 (\pm 0.000)}$ \\
EGO &$0.000 (\pm 0.000)$ & $n.d.(\pm n.d.)$ & $n.d.(\pm n.d.)$ & $n.d.(\pm n.d.)$ & $n.d.(\pm n.d.)$ \\
Random Edges &$0.000 (\pm 0.000)$ & $n.d.(\pm n.d.)$ & $n.d.(\pm n.d.)$ & $n.d.(\pm n.d.)$ & $n.d.(\pm n.d.)$ \\
Random Features &$0.935 (\pm 0.010)$ & $\uline{0.998 (\pm 0.004)}$ & $28.516 (\pm 0.198)$ & $0.943 (\pm 0.009)$ & $n.d.(\pm n.d.)$ \\
CFF &$0.028 (\pm 0.013)$ & $\mathbf{1.000 (\pm 0.000)}$ & $\mathbf{3.782 (\pm 1.022)}$ & $n.d.(\pm n.d.)$ & $\uline{0.638 (\pm 0.079)}$ \\
CF-GNNExplainer &$0.000 (\pm 0.000)$ & $n.d.(\pm n.d.)$ & $n.d.(\pm n.d.)$ & $n.d.(\pm n.d.)$ & $n.d.(\pm n.d.)$ \\
UNR &$0.000 (\pm 0.000)$ & $n.d.(\pm n.d.)$ & $n.d.(\pm n.d.)$ & $n.d.(\pm n.d.)$ & $n.d.(\pm n.d.)$ \\
\hline

 & \mc{5}{\textbf{Dataset: Enzymes}} \\ \hline
COMBINEX$_{\textit{feat}}$ &$\mathbf{0.412 (\pm 0.006)}$ & $\mathbf{0.951 (\pm 0.001)}$ & $202.277 (\pm 0.703)$ & $0.105 (\pm 0.000)$ & $n.d.(\pm n.d.)$ \\
COMBINEX$_{\textit{def}}$ &$\mathbf{0.412 (\pm 0.006)}$ & $\mathbf{0.951 (\pm 0.001)}$ & $201.834 (\pm 0.711)$ & $0.104 (\pm 0.002)$ & $\mathbf{0.000 (\pm 0.000)}$ \\
COMBINEX$_{\textit{dyn}}$ &$\mathbf{0.412 (\pm 0.006)}$ & $\mathbf{0.951 (\pm 0.001)}$ & $\uline{201.733 (\pm 0.686)}$ & $0.104 (\pm 0.002)$ & $\mathbf{0.000 (\pm 0.000)}$ \\
COMBINEX$_{\textit{exp}}$ &$\mathbf{0.412 (\pm 0.006)}$ & $\mathbf{0.951 (\pm 0.001)}$ & $201.739 (\pm 0.724)$ & $0.109 (\pm 0.001)$ & $\mathbf{0.000 (\pm 0.000)}$ \\
COMBINEX$_{\textit{lin}}$ &$\mathbf{0.412 (\pm 0.006)}$ & $\mathbf{0.951 (\pm 0.001)}$ & $201.741 (\pm 0.712)$ & $\mathbf{0.094 (\pm 0.002)}$ & $\mathbf{0.000 (\pm 0.000)}$ \\
COMBINEX$_{\textit{sin}}$ &$\mathbf{0.412 (\pm 0.006)}$ & $\mathbf{0.951 (\pm 0.001)}$ & $201.742 (\pm 0.709)$ & $\uline{0.095 (\pm 0.002)}$ & $\mathbf{0.000 (\pm 0.000)}$ \\
EGO &$0.000 (\pm 0.000)$ & $n.d.(\pm n.d.)$ & $n.d.(\pm n.d.)$ & $n.d.(\pm n.d.)$ & $n.d.(\pm n.d.)$ \\
Random Edges &$0.000 (\pm 0.000)$ & $n.d.(\pm n.d.)$ & $n.d.(\pm n.d.)$ & $n.d.(\pm n.d.)$ & $n.d.(\pm n.d.)$ \\
Random Features &$\mathbf{0.412 (\pm 0.006)}$ & $\mathbf{0.951 (\pm 0.001)}$ & $308.878 (\pm 5.147)$ & $0.987 (\pm 0.003)$ & $n.d.(\pm n.d.)$ \\
CFF &$\uline{0.098 (\pm 0.018)}$ & $\uline{0.934 (\pm 0.089)}$ & $\mathbf{18.847 (\pm 6.647)}$ & $n.d.(\pm n.d.)$ & $\uline{0.762 (\pm 0.023)}$ \\
CF-GNNExplainer &$0.000 (\pm 0.000)$ & $n.d.(\pm n.d.)$ & $n.d.(\pm n.d.)$ & $n.d.(\pm n.d.)$ & $n.d.(\pm n.d.)$ \\
UNR &$0.000 (\pm 0.000)$ & $n.d.(\pm n.d.)$ & $n.d.(\pm n.d.)$ & $n.d.(\pm n.d.)$ & $n.d.(\pm n.d.)$ \\
\hline

 & \mc{5}{\textbf{Dataset: Proteins}} \\ \hline

COMBINEX$_{\textit{feat}}$ &$\mathbf{0.390 (\pm 0.019)}$ & $\mathbf{0.885 (\pm 0.046)}$ & $\uline{859.827 (\pm 3.031)}$ & $\mathbf{0.339 (\pm 0.000)}$ & $n.d.(\pm n.d.)$ \\
COMBINEX$_{\textit{def}}$ &$\mathbf{0.390 (\pm 0.019)}$ & $\mathbf{0.885 (\pm 0.046)}$ & $\uline{859.827 (\pm 3.031)}$ & $\mathbf{0.339 (\pm 0.000)}$ & $\mathbf{0.000 (\pm 0.000)}$ \\
COMBINEX$_{\textit{dyn}}$ &$\mathbf{0.390 (\pm 0.019)}$ & $\mathbf{0.885 (\pm 0.046)}$ & $\uline{859.827 (\pm 3.031)}$ & $\mathbf{0.339 (\pm 0.000)}$ & $\mathbf{0.000 (\pm 0.000)}$ \\
COMBINEX$_{\textit{exp}}$ &$\mathbf{0.390 (\pm 0.019)}$ & $\mathbf{0.885 (\pm 0.046)}$ & $\uline{859.827 (\pm 3.031)}$ & $\mathbf{0.339 (\pm 0.000)}$ & $\mathbf{0.000 (\pm 0.000)}$ \\
COMBINEX$_{\textit{lin}}$ &$\mathbf{0.390 (\pm 0.019)}$ & $\mathbf{0.885 (\pm 0.046)}$ & $\uline{859.827 (\pm 3.031)}$ & $\mathbf{0.339 (\pm 0.000)}$ & $\mathbf{0.000 (\pm 0.000)}$ \\
COMBINEX$_{\textit{sin}}$ &$\mathbf{0.390 (\pm 0.019)}$ & $\mathbf{0.885 (\pm 0.046)}$ & $\uline{859.827 (\pm 3.031)}$ & $\mathbf{0.339 (\pm 0.000)}$ & $\mathbf{0.000 (\pm 0.000)}$ \\
EGO &$0.000 (\pm 0.000)$ & $n.d.(\pm n.d.)$ & $n.d.(\pm n.d.)$ & $n.d.(\pm n.d.)$ & $n.d.(\pm n.d.)$ \\
Random Edges &$0.000 (\pm 0.000)$ & $n.d.(\pm n.d.)$ & $n.d.(\pm n.d.)$ & $n.d.(\pm n.d.)$ & $n.d.(\pm n.d.)$ \\
Random Features &$\mathbf{0.390 (\pm 0.019)}$ & $\mathbf{0.885 (\pm 0.046)}$ & $5361.353 (\pm 205.403)$ & $\uline{0.980 (\pm 0.000)}$ & $n.d.(\pm n.d.)$ \\
CFF &$\uline{0.038 (\pm 0.017)}$ & $\uline{-0.019 (\pm 0.432)}$ & $\mathbf{524.514 (\pm 173.529)}$ & $n.d.(\pm n.d.)$ & $\uline{0.768 (\pm 0.050)}$ \\
CF-GNNExplainer &$0.000 (\pm 0.000)$ & $n.d.(\pm n.d.)$ & $n.d.(\pm n.d.)$ & $n.d.(\pm n.d.)$ & $n.d.(\pm n.d.)$ \\
UNR &$0.000 (\pm 0.000)$ & $n.d.(\pm n.d.)$ & $n.d.(\pm n.d.)$ & $n.d.(\pm n.d.)$ & $n.d.(\pm n.d.)$ \\

\hline

\etable

\end{document}